\documentclass{article}

\usepackage[main, final]{neurips_2026}

\usepackage[utf8]{inputenc} 

\usepackage{hyperref}       
\usepackage{url}            
\usepackage{booktabs}       
\usepackage{amsfonts}       
\usepackage{nicefrac}       
\usepackage{microtype}      
\usepackage{xcolor}         
\usepackage{graphicx}
\usepackage{subcaption}
\usepackage{enumitem}

\usepackage{amsmath}
\usepackage{amssymb}
\usepackage{mathtools}
\usepackage{amsthm}
\usepackage{bbm}
\usepackage{algorithm}
\usepackage{algorithmic}

\usepackage{tikz}
\usetikzlibrary{calc}
\usetikzlibrary{positioning}

\DeclareMathOperator*{\argmax}{argmax}
\DeclareMathOperator*{\argmin}{argmin}
\newcommand{\tsum}{{\textstyle\sum}}
\newcommand{\tint}{{\textstyle\int}}

\theoremstyle{plain}
\newtheorem{theorem}{Theorem}[section]
\newtheorem{proposition}[theorem]{Proposition}
\newtheorem{lemma}[theorem]{Lemma}
\newtheorem{corollary}[theorem]{Corollary}
\theoremstyle{definition}
\newtheorem{definition}[theorem]{Definition}
\newtheorem{assumption}[theorem]{Assumption}
\theoremstyle{remark}
\newtheorem{remark}[theorem]{Remark}
\newtheorem{example}[theorem]{Example}

\title{PPO in the Fisher--Rao geometry}

\author{Razvan-Andrei Lascu\thanks{Equal contribution.} \\
  RIKEN AIP\\
  \texttt{razvan-andrei.lascu@riken.jp} \\
  \And
   David \v{S}i\v{s}ka\protect\footnotemark[1] \\
   University of Edinburgh \\
   \texttt{d.siska@ed.ac.uk} \\
   \AND
  \L{}ukasz Szpruch \\
   University of Edinburgh\\
   \texttt{l.szpruch@ed.ac.uk} \\
}

\begin{document}

\maketitle

\begin{abstract}
  Proximal Policy Optimization (PPO) is widely used in reinforcement learning due to its strong empirical performance, yet it lacks formal guarantees for policy improvement and convergence. PPO's clipped surrogate objective is motivated by a lower bound on linearization of the value function in flat geometry setting. We derive a tighter surrogate objective and introduce Fisher--Rao PPO (FR-PPO) by leveraging the Fisher--Rao (FR) geometry. Our scheme provides strong theoretical guarantees, including monotonic policy improvement. In the direct parametrization setting, we show that FR-PPO achieves sub-linear convergence, and for parametrized policies we further obtain sub-linear convergence up to the compatible function approximation error. Finally, although our primary focus is theoretical, we also demonstrate empirically that FR-PPO can be implemented as a drop-in replacement for PPO and achieves performance comparable to PPO across a range of standard reinforcement learning tasks.
\end{abstract}

\section{Introduction}
\label{sec:Introduction}
Policy gradient methods are among the standard approaches for policy optimization in reinforcement learning (RL) \cite{PETERS2008682,HU}, but their performance is highly sensitive to the size of the policy update. Since training data are generated by the current policy, an overly aggressive update may lead to performance degradation \cite{kakade,schulman-15}. This has motivated a large class of trust region and proximal policy optimization methods whose purpose is to control policy change.

Trust Region Policy Optimization (TRPO) addresses this issue by optimizing a linearized performance surrogate subject to a constraint on the Kullback-Leibler (KL) divergence between the old and new policies \cite{schulman-15}. This trust region formulation is supported by a performance difference improved lower bound and leads to monotonic improvement guarantees. However, enforcing the KL constraint requires solving a constrained optimization problem, typically with second-order approximations and backtracking line search, which makes TRPO less convenient in large scale implementations.

Proximal Policy Optimization (PPO) was introduced as a simpler first-order alternative to TRPO \cite{schulman2017}. In its clipped form, PPO replaces the explicit KL trust region constraint by a clipped objective depending on the ratio between the new and old policies. This clipping mechanism is intended to discourage large ratio deviations while retaining the simplicity of unconstrained optimization. This design has made PPO one of the most widely used policy gradient algorithms in practice \cite{ye,makoviychuk2021isaac}.

From a theoretical perspective, however, the clipped PPO surrogate is much less understood than the TRPO surrogate. In particular, the clipped objective is not, in general, a performance difference improved lower bound of the kind used in the TRPO analysis, and the original PPO update does not come with a monotonic improvement or convergence theory analogous to trust region methods. Existing analyses typically study PPO-inspired variants, modified objectives, special policy classes, or weaker notions of convergence. Moreover, empirical studies have observed that PPO's clipping mechanism does not necessarily keep policy ratios within the intended range \cite{Ilyas2018AreDP,Engstrom2020Implementation,wang20}. Thus, there remains a gap between PPO's ratio control heuristic and the performance difference improvement theory underlying trust region methods.

This paper studies this gap from a geometric perspective. We ask whether the ratio control principle behind PPO can be replaced by a smooth divergence between policies while still remaining close to PPO's original motivation. Our answer is based on the Fisher--Rao (FR) geometry. The key technical step is a new performance lower bound controlled by an occupancy averaged squared total-variation (TV) term. This bound connects PPO-style ratio control to an FR surrogate on squared policy densities, leading to an exact proximal update with monotonic improvement and convergence guarantees.

\subsection{Our contribution}
The goal of this paper is to give a mathematically rigorous account of a PPO-motivated ratio based update. The starting point is the observation that PPO's clipped surrogate is naturally connected to controlling ratio deviations, and hence to TV-type control between the old and new policies. However, this connection is difficult to turn into a convergence theory since the TV distance is non-smooth, is not a Bregman divergence, and therefore does not directly fit the standard mirror descent analysis based on first-order optimality conditions or the Bregman three-point inequality \cite{lan2022policymirrordescentreinforcement,kerimkulov2024fisherraogradientflowentropyregularised,korba}.

Our main contribution is to show that this PPO-style ratio control principle admits a smooth Bregman formulation. The key technical step is a new lower bound on the performance difference in which the penalty is an occupancy averaged $\operatorname{TV}^2$ term. This bound replaces the classical uniform-in-state control used in TRPO by a penalty integrated under the old policy occupancy measure. The Fisher--Rao geometry then provides a smooth Bregman upper bound for this penalty on squared policy densities, giving a proximal update that can be analyzed by mirror descent tools.

The contributions are as follows:
\begin{itemize}[leftmargin = 5mm]
    \item We prove a new surrogate lower bound for the performance difference between two policies (Theorem \ref{thm:tv2_mixed} and Corollary \ref{thm lower bound perf diff}). Unlike the classical TRPO bound of \cite{schulman-15}, which uses a uniform-in-state divergence penalty, our bound involves $\operatorname{TV}^2$ averaged with respect to the old policy occupancy measure. This is the technical step that connects PPO-style ratio control to a rigorous improvement surrogate.

    \item We show that the occupancy averaged $\operatorname{TV}^2$ penalty is controlled by a Bregman divergence generated by the $\chi^2$ divergence, equivalently by $\operatorname{FR}^2$ on squared densities (Corollary \ref{cor lower bound FR^2}). This converts the non-smooth TV penalty suggested by PPO clipping into a smooth Bregman proximal term, making the resulting update amenable to mirror descent analysis.

    \item Motivated by the FR lower bound, we define an exact policy update. For this update, we prove monotonic policy improvement together with sublinear and linear last-iterate convergence bounds in the direct parametrization setting (Theorem \ref{thm:convergence_fr_steps}).

    \item We then show that the same FR geometry induces a projected natural policy gradient update for a compatible linear policy class. Under exact compatible function approximation, we obtain sublinear and linear convergence bounds, and under inexact approximation, we obtain the corresponding bound up to an explicit approximation error term (Theorems \ref{thm:sublinear_convergence_proj_NPG} and \ref{thm:sublinear_convergence_approx_proj_NPG}).

    \item Finally, we derive a sample-based FR-PPO policy loss that can be used as a drop-in replacement for PPO's clipped loss. The experiments are included to show that the resulting loss is implementable and matches PPO under standard training pipelines, without claiming empirical superiority.
\end{itemize}
In summary, the paper provides a convergence theory for a PPO-motivated ratio surrogate. The main novelty is the path from a new occupancy averaged $\operatorname{TV}^2$ performance lower bound to an FR proximal update with monotonic improvement and convergence guarantees. To our knowledge, these results provide the first monotonic improvement and convergence analysis for a PPO-motivated ratio surrogate derived through FR geometry.

\subsection{Related work}
Policy gradient methods originate with REINFORCE \cite{williams} and the policy gradient theorem with function approximation \cite{sutton1999policy}. Natural policy gradient \cite{kakadenatural}, and later actor-critic methods \cite{PETERS2008682}, showed that the suitable geometry of policy space can stabilize and improve policy optimization relative to vanilla gradient ascent. Our analysis of parametrized policies in the FR geometry belongs to this compatible function approximation line of work, but with a different underlying geometry.

A second line of work concerns trust region and PPO-style surrogates. TRPO \cite{schulman-15} derives monotonic improvement guarantees from a surrogate objective coupled with a KL trust region. PPO \cite{schulman2017} replaces the constrained step by a clipped or KL penalized first-order objective and became the dominant practical baseline. Subsequent empirical and theoretical work clarified that clipping does not, in general, enforce an exact trust region and can behave differently from the intended ratio control \cite{Engstrom2020Implementation,wang20}. Several later theory papers study PPO-inspired variants rather than the original algorithm. A closer exception is \cite{jin2024on}, who prove stationary point convergence rates for PPO-type updates. Other works analyze modified settings or objectives. For instance, \cite{liuCaiYangWang} study over-parametrized neural PPO/TRPO variants, \cite{Huang2023PPOClipAG} study Neural-PPO-Clip with softmax policies and neural network function approximation, \cite{queeney2021generalized} derive off-policy improvement bounds that motivate GePPO, and \cite{zhong2023a} analyze an optimistic PPO variant in linear MDPs. Our paper differs in both objective and technique. We do not analyze the non-smooth clipped PPO objective directly. Instead, we derive a new performance lower bound whose penalty is an occupancy averaged $\operatorname{TV}^2$ term, show that this penalty admits an $\operatorname{FR}^2$ upper bound on a suitable policy class, and analyze the resulting FR update. This yields monotonic improvement and last-iterate convergence guarantees for the FR surrogate, rather than stationary point or average-iterate convergence for the original clipped objective.

A third line studies policy mirror descent (MD) and related policy space optimization methods under KL or other divergences. \cite{agarwal} establish global convergence and approximation guarantees for policy gradient methods and obtain sub-linear guarantees in a compatible function approximation setting for their parametrized NPG analysis. \cite{mei20b} prove global convergence results for softmax policy gradient in tabular settings via a non-uniform Polyak--\L ojasiewicz inequality. MD on spaces of measures has been studied for static optimization by \cite{korba}, and policy MD has been considered in RL by \cite{lan2022policymirrordescentreinforcement,zhan2023policymirrordescentregularized,tomar2021mirrordescentpolicyoptimization}. These works typically use KL geometries or entropy-regularized objectives. Related results for log-linear policies were obtained by \cite{yuan2023linear}, who prove linear convergence with geometrically increasing step sizes. We show that an analogous rate can be obtained for the FR geometry considered here. More recently, \cite{asad25a} study softmax policy mirror ascent (SPMA). They prove linear convergence in finite tabular MDPs with exact policy gradients, and convergence to a neighborhood of the optimum under function approximation, including regimes beyond compatible approximation. Their analysis is based on the KL geometry. Simple Policy Optimization (SPO) \cite{xie2025simple} is also closely related in spirit. It modifies the PPO objective while retaining a first-order implementation and ratio-based control by replacing KL with $\chi^2$-divergence in the surrogate objective. However, unlike KL, the $\chi^2$-divergence is not a Bregman
divergence, which prevents the use of MD analysis that is essential for establishing convergence rates. These works show that the relevant comparison class is broader than PPO alone. Our contribution within this comparison class is different. We do not start from a KL, softmax mirror map or $\chi^2$. Instead, we show that a PPO-style ratio surrogate can be recast in an FR geometry on a suitable policy class. This restores a Bregman structure for a smooth ratio-based penalty, allowing MD analysis beyond the usual KL framework. The resulting guarantees apply both to the direct policy update and to a compatible parametrized policy class for general Polish state and action spaces in the MDP model.

Finally, previous work studies FR gradient flows for entropy-regularized MDPs on Polish spaces \cite{kerimkulov2024fisherraogradientflowentropyregularised} and $L^2$-projected policy gradient methods~\cite{xiao2022pgrates,liu2025ppg}.
In the finite tabular setting, our FR proximal stepping scheme reduces to a statewise $L^2$-projection onto the probability simplex. This tabular projection formula should be viewed as a specialization of our more general analysis over Polish state and action spaces. The contributions relative to this literature are threefold. First, we derive a new occupancy averaged $\operatorname{TV}^2$ surrogate lower bound for monotonic policy improvement. Second, we use the FR/Bregman geometry to turn the resulting PPO-style ratio control into a smooth proximal objective amenable to MD analysis. Third, we establish a variational representation for the parametrized update, which extends the convergence analysis from direct parametrization to a compatible linear approximation setting. This compatible linear approximation setting includes the finite tabular case as a special case, where the parametrized FR update recovers the projected policy gradient update.

\section{Preliminaries}
This section outlines the geometric notions employed in the paper, most notably geodesic convexity in flat and Fisher--Rao geometries, followed by standard Markov Decision Processes (MDPs) notation and a review of key literature results that support our main contributions. 
\subsection{Geodesic convexity in the flat and Fisher-Rao geometries}
\label{subsection:geod-convexity}
For a Polish space $A$ we denote by $\mathcal{M}^+_\lambda(A)\subset \mathcal{M}^+(A)$ the subset of the space of non-negative measures $\mathcal{M}^+(A)$ that are absolutely continuous with respect to a fixed reference measure $\lambda$ and by $\mathcal{P}_\lambda(A) \subset\mathcal{M}^+_\lambda(A)$ the subset of probability measures.
\begin{definition}[The squared Fisher--Rao (FR) distance]
\label{def:FR2}
Let $\lambda$ be a reference measure.
For all $\mu, \mu' \in \mathcal{M}^+_\lambda(A)$, the squared Fisher--Rao (FR) distance is defined by 
\begin{equation*}
\textnormal{FR}^2(\mu, \mu') := 4\int_A\Big|\sqrt{\tfrac{\mathrm{d}\mu}{\mathrm{d}\lambda}} -\sqrt{\tfrac{\mathrm{d}\mu'}{\mathrm{d}\lambda}}\Big|^2\mathrm{d}\lambda.
\end{equation*} 
\end{definition}
If we restrict to $\mathcal{P}_\lambda(A)$, then $\big\|\sqrt{\frac{\mathrm{d}\mu}{\mathrm{d}\lambda}}\big\|_{L^2_{\lambda}(A)} = 1$, so the $\textnormal{FR}^2$ metric measures distances between points lying on the positive orthant of the $L^2_\lambda$-unit sphere.
\begin{definition}[Geodesic convexity in the flat and Fisher--Rao geometry]
\label{def:geod-conv}
A function $F:\mathcal{M}^+(A) \to \mathbb R$ is \textit{geodesically convex} if, for all $\mu,\mu' \in \mathcal{M}^+(A)$, and $\eta \in [0,1]$, there exists $(\mu_\eta)_{\eta \in [0,1]}$ such that $F\left(\mu_\eta\right) \leq (1-\eta)F(\mu) + \eta F(\mu')$.
If $\mu_\eta^{\text{flat}} = (1-\eta)\mu + \eta \mu'$ and $\mu_\eta^{\text{FR}} = \left((1-\eta)\sqrt{\mu}+\eta\sqrt{\mu'}\right)^2$, then $F$ is called flat and FR geodesically convex, respectively. We may restrict to $\mathcal{P}_\lambda(A)$ by normalizing $\mu_\eta^{\text{FR}}$.
\end{definition}
Since $\operatorname{FR}^2(\mu,\mu')$ is symmetric, it is FR geodesically convex in both arguments. Other examples of FR geodesically convex functionals include $\int_A f(a)\sqrt{\tfrac{\mathrm{d}\mu}{\mathrm{d}\lambda}}(a)\lambda(\mathrm{d}a)$ and the reversed KL divergence $\operatorname{KL}(\pi|\mu)$, for fixed $\pi \in \mathcal{P}_\mu(A)$. After applying the transformation $\frac{\mathrm{d}\mu}{\mathrm{d}\lambda} \mapsto \big(\frac{\mathrm{d}\mu}{\mathrm{d}\lambda}\big)^2$, the lifted $\operatorname{FR}^2(\mu^2,{\mu'}^2)$ distance becomes flat geodesically convex.
In particular, note that $\operatorname{FR}^2(\mu^2,{\mu'}^2) = 4\Vert{}\frac{\mathrm{d}\mu}{\mathrm{d}\lambda}-\frac{\mathrm{d}\mu'}{\mathrm{d}\lambda}\Vert{}_{L_\lambda^2(A)}^2$. This demonstrates that the $\operatorname{FR}^2$ metric on squared densities reduces, up to a constant factor, to the standard $L_\lambda^2$ distance between densities.
To ensure this distance is well-defined, we impose $\frac{\mathrm{d}\mu}{\mathrm{d}\lambda} \in L_\lambda^2(A)$, restricting our measures to $\mathfrak C_\lambda \coloneqq \big\{\mu \in \mathcal{P}_\lambda(A): \frac{\mathrm{d}\mu}{\mathrm{d}\lambda} \in L_\lambda^2(A)\big\}$.

\subsection{Markov Decision Processes}
\label{sec:MDP}
Let $(S, A, P, r, \gamma)$ be an infinite horizon MDP, with state and action spaces $S$ and $A$ (which may be arbitrary Polish spaces), $P$ is the transition kernel, $r$ is a bounded reward and $\gamma\in [0,1)$ is a discount factor. For full details on our assumptions and notation, we refer to Appendix \ref{appendix:more-notation-MDP}. For a given policy $\pi\in\mathcal{P}(A|S)$, we define the value function $V_{\pi}:S\rightarrow \mathbb{R}$ by $V_{\pi}(s)= \mathbb{E}_{s}^{\pi}\big[\sum_{n=0}^{\infty}\gamma^{n}r(s_n,a_n)\big]$, where $\mathbb E_s^{\pi}$ denotes the expectation over the state-action trajectory $(s_0,a_0,s_1,a_1,...)$ generated by policy $\pi$ and kernel $P$ such that $s_0 \coloneqq s$, $a_n \sim \pi(\cdot|s_n)$ and $s_{n+1} \sim P(\cdot|s_n,a_n)$, for all $n \geq 0$. For fixed $\rho \in \mathcal{P}(S)$, we consider the maximization problem
\begin{equation}
\label{eq:objective-for-pi}
\max_{\pi \in \mathcal{P}(A|S)} V_{\pi}(\rho), \text{ with } V_{\pi}(\rho) := \int_{S}V_{\pi}(s)\rho(\mathrm{d}s).
\end{equation}
Under Assumption \ref{ass for dpp}, which guarantees the required measurable selection, the Dynamic Programming Principle (Theorem \ref{thm:DPP}) gives an optimal policy $\pi^* \in \mathcal{P}(A|S)$, not necessarily unique, such that $V^*:=V_{\pi^*} = \max_{\pi \in \mathcal{P}(A|S)} V_{\pi}(\rho)$. With the on-policy Q-function denoted by $Q_\pi$ and defined in~\eqref{def:Qpi} we have the advantage function $A_\pi = Q_\pi - V_{\pi}$. The occupation measure under policy $\pi$ with initial state distribution $\rho$ is denoted $d^\pi_\rho$ and defined in~\eqref{occupancy kernel}.
We now recall one of the key tools in the theoretical analysis of policy optimization algorithms: the performance difference lemma. Originally due to~\cite[Ch. 7, p. 87]{howard1960dynamic}, it has since become fundamental in the study of policy gradient methods for both tabular and general state-action spaces (see, e.g., \cite{kakade,lan2022policymirrordescentreinforcement, kerimkulov2024fisherraogradientflowentropyregularised}). We introduce the following compact notation: given $\nu \in \mathcal{P}(S)$, define $\langle\cdot,\cdot \rangle_\nu: B_b(S\times A) \times b\mathcal{M}(A|S) \to \mathbb R$ by $\langle f,\mu \rangle_\nu = \frac{1}{1-\gamma}\int_S \int_A f(s,a)\mu(\mathrm{d}a|s)\nu(\mathrm{d}s)$.

\begin{lemma}[Performance difference]
\label{lemma perf diff}
Let $\rho \in \mathcal{P}(S)$. For all $\pi', \pi \in \mathcal{P}(A|S)$, 
\begin{equation}
\label{eqn perf diff}
(V_{\pi'}-V_{\pi})(\rho) = \left\langle A_\pi, \pi'-\pi \right\rangle_{d^{\pi'}_{\rho}}.
\end{equation}
\end{lemma}
For completeness, we provide a proof in Appendix \ref{proofs perf diff and tight bound}.
Since $\int_A A_{\pi}(s,a)\pi(\mathrm{d}a|s) = 0$, for every $s \in S$, Lemma \ref{lemma perf diff} implies that any update to a policy $\pi'$ satisfying $\int_A A_{\pi}(s,a)\pi'(\mathrm{d}a|s) \geq 0$ will either increase the value function or keep it unchanged when the expected advantage is zero at all states. From Lemma~\ref{lemma perf diff} one gets a policy gradient theorem for Polish spaces (see, e.g., \citep[Proposition A.1]{kerimkulov2024fisherraogradientflowentropyregularised}), which complements the classical policy gradient result of \cite{sutton1999policy}.
\begin{theorem}[Policy gradient without parametrization]
\label{thm:policy-grad-theorem}
Let $\rho \in \mathcal{P}(S)$. For all $\pi',\pi \in \mathcal{P}(A|S)$,
\begin{equation*}
\lim_{\varepsilon \searrow 0 }\frac{1}{\varepsilon}(V^{\pi + \varepsilon (\pi' - \pi)}-V_{\pi})(\rho) = \left\langle A_\pi, \pi'-\pi \right\rangle_{d^{\pi}_{\rho}}\,.
\end{equation*}
\end{theorem}
Since the first-order variation of the value function is evaluated along the family of flat geodesics $\pi_\varepsilon = \pi + \varepsilon (\pi' - \pi)$, it is natural to expect that the map $\mathcal{P}(A|S) \ni \pi' \mapsto \left\langle A_\pi, \pi'-\pi \right\rangle_{d^{\pi}_{\rho}}$ is flat geodesically convex.
\subsection{TRPO}
Lemma~\ref{lemma perf diff} almost provides a first-order expansion of $V_{\pi}(\rho)$ along flat geodesics but with the dependence on $d_\rho^{\pi'}$, making \eqref{eqn perf diff} difficult to optimize directly and no simpler than the original objective \eqref{eq:objective-for-pi}. As an initial step toward addressing this difficulty, \cite{schulman-15} introduced the following lower bound.
\begin{proposition}
\label{thm:perf diff lower bdd max TV^2}
Let $\rho \in \mathcal{P}(S)$. For all $\pi', \pi \in \mathcal{P}(A|S)$, 
\begin{equation}
\label{eqn perf diff lower bdd max TV^2}
(V_{\pi'}-V_{\pi})(\rho) \geq \Big\langle \tfrac{\mathrm{d}\pi'}{\mathrm{d}\pi}A_\pi, \pi \Big\rangle_{d^{\pi}_{\rho}}- \tfrac{8\gamma\|r\|_{B_b(S\times A)}}{(1-\gamma)^3}\max_{s \in S} \operatorname{TV}^2(\pi'(\cdot|s),\pi(\cdot|s)).
\end{equation}
\end{proposition}
From \eqref{eqn perf diff lower bdd max TV^2}, by Pinsker's inequality, i.e., $\operatorname{TV}^2(\pi'(\cdot|s),\pi(\cdot|s)) \leq \tfrac{1}{2}\operatorname{KL}(\pi'(\cdot|s),\pi(\cdot|s))$,
it follows that
\begin{equation}
\label{eqn perf diff lower bdd max KL}
(V_{\pi'}-V_{\pi})(\rho) \geq \Big\langle \tfrac{\mathrm{d}\pi'}{\mathrm{d}\pi}A_\pi, \pi \Big\rangle_{d^{\pi}_{\rho}} - \tfrac{4\gamma\|r\|_{B_b(S\times A)}}{(1-\gamma)^3}\max_{s \in S} \operatorname{KL}(\pi'(\cdot|s),\pi(\cdot|s)).
\end{equation}
Bound \eqref{eqn perf diff lower bdd max KL} implies that increasing the right-hand side leads to an improvement in $V_{\pi}(\rho)$, provided that $\max_{s \in S} \operatorname{KL}(\pi'(\cdot|s),\pi(\cdot|s))$ is controlled. In practice, though, this objective is difficult to optimize, as evaluating the maximum KL divergence across all states is computationally demanding for even moderate $S$, and becomes intractable when $S$ is large or continuous. In TRPO \citep{schulman-15}, the lower bound in \eqref{eqn perf diff lower bdd max KL} is indirectly optimized by solving the constrained optimization problem
\begin{equation}
\label{TRPO}
\max_{\pi' \in \mathcal{P}(A|S)} \Big\langle \tfrac{\mathrm{d}\pi'}{\mathrm{d}\pi}A_\pi, \pi \Big\rangle_{d^{\pi}_{\rho}},\text{ such that } \int_S \operatorname{KL}(\pi'(\cdot|s),\pi(\cdot|s))d_\rho^{\pi}(\mathrm{d}s) \leq \varepsilon,
\end{equation}
where $\varepsilon > 0$ is a hyperparameter that limits the KL divergence between successive policies. \cite{achiam17a} derived the following relaxation of \eqref{eqn perf diff lower bdd max TV^2} for finite state and action spaces. The result also holds when $S$ and $A$ are Polish.
\begin{proposition}
\label{thm:perf diff lower bdd integrated TV}
Let $\rho \in \mathcal{P}(S)$. For all $\pi', \pi \in \mathcal{P}(A)$, 
\begin{equation}
\label{eqn perf diff lower bdd integrated TV}
(V_{\pi'}-V_{\pi})(\rho) \geq \Big\langle \tfrac{\mathrm{d}\pi'}{\mathrm{d}\pi}A_\pi, \pi \Big\rangle_{d^{\pi}_{\rho}}-\tfrac{4\gamma\|r\|_{B_b(S\times A)}}{(1-\gamma)^3}\int_S \operatorname{TV}(\pi'(\cdot|s),\pi(\cdot|s))d_\rho^{\pi}(\mathrm{d}s).
\end{equation}
\end{proposition}

\subsection{PPO}
\label{sec:ppo}
Because each TRPO update requires solving the constrained optimization problem \eqref{TRPO}, the method is computationally expensive and difficult to scale to large reinforcement learning tasks. To overcome this limitation, \cite{schulman2017} introduced the \textit{clipped surrogate objective}, which forms the basis of the PPO algorithm. PPO is inspired by TRPO in that its clipping mechanism discourages updates that produce large ratios $\frac{\mathrm d \pi'}{\mathrm d\pi}$, while being far simpler to implement and relying only on first-order optimization.
The clipped surrogate objective can be derived from \eqref{TRPO} by removing the KL constraint and applying the elementary inequality $a \geq \min\{a,b\}$ giving $\big\langle \tfrac{\mathrm{d}\pi'}{\mathrm{d}\pi}A_\pi, \pi \big\rangle_{d^{\pi}_{\rho}}\geq \big\langle f_\varepsilon\big(\tfrac{\mathrm{d}\pi'}{\mathrm{d}\pi},A_\pi\big), \pi \big\rangle_{d^{\pi}_{\rho}}$,
where $f_\varepsilon(x,y) \coloneqq \min\left(xy,\text{clip}_{1-\varepsilon}^{1+\varepsilon}(x)y\right)$ and $\varepsilon \in (0,1)$ is a hyperparameter that removes the incentive for pushing the new policy $\pi'$ away from the old policy $\pi$. Importantly, the clipped surrogate does \textit{not} generally guarantee policy improvement, since it is not a lower bound on $(V_{\pi'}-V_{\pi})(\rho)$. Although heuristic in PPO, the clipping technique is designed to effectively bound the TV distance, as noted by \cite{queeney2021generalized}. Indeed,
\begin{equation*}
\int_S \operatorname{TV}(\pi'(\cdot|s),\pi(\cdot|s))d_\rho^{\pi}(\mathrm{d}s) = \frac{1}{2} \int_S \int_A \left|\tfrac{\mathrm{d}\pi'}{\mathrm{d}\pi}(a|s) - 1\right|\pi(\mathrm{d}a|s)d_\rho^{\pi}(\mathrm{d}s).
\end{equation*}
Therefore, the lower bound in \eqref{eqn perf diff lower bdd integrated TV} provides a motivation for PPO's strategy of constraining the probability ratio $\left|\frac{\mathrm{d}\pi'}{\mathrm{d}\pi}(a|s) - 1\right| \leq \varepsilon$, as doing so effectively imposes a penalty in terms of TV distance. Building on this observation, one could in principle define the update rule
\begin{equation}
\label{TV PPO}
\pi^{n+1}= \argmax_{\pi \in \mathcal P(A|S)}\Big[\Big\langle \tfrac{\mathrm{d}\pi}{\mathrm{d}\pi^n}A_{\pi^n}, \pi^n \Big\rangle_{d^{\pi^n}_{\rho}}- \frac{1}{\tau}\int_S \operatorname{TV}(\pi(\cdot|s),\pi^n(\cdot|s))d_\rho^{\pi^n}(\mathrm{d}s)\Big],
\end{equation}
which penalizes the policy ratio update through the TV distance with strength $\tau^{-1} > 0$. If $\tau > 0$, then \eqref{TV PPO} ensures monotonic improvement by applying Lemma \ref{lemma perf diff} to $\pi'=\pi^{n+1}$ and $\pi=\pi^n$, but it suffers from a key drawback. Determining $\pi^{n+1}$ requires prior knowledge of whether the ratio $\frac{\mathrm{d}\pi^{n+1}}{\mathrm{d}\pi^n}$ lies above or below $1$. 
This arises from the first order condition for~\eqref{TV PPO}, see Lemma~\ref{first order tv ppo} and further discussion in Appendix \ref{appendix:TV PPO}.
\section{Main results}
\label{section:main-results}
In this section, we introduce our FR-PPO method and state its principal convergence guarantees. We begin with the new surrogate lower bound on the performance difference involving $\operatorname{TV}^2$.

\subsection{Surrogate lower bounds in TV and FR$^2$: motivation for FR--PPO}

For $\pi,\pi'$, let $\xi_s:=\pi'(\cdot|s)-\pi(\cdot|s)$ and define the auxiliary policy $\widehat\pi_{\pi,\pi'}(\mathrm da|s) := \frac{|\xi_s|(\mathrm da)}{|\xi_s|(A)}$, if $|\xi_s|(A)>0$, and $\widehat\pi_{\pi,\pi'}(\mathrm da|s) := p(\mathrm da)$, otherwise, where $p\in\mathcal P(A)$ is arbitrary. Let $\widehat P_{\pi,\pi'}(\mathrm ds'|s) := \int_A P(\mathrm ds'|s,a)\widehat\pi_{\pi,\pi'}(\mathrm da| s)$. Define the auxiliary occupancy measure $\widetilde d_{\rho}^{\pi,\pi'}$ by
\begin{equation*}
\widetilde d_{\rho}^{\pi,\pi'} :=
(1-\gamma)^2 \sum_{k,m\geq 0} \gamma^{k+m} \rho P_\pi^k \widehat P_{\pi,\pi'} P_{\pi'}^m.
\end{equation*}
Note $\widetilde d_{\rho}^{\pi,\pi'}$ is the occupancy distribution obtained by following policy $\pi$ for a geometrically distributed number of steps, taking one transition using the normalized TV difference between $\pi'$ and $\pi$, and then following $\pi'$ for another geometrically distributed number of steps. The normalization $(1-\gamma)^2\sum_{k,m\geq0}\gamma^{k+m}=1$ shows that this is indeed a probability distribution.

\begin{theorem}[TV$^2$ lower bound with auxiliary occupancy measure]
\label{thm:tv2_mixed}
For any $\rho\in\mathcal P(S)$ and any $\pi,\pi'\in\mathcal P(A|S)$,
\begin{align}
\label{eq:tv2_mixed_reward}
(V_{\pi'}&-V_\pi)(\rho) \ge \langle A_\pi, \pi'-\pi \rangle_{d^\pi_\rho}\\
\nonumber
&-\tfrac{4\gamma\|r\|_{B_b(S \times A)}}{(1-\gamma)^3}
\left[
\int_S \mathrm{TV}^2\big(\pi'(\cdot|s),\pi(\cdot|s)\big)d^\pi_\rho(\mathrm ds)
+
\int_S \mathrm{TV}^2\big(\pi'(\cdot|s),\pi(\cdot|s)\big)\tilde d_{\rho}^{\pi,\pi'}(\mathrm ds)
\right].
\end{align}
\end{theorem}
\paragraph{Proof ideas.} The proof is rather technical and is provided in Appendix \ref{proofs perf diff and tight bound}. Here, we highlight the main ideas. Starting from the performance difference lemma, the obstacle is that the advantage term is averaged under the new occupancy measure \(d_\rho^{\pi'}\) rather than the old occupancy measure \(d_\rho^\pi\). We expand the occupancy mismatch via a resolvent style decomposition of the transition kernels \(P_{\pi'}^n-P_\pi^n\), isolate one factor of \(P_{\pi'}-P_\pi\), and apply the elementary $\operatorname{TV}$ bound between $\pi'$ and $\pi$ twice: once at the current step and once after propagating through future transitions. This introduces the auxiliary mixed occupancy measure \(\tilde d_\rho^{\pi,\pi'}\). The theorem therefore controls policy improvement by two occupancy averaged \(\operatorname{TV}^2\) terms. 

Under a standard concentrability condition, in the spirit of classical RL analyses (see, e.g., \cite{NIPS2007_da0d1111, JMLR:v9:munos08a, scherrer2013performanceboundspolicysearch}), used only to absorb the auxiliary occupancy measure into the reference measure, we obtain the following corollary.

For a given pair $(\pi,\pi')$, define $C_\rho(\pi,\pi') := \big\|\frac{\mathrm d\tilde d_{\rho}^{\pi,\pi'}}{\mathrm d d_\rho^\pi}\big\|_{L^\infty(S,d_\rho^\pi)}$, with $C_\rho(\pi,\pi')=+\infty$, if $\widetilde d_{\rho}^{\pi,\pi'}$ is not absolutely continuous with respect to $d_\rho^\pi$, and $C_\rho(\pi,\pi')<+\infty$, otherwise.
\begin{corollary}
\label{thm lower bound perf diff}
Let $\rho \in \mathcal{P}(S)$ and $\pi',\pi \in \mathcal{P}(A|S)$. Assume that $C_\rho(\pi,\pi')<+\infty$.
Then 
\begin{equation*}
(V_{\pi'}-V_\pi)(\rho) \ge 
\Big\langle \tfrac{\mathrm{d}\pi'}{\mathrm{d}\pi}A_{\pi}, \pi \Big\rangle_{d^{\pi}_{\rho}} - \tfrac{4\gamma\|r\|_{B_b(S \times A)}}{(1-\gamma)^3}\left(1+C_\rho\right)
\int_S \mathrm{TV}^2\left(\pi'(\cdot|s),\pi(\cdot|s)\right)d^\pi_\rho(\mathrm ds).
\end{equation*}
\end{corollary}
\begin{remark}[Relation to classical concentrability assumptions]
Classical concentrability assumptions in, e.g., \citep{NIPS2007_da0d1111, JMLR:v9:munos08a,scherrer2013performanceboundspolicysearch}
are typically stated with respect to a fixed reference distribution $\nu\in\mathcal P(S)$. In contrast, our condition is a relative concentrability condition. It requires that the auxiliary distribution does not place excessive mass on states visited under the baseline occupancy $d^\pi_\rho$. 
Note that this condition does not stem from a policy update step. Instead, it represents a coverage property of $(\rho,P)$, which is guaranteed under standard conditions; see Appendix \ref{proofs perf diff and tight bound}.
\end{remark}
Note that the lower bound in Theorem \ref{thm lower bound perf diff} is already tighter than Theorem \ref{thm:perf diff lower bdd max TV^2} since we do not maximize over all states in $\operatorname{TV}$, and compared to Theorem \ref{thm:perf diff lower bdd integrated TV} it employs the tighter $\operatorname{TV}^2$ metric. 

Building on the discussion in Subsection \ref{subsection:geod-convexity} about restricting to measures with finite $\operatorname{FR}^2$ distance between squared densities, we define the admissible policy class by $\mathcal P_\lambda(A|S) := \left\{\pi \in \mathcal P(A|S): \pi(\cdot|s) \in \mathfrak C_\lambda, \text{ for all } s \in S\right\}$.
For example, if $A$ is of finite cardinality and $\lambda$ is chosen to be uniform on $A$, then $\mathcal P_\lambda(A|S) = \mathcal P(A|S)$. If $\lambda \in \mathcal{P}(A)$, 
by the Cauchy-Schwarz inequality, 
\begin{equation*}
\int_S \operatorname{TV}^2(\pi'(\cdot|s), \pi(\cdot|s))d_{\rho}^{\pi}(\mathrm{d}s) \leq \frac{1}{16}\int_S \operatorname{FR}^2(\pi'(\cdot|s)^2, \pi(\cdot|s)^2)d_{\rho}^{\pi}(\mathrm{d}s),
\end{equation*}
for all $\pi, \pi' \in \mathcal{P}_{\lambda}(A|S)$.
Hence, we obtain

\begin{corollary}
\label{cor lower bound FR^2}
Let $\rho \in \mathcal{P}(S)$ and $\pi', \pi \in \mathcal{P}_{\lambda}(A|S)$.
Assume that $C_\rho(\pi,\pi')<+\infty$.
Then
\begin{equation*}
(V_{\pi'} - V_{\pi})(\rho) \geq \Big\langle \tfrac{\mathrm{d}\pi'}{\mathrm{d}\pi}A_\pi, \pi \Big\rangle_{d^{\pi}_{\rho}} - \tfrac{\gamma\|r\|_{B_b(S \times A)}}{4(1-\gamma)^3}\left(1+C_\rho\right)\int_S \operatorname{FR}^2(\pi'(\cdot|s)^2, \pi(\cdot|s)^2)d_{\rho}^{\pi}(\mathrm{d}s),
\end{equation*}
where $\pi(\cdot|s)^2 := \left(\frac{\mathrm{d}\pi}{\mathrm{d}\lambda}\right)^2(\cdot|s)$, for all $s \in S$.
\end{corollary}
Note that the FR term is not introduced as an external regularizer. It appears because the new lower bound already contains an occupancy averaged $\operatorname{TV}^2$ term, and that quantity is upper bounded by a lifted $\operatorname{FR}^2$ term. This is the point at which a PPO-style surrogate becomes compatible with MD tools. 

\subsection{Convergence of FR--PPO for direct parametrization}
\label{sec:FR-PPO-direct-param}
Motivated by Corollary \ref{cor lower bound FR^2}, we fix $\tau > 0$ and consider the scheme given by 
\begin{equation}
\label{eq:pointwise_min}
\pi^{n+1}= \argmax_{\pi \in \mathcal P_\lambda(A|S)}\Big[\Big\langle \tfrac{\mathrm{d}\pi}{\mathrm{d}\pi^n}A_{\pi^n}, \pi^n \Big\rangle_{d^{\pi^n}_{\rho}} - \tfrac{1}{8\tau}\int_S \operatorname{FR^2}(\pi(\cdot|s)^2,\pi^n(\cdot|s)^2)d_\rho^{\pi^n}(\mathrm{d}s)\Big].
\end{equation}
Lemma~\ref{lemma:well-posedness-FR-PPO} tells us that if $\pi^n \in \mathcal{P}_\lambda(A|S)$ then the argmax in~\eqref{eq:pointwise_min} is single valued and in $\mathcal{P}_\lambda(A|S)$.  
Most importantly, the concentrability condition $C_\rho(\pi,\pi') < +\infty$ required by Corollaries~\ref{thm lower bound perf diff} and~\ref{cor lower bound FR^2} is not required for the convergence of scheme~\eqref{eq:pointwise_min}. Rather, its purpose is to motivate the FR penalty and show when it gives an explicit surrogate lower bound for policy improvement. The convergence proof instead relies on the pointwise optimality of~\eqref{eq:pointwise_min}, the Bregman proximal inequality (Lemma~\ref{lem three point}), and the exact performance difference identity (Lemma~\ref{lemma perf diff}).

When concentrability does hold, Corollary \ref{cor lower bound FR^2} provides the additional interpretation that \eqref{eq:pointwise_min} optimizes a performance lower bound. In particular, with $\kappa_\rho := \frac{\gamma|r|_{B_b(S\times A)}}{4(1-\gamma)^3}(1+C_\rho)$, the objective in \eqref{eq:pointwise_min} is itself a lower bound surrogate whenever $8\tau\kappa_\rho \leq 1$. Thus, concentrability establishes the connection between the new $\operatorname{TV}^2$ lower bound and the FR objective, but it is not needed for monotonicity or convergence of the exact update.

\begin{theorem}[Policy improvement]
\label{thm policy improvement}
Let $V_n := V_{\pi^n}$ for $n\in \mathbb N$ and $(\pi^n)_{n\in \mathbb N_0} \subset \mathcal{P}_\lambda(A|S)$ be given by \eqref{eq:pointwise_min}.
If $\tau > 0$, then for any $\rho \in \mathcal P(S)$,
\[
V_{n+1}(\rho) \geq  V_n(\rho)\, \text{ for all } n\in \mathbb N_0.
\]
\end{theorem}
\begin{theorem}[Convergence of FR--PPO under direct parametrization]
\label{thm:convergence_fr_steps}
Define $D_0(\pi):=\int_S [\operatorname{FR}^2(
(\pi(\cdot|s)^2,\pi^0(\cdot|s)^2 ) + 8\tau(V^*-V_0)(s)]
d^{\pi}_\rho(\mathrm{d}s)$. Let $V_n := V_{\pi^n}$ for $n\in \mathbb N$, $\tau > 0$ and $(\pi^n)_{n\in \mathbb N_0} \subset \mathcal{P}_\lambda(A|S)$ be given by \eqref{eq:pointwise_min}. Then, for any $\rho \in \mathcal P(S)$ and $\pi \in \mathcal P_\lambda(A|S)$,
\begin{equation}
\label{eq:sublinear_convergence}
(V_{\pi} - V_n)(\rho) \leq  \tfrac 1{8n(1-\gamma)\tau}D_0(\pi).
\end{equation}
Suppose additionally that $\pi^\ast \in \mathcal P_\lambda(A|S)$ and that instead of the constant step size $\tau$ in \eqref{eq:pointwise_min} is replaced with $(\tau_n)_n$ satisfying $\tau_0>0$ $\tau_{n+1} \geq \gamma^{-1}\tau_n$ and let $\xi := (1-\gamma)^{-1}\|\frac{\mathrm d d^{\pi^*}_\rho}{\mathrm d \rho}\|_{B_b(S)} < \infty$.
Then
\begin{equation*}
    V^*-V_n \leq \Big(1-\tfrac{1}{\xi}\Big)^n\Big(V^*-V_0 + \tfrac{1}{\tau_0\gamma}\int_S \operatorname{FR}^2(
(\pi^*(\cdot|s)^2,\pi^0(\cdot|s)^2 )d^{\pi^*}_\rho(\mathrm{d}s)\Big).
\end{equation*}
\end{theorem}
In the second statement of Theorem~\ref{thm:convergence_fr_steps} we
assume that $\xi < \infty$, which is the case, for example, if $\rho(\mathrm{d}s) >0$ for all $s \in S$. This is commonly used in the literature on policy gradient methods \cite{zhang_csaba,Wang2020Neural} and the NPG convergence analysis \cite{cayci2024natural,xiao2022pgrates}.

\subsection{Continuous time limit of FR--PPO}

One may naturally ask the question as to what the continuous-time limit of the update~\eqref{eq:pointwise_min} is, as $\tau \searrow 0$.
With appropriate interpolation of the update sequence $(\frac{\mathrm d \pi^n}{\mathrm d \lambda})_{n\in \mathbb N_0}$, we show in Proposition~\ref{prop:ct_projected_global} that in the limit we obtain the flow 
\begin{equation}
\label{eq:intro_projected_lift_flow}
\partial_t \tfrac{\mathrm d \pi_t}{\mathrm d \lambda}(\cdot|s) =\operatorname{Proj}^{L^2_\lambda(A)}_{T_C\big(\tfrac{\mathrm d \pi_t}{\mathrm d \lambda}(\cdot|s)\big)}\Big(\tfrac{1}{1-\gamma} A_{\pi_{t}}(s,\cdot)\Big),\, \text{for a.e. }t>0,\, \tfrac{\mathrm d \pi_0}{\mathrm d \lambda}(\cdot|s) = \tfrac{\mathrm d \pi^0}{\mathrm d \lambda}(\cdot|s)\in C,
\end{equation}
for every $s \in S$. To see what is happening we have to understand the projection operator. To that end, we define the subset of densities in $L_\lambda^2(A)$ by
\begin{equation*}
C:=\Big\{v\in L^2_{\lambda}(A):v\geq0\ \lambda\text{-a.e.},\
\int_Av(a)\lambda(\mathrm da)=1\Big\}
\end{equation*}
and note that, by Lemma~\ref{lem:ct_density_cones}, its tangent cone at any $v \in C$ is given by
\begin{equation*}
T_C(v)=\Big\{z\in {L^2_\lambda(A)}:\int_Az\mathrm d\lambda=0,\ z\geq0\ \text{on }\{v=0\}\Big\}.
\end{equation*}
For each $v\in C$, the set $T_C(v)$ is nonempty, closed and convex in $L^2_\lambda(A)$. Hence, for each $v\in C$, 
\[
\operatorname{Proj}^{L^2_\lambda(A)}_{T_C(v)}(\varphi)=\argmin_{\phi\in T_C(v)}\|\phi-\varphi\|_{L^2_\lambda(A)}\, \text{for all $\varphi \in L^2_\lambda(A)$}
\]
is the orthogonal projection from $L^2_\lambda(A)$ to its convex subset $T_C(v)$. 

Furthermore, in Proposition~\ref{prop:ct_long_time}, we prove that the value function $t\mapsto V_{\pi_t}$ is non-decreasing along the flow \eqref{eq:intro_projected_lift_flow} and the gap to any fixed reference comparator policy decreases at $O(1/t)$ rate. 
Finally, in Remark~\ref{rem:ct_matching}, we discuss why the flow \eqref{eq:intro_projected_lift_flow} differs from the Fisher--Rao flow studied in~\cite{kerimkulov2024fisherraogradientflowentropyregularised}.

\subsection{Projected Natural Policy Gradient}
\label{sec:npg}
NPG, as introduced by \cite{kakadenatural}, is known to yield updates equivalent to policy mirror descent under softmax parametrization in Euclidean state-action spaces \citep[Lemma 15]{agarwal}. In this section, we introduce Projected NPG (Proj--NPG). We demonstrate that for a suitable class of parametrized policies, Proj--NPG recovers the FR--PPO update \eqref{eq:pointwise_min}. Furthermore, we establish convergence rates for Proj--NPG under compatible function approximation. 

The policy class studied here is chosen because it is isometric to the parameter space under a feature-induced seminorm, making the Fisher–Rao proximal step exactly representable as projected NPG. Similar to log-linear, this parametrization admits a provable variational characterization, enabling analysis via MD tools.
In general, clean theoretical analysis requires matching the policy class to the update algorithm: log-linear policies suit MD-type updates, linear policies suit $\operatorname{FR}^2$ updates, and parametrized transport maps from noise distributions suit Wasserstein updates \citep{zhu2026wassersteinproximalpolicygradient}.

Let $(\mathbb{H}, \langle\cdot, \cdot\rangle_{\mathbb{H}})$ be a Hilbert space representing the parameter space, and let $\phi:S \times A \to \mathbb H$ be a fixed feature basis. We define the feasible parameter set $\Theta \subset \mathbb H$ as 
\begin{equation*}
\Theta := \big\{\theta \in \mathbb H\vert \langle \theta,\phi(s,a) \rangle_{\mathbb H} \geq 0 \ \lambda\text{-a.e.}, \int_A \langle \theta,\phi(s,a) \rangle_{\mathbb H} \lambda(\mathrm{d}a) = 1, \text{ for all } s\in S\big\}.    
\end{equation*}
As $\Theta$ is a subset of a linear space created by applying linear constraints, it is convex.Correspondingly, we define the class of parametrized policies $\Pi$ as
\begin{equation}
\label{def:Pi_class}
\Pi \coloneqq \left\{\langle\theta, \phi(s,a)\rangle_{\mathbb{H}}\big\vert \theta \in \Theta\right\} \subset \mathcal{P}_\lambda(A|S).    
\end{equation}
Since $\Theta$ is a convex subset of $\mathbb H$, we can easily see that $\Pi$ is a convex subset of $\mathcal P_\lambda(A|S)$.
Restricting policies to this class naturally leads to linear function approximation for the advantage function, $A_{\pi_\theta}(s,a) - \langle w, \phi(s,a)\rangle_{\mathbb H}$ for some $w \in \mathbb H$, a standard formulation in RL \citep[Sec. 9.3]{sutton2018reinforcement}. This class includes common policy parametrizations such as tabular and Gaussian mixture policies; see Appendix \ref{app:D}.

The Proj--NPG algorithm operates within the geometry induced by the positive semi-definite operator 
\begin{equation*}
G(\theta) \coloneqq \int_S\int_A \phi(s,a) \otimes \phi^*(s,a)\lambda(\mathrm{d}a)d_\rho^{\pi_\theta}(\mathrm{d}s).
\end{equation*}
Here, $\phi^* \in \mathbb H^*$ acts on any $u \in \mathbb H$ via the inner product $\phi^*(s,a)u = \langle u, \phi(s,a)\rangle_{\mathbb{H}}$. Thus, the tensor product action is given by $\left(\phi(s,a) \otimes \phi^*(s,a)\right)u = \phi(s,a)\langle u, \phi(s,a)\rangle_{\mathbb{H}}$.
The update rule of the Proj--NPG algorithm is given by
\begin{equation}
\label{eq:proj-NPG-parameters}
\theta^{n+1} \!\!= \Gamma_{\Theta}^{G(\theta^n)} \!\left(\theta^n \!+\! \eta G^{\dagger}(\theta^n)\nabla_\theta V_{\pi_{\theta^n}}(\rho)\right),  \theta^0 \in \Theta,
\end{equation}
where $\eta > 0$ is the step size, $G^\dagger$ denotes the Moore--Penrose pseudo-inverse of $G$, and $\Gamma_{\Theta}^{G(\theta)}:\mathbb H \to \Theta$ is the projection operator under the semi-norm $\|\cdot\|_{G(\theta)}^2 \coloneqq \langle \cdot,G(\theta) \cdot \rangle_{\mathbb H}$, defined as $\Gamma_{\Theta}^{G(\theta)}(\Tilde{\theta}) = \argmin_{\bar{\theta} \in \Theta}\|\bar{\theta}-\Tilde{\theta}\|_{G(\theta)}$.
A key property of this formulation is the isometry between the parameter space $\Theta$ and the policy space $\Pi$ (Lemma~\ref{lemma:policy-isometry}).
Specifically, the map $\Theta \ni \hat\theta \mapsto \pi_{\hat\theta} \in \Pi$ satisfies 
\begin{equation*}
\Big\|\tfrac{\mathrm{d}\pi_{\hat\theta}}{\mathrm{d}\lambda}\Big\|_{{L^2_\lambda(A)}\times d_\rho^{\pi_\theta}} = \big\|\hat\theta\big\|_{G(\theta)},
\end{equation*}
where $\|\cdot\|_{{L^2_\lambda(A)}\times d_\rho^{\pi_\theta}}^2 :=\int_S \int_A |\cdot|^2 \lambda(\mathrm{d}a) d_\rho^{\pi_\theta}(\mathrm{ds})$. Leveraging this isometry, we show the connection between \eqref{eq:proj-NPG-parameters} and the FR--PPO update \eqref{eq:pointwise_min} by adopting the notion of compatible function approximation \citep{sutton1999policy}. Consider $L^{\pi_{\theta}}: \mathbb H \rightarrow \mathbb R$ defined by
\begin{equation}
\label{eq:loss_A}
L^{\pi_\theta}(w) := \frac{1}{2}\int_S \int_A \left|A_{\pi_\theta}(s,a) - \langle w,  \phi(s,a)\rangle_{\mathbb H}\right|^2\lambda(\mathrm{d}a)d_{\rho}^{\pi_{\theta}}(\mathrm{d}s).
\end{equation}
Fix \(\theta\in\Theta\). We say that the feature map \(\phi\) is \textit{compatible} with $A_{\pi_\theta}$ if there exists $w^*(\theta) \in\arg\min_w L^{\pi_\theta}(w)$ such that $A_{\pi_\theta}(s,a)=\langle w^*(\theta),\phi(s,a)\rangle_{\mathbb H}$ for $d_\rho^{\pi_\theta}\otimes\lambda\text{-a.e. }(s,a)$. The first-order optimality condition $\nabla_w L^{\pi_\theta}(w) = 0$ for $\theta=\theta^n$ yields the minimizer $w^*(\theta^n) = (1-\gamma)G^\dagger(\theta^n)\nabla_\theta V_{\pi_{\theta^n}}(\rho)$.
If the approximation of $A_{\pi_{\theta^n}}$ by the features $\phi$ is exact (i.e., $w^*(\theta^n)$ achieves zero in \eqref{eq:loss_A}), then setting $\tau = \eta(1-\gamma)^{-1}$ yields the following result.
\begin{proposition}[Variational characterization of Proj--NPG]
\label{prop:variational-characterization-Proj--NPG}
Let $\theta^n \in \Theta$ be the parameter with corresponding policy $\pi_{\theta^n} \in \Pi$. If $\theta^{n+1}$ is generated by the Proj--NPG update \eqref{eq:proj-NPG-parameters}, then the corresponding policy $\pi_{\theta^{n+1}}$ is given by
\begin{equation}
\label{eq:pointwise_min_parametrized}
\pi_{\theta^{n+1}}= \argmax_{\pi \in \Pi}\Big[\Big\langle \tfrac{\mathrm{d}\pi}{\mathrm{d}\pi_{\theta^{n}}}A_{\pi_{\theta^n}}, \pi_{\theta^n} \Big\rangle_{d^{\pi_{\theta^n}}_{\rho}} - \tfrac{1}{8\tau}\int_S \operatorname{FR^2}(\pi(\cdot|s)^2,\pi_{\theta^n}(\cdot|s)^2)d_\rho^{\pi_{\theta^n}}(\mathrm{d}s)\Big].
\end{equation}
\end{proposition}
Using Proposition \ref{prop:variational-characterization-Proj--NPG}, following the reasoning of Theorem~\ref{thm:convergence_fr_steps}, we obtain the following convergence.

\begin{theorem}[Convergence of Proj--NPG]
\label{thm:sublinear_convergence_proj_NPG}
Define $D_{\pi_{\theta^0}}(\pi):= \int_S \big(\operatorname{FR}^2(
(\pi(\cdot|s)^2,\pi_{\theta^0}(\cdot|s)^2 )  + 8\tau(V^*-V_0)(s)\big)
d^{\pi}_\rho(\mathrm{d}s)$. Let $V_n := V_{\pi_{\theta^n}}$ for $n\in \mathbb N$ and $(\pi_{\theta^n})_{n\in \mathbb N_0} \subset \Pi$ be given by \eqref{eq:pointwise_min_parametrized}.
If $\tau > 0$, then for any $\rho \in \mathcal P(S)$ and $\pi \in \Pi$,
\begin{equation*}
(V_{\pi} - V_n)(\rho) \leq  
\tfrac 1{8n(1-\gamma)\tau}D_{\pi_{\theta^0}}(\pi).
\end{equation*}
Suppose instead that constant step size $\tau$ in \eqref{eq:pointwise_min_parametrized} is replaced with $(\tau_n)_n$ satisfying $\tau_0 \geq \frac{1-\gamma}{\gamma}\int_S \operatorname{FR}^2(
(\pi^*(\cdot|s)^2,\pi_{\theta^0}(\cdot|s)^2 )
d^{\pi^*}_\rho(\mathrm{d}s)$, $\tau_{n+1} \geq \gamma^{-1}\tau_n$ and let $\xi := (1-\gamma)^{-1}\|\frac{\mathrm d d^{\pi^*}_\rho}{\mathrm d \rho}\|_{B_b(S)} < \infty$. Then
\begin{equation*}
    V^*-V_n \leq \Big(1-\tfrac{1}{\xi}\Big)^n\Big(V^*-V_0 + \tfrac{1}{1-\gamma}\Big).
\end{equation*}
\end{theorem}
In the case where compatible function approximation is not exact (i.e., $w^*(\theta^n)$ does not achieve zero in \eqref{eq:loss_A}), then we employ Approximate Proj--NPG (AProj--NPG):
\begin{equation}
\label{eq:aproj-NPG-parameters}
    \theta^{n+1} = \Gamma_{\Theta}^{G(\theta^n)}\left(\theta^n + \tau \hat{w}(\theta^n)\right), \ \theta^0 \in \Theta, \ \hat{w}(\theta^n) \in \argmin_{w \in \mathbb H} L^{\pi_{\theta^n}}(w).
\end{equation}
By applying the same argument as in Proposition \ref{prop:variational-characterization-Proj--NPG}, we arrive at the analogous result.
\begin{proposition}[Variational characterization of AProj--NPG]
\label{prop:variational-characterization-AProj--NPG}
Let $\theta^n \in \Theta$ be the parameter with corresponding policy $\pi_{\theta^n} \in \Pi$. If $\theta^{n+1}$ is generated by the AProj--NPG update \eqref{eq:aproj-NPG-parameters}, then the corresponding policy $\pi_{\theta^{n+1}}$ is given by
\begin{equation}
\label{eq:approx-max_parametrized}
\pi_{\theta^{n+1}} =  \argmax_{\pi \in \Pi}\Big[\Big\langle \tfrac{\mathrm{d}\pi}{\mathrm{d}\pi_{\theta^{n}}}\big\langle \hat{w}(\theta^n), \phi \big\rangle_{\mathbb H}, \pi_{\theta^n} \Big\rangle_{d^{\pi_{\theta^n}}_{\rho}} \!\!\!- \tfrac{1}{8\tau}\int_S \operatorname{FR^2}(\pi(\cdot|s)^2,\pi_{\theta^n}(\cdot|s)^2)d_\rho^{\pi_{\theta^n}}(\mathrm{d}s)\Big].
\end{equation}
\end{proposition}
Note that the exact advantage $A_{\pi_{\theta^n}}$ in~\eqref{eq:pointwise_min_parametrized} is replaced by its approximation $\left\langle \hat{w}(\theta^n), \phi \right\rangle_{\mathbb H}$ in~\eqref{eq:approx-max_parametrized}.
Finally, using Proposition \ref{prop:variational-characterization-AProj--NPG}, we obtain the following convergence result.

\begin{theorem}[Convergence of AProj--NPG]
\label{thm:sublinear_convergence_approx_proj_NPG}
Let $V_n := V_{\pi_{\theta^n}}$ for $n\in \mathbb N$ and $(\pi_{\theta^n})_{n\in \mathbb N_0} \subset \Pi$ be given by \eqref{eq:approx-max_parametrized}. 
If $\tau > 0$, then for any $\rho \in \mathcal P(S)$ and $\pi \in \Pi$,
\begin{equation*}
\begin{aligned}
\min_{k < n}(V_{\pi} - V_k)(\rho)\leq  
\tfrac 1{8n(1-\gamma)\tau}D_{\pi_{\theta^0}}(\pi) + \tfrac{1}{n(1-\gamma)}\tsum_{k=1}^{n-1}\mathcal{E}^\pi_k,
\end{aligned}
\end{equation*}
where $\mathcal{E}^\pi_{n+1} \coloneqq \left\langle A_{\pi_{\theta^n}}-\left\langle \hat{w}(\theta^n), \phi \right\rangle_{\mathbb H}, \pi-\pi_{\theta^{n+1}} \right\rangle_{d^{\pi}_{\rho}}$. 

Suppose additionally that $\pi^\ast \in \Pi$ and that the constant step size $\tau$ in \eqref{eq:pointwise_min_parametrized} is replaced with $(\tau_n)_n$ satisfying $\tau_0 > 0$, $\tau_{n+1} \geq \gamma^{-1}\tau_n$ and let $\xi := (1-\gamma)^{-1}\|\frac{\mathrm d d^{\pi^*}_\rho}{\mathrm d \rho}\|_{B_b(S)} < \infty$. 
Then
\begin{equation*}
\begin{split}
&(V^* - V_n)(\rho) \\
&\leq \big(1-\tfrac1\xi\big)^n  \Big((V^* \! - \!V_0)(\rho) \!+\! \tfrac{1}{\tau_0\gamma}\tint_S \operatorname{FR}^2(
(\pi^*(\cdot|s)^2,\pi_{\theta^0}(\cdot|s)^2 )
d^{\pi^*}_\rho(\mathrm{d}s)\Big)+\tsum_{k=1}^n\big(1-\tfrac1\xi\big)^{n-k}\mathcal E^{\pi^\ast}_k.
\end{split}
\end{equation*}
\end{theorem}
Our proof technique differs fundamentally from prior works such as \cite{Huang2023PPOClipAG} and \cite{agarwal}. Standard analyses typically operate directly on the parameter updates via NPG, requiring smoothness assumptions on the policy with respect to its parameters (e.g., \citep[Lemma 34]{agarwal}). This approach invariably introduces error terms that scale with the step-size of the NPG update. In contrast, we demonstrate that Proj--NPG and its approximated variant AProj--NPG admit a variational representation. This allows us to apply the three-point lemma directly, thereby bypassing the smoothness requirement and eliminating the associated stepsize dependent errors.

\section{Limitations and conclusion}
We introduced FR-PPO, an FR formulation of a PPO-motivated ratio surrogate. The main technical step is an occupancy averaged $\operatorname{TV}^2$ performance lower bound, which is controlled by the $\operatorname{FR}^2$ distance. This yields a smooth Bregman proximal update and gives monotonic improvement and convergence guarantees for exact updates in direct and compatible linear approximation settings. The theory is for idealized updates with exact advantages and occupancy measures. The practical algorithm uses sampled rollouts, minibatches, and neural network optimization. The concentrability condition is part of the motivation for the surrogate lower bound, not a claim about arbitrary policy pairs. The parametrized analysis uses a compatible linear policy class rather than standard log-linear policies. Finally, $\tau$ can require careful tuning. Thus, the experiments show implementability and PPO-level performance, not empirical dominance over PPO.

\section*{Acknowledgments and Disclosure of Funding}
We thank the anonymous reviewers for their insightful comments, which helped improve the paper. R-AL acknowledges the support of the EPSRC Centre for Doctoral Training in Mathematical Modelling, Analysis and Computation (MAC-MIGS), funded by the UK Engineering and Physical Sciences Research Council (grant EP/S023291/1), Heriot-Watt University, and the University of Edinburgh. This work was partly carried out while R-AL was a PhD student at Heriot-Watt University. The paper was finalized during R-AL’s postdoctoral appointment at the RIKEN Center for Advanced Intelligence Project (RIKEN AIP). R-AL also acknowledges support from JSPS KAKENHI grant no JP26K25544 for travel to NeurIPS 2026 to present this work. The authors would also like to thank the Isaac Newton Institute for Mathematical Sciences, Cambridge, for support and hospitality during the programme Bridging Stochastic Control And Reinforcement Learning where work on this paper was undertaken. This work was supported by EPSRC grant no EP/Z000580/1.

\bibliography{references}
\bibliographystyle{abbrv}

\newpage
\appendix

\onecolumn
Here we give a roadmap for the appendices. Appendix \ref{app:A} collects the analytical and MDP preliminaries. Appendix \ref{proofs perf diff and tight bound} derives the performance bounds and the lifted FR surrogate. Appendix \ref{app:C} establishes well-posedness, policy improvement, and convergence for the direct update, while Appendix \ref{app:D} develops the variational characterizations and convergence analysis for parametrized updates. Appendix \ref{sec:continuous_time_comparison} establishes the convergence of the FR--PPO scheme to its continuous-time flow as the step size tends to zero, as well as the convergence of the flow. Appendix \ref{app:F} studies the TV-penalized alternative, and Appendix \ref{sec:frpoo_algorithm} presents implementation details and experiments.
\section{Preliminaries and auxiliary results}
\label{app:A}
\subsection{Notation, definitions, and auxiliary results on measure space}
\label{sec:notation}
Let $\mathbb N_0=\mathbb N\cup\{0\}$. For given normed vector spaces $(X,\|\cdot\|_X)$ and $(Y,\|\cdot\|_Y)$,
we denote by $\mathcal{L}(X,Y)$ the normed vector space of bounded linear operators $T: X\to Y$, equipped with the operator norm
$\|T\|_{\mathcal{L}(X,Y)}=\sup_{\|x\|_X\le 1} \|Tx \|_Y=1$.  For simplicity, we write $\mathcal{L}(X)=\mathcal{L}(X,X)$.  

Let $(E,d)$ denote a Polish space (i.e., a complete separable metric space). We always equip a Polish space with its Borel sigma-field $\mathcal{B}(E)$. For a given measure $\mu$ in $E$, denote by $L^p(E,\mu)$, $p\in [1,\infty)$, the Lebesgue spaces of integrable functions, and by $L^\infty(E,\mu)$ the $\mu$-essential supremum norm. Denote by $B_b(E)$ the space of bounded measurable functions $f:E\rightarrow \mathbb{R}$ endowed with the supremum norm $\|f\|_{B_b(E)}=\sup_{x\in E}|f(x)|$. Denote by $\mathcal{M}(E)$ the Banach space of finite signed measures $\mu$ on $E$ endowed with the total variation norm $\|\mu\|_{\mathcal{M}(A)}=|\mu|(E)$, where $|\mu|$ is the total variation measure. 
Recall that if $\mu=f d\rho$, where $\rho \in \mathcal{M}_+(E)$ is a nonnegative measure and $f\in L^1(E,\rho)$, then $\|\mu\|_{\mathcal{M}(E)}=\|f\|_{L^1(E,\rho)}$.  Denote by $\mathcal{P}(E)\subset\mathcal{M}(E)$ the set of probability measures on $E$. For given $\mu,\mu'\in \mathcal{P}(E)$, we write   $\mu\ll \mu'$ if $\mu$ is absolutely continuous with respect to $\mu'$. Following \cite{santambrogio2015optimal,Carmona2018ProbabilisticTO}, we recall the notion of differentiability on the space of probability measures.
\begin{definition}[Flat differentiability on $\mathcal{P}(E)$]
\label{def:flat-derivative}
We say a function $F:\mathcal{P}(E)\to \mathbb R$ has a flat derivative, if there exists a bounded measurable function $\frac{\delta F}{\delta \mu}:\mathcal{P}(E)\times \mathbb R^d \to \mathbb{R}$, called the flat derivative of $F$, for which there exists $\kappa > 0$ such that for all $(\mu,x) \in \mathcal{P}(E) \times \mathbb R^d$, $\left|\frac{\delta F}{\delta \mu}(\mu,x)\right| \leq \kappa$, and for all $\mu' \in \mathcal{P}(E)$,
\begin{equation*}
\lim_{\varepsilon \searrow 0 }\frac{F(\mu^\varepsilon)- F(\mu)}{\varepsilon} =
\int_{E} \frac{\delta F}{\delta \mu} (\mu, x) (\mu'-\mu) (\mathrm{d}x), \quad \textnormal{with $\mu^\varepsilon =\mu + \varepsilon (\mu' - \mu)$\,,}
\end{equation*}
and $\int_{E} \frac{\delta F}{\delta \mu} (\mu, x) \mu(\mathrm{d}x)=0$.
\end{definition}
We now recall the notion of Bregman divergence on the space of probability measures.
\begin{definition}[Bregman divergence]
For all $\mu,\mu' \in \mathcal{P}(E)$, the Bregman divergence $D_h:\mathcal{P}(E) \times \mathcal{P}(E) \to [0,\infty)$ generated by a convex function $h:\mathcal{P}(E) \to \mathbb R$ is defined by
\begin{equation*}
D_h(\mu'\|\mu):= h(\mu') - h(\mu) - \int_E \frac{\delta h}{\delta \mu}(\mu,x)(\mu'-\mu)(\mathrm{d}x).
\end{equation*}
\end{definition}
To prove our convergence results, we will need the following three-point lemma. For convenience, we provide the proof below. One can also consult the proofs in \cite{korba,kerimkulov2024mirrordescentstochasticcontrol}.
\begin{lemma}[Bregman proximal inequality]
\label{lem three point}
Let $\mathfrak C \subseteq \mathcal P(A)$ be convex. Fix $\nu\in \mathfrak C$ and let $G:\mathcal P(A) \to \mathbb R$ be concave and have a flat derivative $\frac{\delta G}{\delta m}$ on $\mathfrak C$. Suppose that there exists  $\overline m\in \mathfrak C$ such that 
\begin{equation*}
\overline m \in \argmax_{m\in \mathfrak C}\left\{G(m) - D_h(m\|\nu) \right\}.
\end{equation*}
Then, for all $m'\in \mathfrak C$,
\begin{equation}
G(\overline m ) - D_h(\overline m\|\nu) \geq G(m') - D_h(m'\|\nu) + D_h(m'\|\overline m ).
\end{equation}
\end{lemma}
\begin{proof}
For simplicity, write $\langle f,m-n\rangle := \int_A f(a)(m-n)(\mathrm da)$, for any $m,n \in \mathcal{P}(A)$, and recall that
\begin{equation*}
D_h(m\|n)=h(m)-h(n)-\left\langle \frac{\delta h}{\delta m}(n,\cdot),m-n\right\rangle.
\end{equation*}
We assume that the relevant flat derivatives exist and the pairings
below are finite.

Fix $m'\in\mathfrak C$. By convexity of $\mathfrak C$, $m_t := (1-t)\overline m+tm'$ belongs to $\mathfrak C$ for every $t\in[0,1]$. Optimality of $\overline m$ gives
\begin{equation*}
    0 \geq \left.\frac{\mathrm d}{\mathrm dt}\right|_{t=0^+}(G(m_t)-D_h(m_t\|\nu)) = \left\langle \frac{\delta G}{\delta m}(\overline m,\cdot) - \frac{\delta h}{\delta m}(\overline m,\cdot) + \frac{\delta h}{\delta m}(\nu,\cdot), m'-\overline m \right\rangle.
\end{equation*}
Combining this with concavity of $G$, we obtain
\begin{align*}
G(m')-G(\overline m)\leq\left\langle\frac{\delta G}{\delta m}(\overline m,\cdot), m'-\overline m\right\rangle &\leq \left\langle \frac{\delta h}{\delta m}(\overline m,\cdot) - \frac{\delta h}{\delta m}(\nu,\cdot),m'-\overline m\right\rangle \\
&= D_h(m'\|\nu)-D_h(m'\|\overline m)-D_h(\overline m\|\nu),
\end{align*}
where the last equality follows by expanding the three Bregman divergences. Rearranging implies
\begin{equation*}
    G(\overline m)-D_h(\overline m\|\nu) \geq G(m')-D_h(m'\|\nu)+D_h(m'\|\overline m),
\end{equation*}
as claimed.
\end{proof}

\subsection{Notation and auxiliary results for MDPs}
\label{appendix:more-notation-MDP}
Given Polish spaces $(E_1,d_1)$ and $(E_2,d_2)$, we introduce notation for measurable functions $k:E_1\rightarrow \mathcal{M}(E_2)$. 
Denote by $b\mathcal{M}(E_1|E_2)$ the Banach space of bounded signed kernels $k: E_2\rightarrow \mathcal{M}(E_1)$ endowed with the norm $\|k\|_{b\mathcal{M}(E_1|E_2)}=\sup_{x\in E_2}\|k(x)\|_{\mathcal{M}(E_1)}$; that is, $k(U|\cdot): E_2\rightarrow \mathbb{R}$ is measurable for all $U\in \mathcal{M}(E_1)$ and $k(\cdot|x)\in \mathcal{M}(E_1)$ for all $x\in E_2$. For a fixed positive measure $\mu\in \mathcal{M}(E_1)$ and $k\in  b\mathcal{M}(E_1|E_2)$, we write $\mu\ll k$ (resp.~$ k\ll \mu$) if the associated kernel $k_\mu \in b\mathcal{M}(E_1|E_2)$ given by $k_\mu(dy|x)=\mu(dy)$ for all $x\in E_2$ satisfies  $ k_\mu\ll k$ (resp.~$  k\ll k_\mu$). We denote by $b\mathcal{M}_{\mu}(E_1|E_2)$ the space of kernels $k\in  b\mathcal{M}(E_1|E_2)$ such that $k\ll \mu$. We denote by $\mathcal{P}(E_1|E_2)$ (resp.~$\mathcal{P}_\mu(E_1|E_2)$) the  collection  of $P\in b\mathcal{M}(E_1|E_2)$ (resp.~$P\in b\mathcal{M}_\mu(E_1|E_2)$) such that $P(\cdot|x)\in \mathcal{P}(E_1)$ for all $x\in E_2$. For given $\mu\in \mathcal{P}(E_2)$ and $k\in \mathcal{P}(E_1|E_2)$, we define the semidirect product $\mu\otimes k\in \mathcal{P}(E_2\times E_1) $ of $\mu$ with $k$ by $ (\mu\otimes k)(A\times B)=\int_A k(B|x)\mu(dx)$, for all $A\in \mathcal{B}(E_2)$, $B\in \mathcal{B}(E_1)$. 

Every kernel $k\in b\mathcal{M}(E_1|E_2)$ induces bounded linear operators $T_k \in \mathcal{L}(\mathcal{M}(E_2),\mathcal{M}(E_1))$ and $S_k \in \mathcal{L}(B_b(E_1),B_b(E_2))$ defined by
\begin{equation}
\label{eq:kernel-T}
    T_k\mu(\mathrm{d}y)=\mu k(\mathrm{d}y)=\int_{E_2}\mu(\mathrm{d}x)k(\mathrm{d}y|x)
\end{equation}
and 
\begin{equation}
\label{eq:kernel-S}
    S_kf(x)=\int_{E_1}k(\mathrm{d}y|x)f(y),
\end{equation}
respectively. Moreover, by \citep[Ex.2.3 and Prop.3.1]{kunze2011pettis}, we have
\begin{align*}
\|k\|_{b\mathcal{M}(E_1|E_2)}=\sup_{x\in E_2}\underset{\|h\|_{B_b(E_1)}\le 1}{\sup_{h\in B_b(E_1)}}\int_{E_1}h(y)k(\mathrm{d}y|x) =\|S_k\|_{\mathcal{L}(B_b(E_1),B_b(E_2))} =\|T_k\|_{\mathcal{L}(\mathcal{M}(E_2),\mathcal{M}(E_1))},
\end{align*}
where the latter are operator norms. Thus, $b\mathcal{M}(E|E)$ is a Banach algebra with the product defined via composition of the corresponding linear operators; in particular, for a given $k\in b\mathcal{M}(E|E)$,
\begin{equation*}
T_k^n\mu(\mathrm{d}y)=\mu k^n(\mathrm{d}y) = \int_{E^{n}}\mu(\mathrm{d}x_0)k(\mathrm{d}x_1|x_0)\cdots k(\mathrm{d}x_{n-1}|x_{n-2}) k(\mathrm{d}y|x_{n-1}).
\end{equation*}
Notice that if $f\in L^\infty(E_1,\mu)$ and $k\in b\mathcal{M}_{\mu}(E_1|E_2)$, then for all $x\in E_2$,
\begin{align*}
S_kf(x)=\int_{E_1}\mu(\mathrm{d}y)\frac{\mathrm{d}k}{\mathrm{d}\mu}(y|x)f(y) &\le \|f\|_{L^\infty(E_1,\mu)} \left\|\frac{\mathrm{d}k}{\mathrm{d}\mu}(\cdot|x)\right\|_{L^1(E_1,\mu)}\\&\le \|f\|_{L^\infty(E_1,\mu)} \|k\|_{b\mathcal{M}(E_1|E_2)}. \label{ineq:essup_kernel}
\end{align*}

\begin{lemma}[Adjoint relationship between $T_k$ and $S_k$]
\label{lem:adjoint_T_S}
Let $E_1,E_2$ be Polish spaces and let $k\in b\mathcal M(E_1|E_2)$ be a bounded signed kernel from $E_2$ to $\mathcal{M}(E_1)$. Then for every $\mu\in\mathcal M(E_2)$ and every $f\in B_b(E_1)$,
\begin{equation*}
\label{eq:adjoint_pairing}
\int_{E_1} f(y)\,(T_k\mu)(\mathrm dy) = \int_{E_2} (S_k f)(x)\,\mu(\mathrm dx).
\end{equation*}
Equivalently, with the duality pairing $\langle \mu,f\rangle\coloneqq \int f\,\mathrm d\mu$,
\[
\langle T_k\mu,f\rangle=\langle \mu,S_k f\rangle.
\]
\end{lemma}
\begin{proof}
By the definition of $T_k$ and Fubini's theorem (which holds since $f$ is bounded and $k$ is a bounded kernel),
\begin{align*}
\int_{E_1} f(y)(T_k\mu)(\mathrm dy) &= \int_{E_1} f(y)\left(\int_{E_2}\mu(\mathrm dx)k(\mathrm dy|x)\right)\\ &= \int_{E_2}\mu(\mathrm dx)\left(\int_{E_1} f(y)k(\mathrm dy|x)\right) = \int_{E_2} (S_k f)(x)\mu(\mathrm dx).
\end{align*}
\end{proof}

\begin{lemma}[Adjoint relationship for products of kernels]
\label{lem:adjoint_products}
Let $E_0,E_1,\dots,E_n$ be Polish spaces and for each $i=1,\dots,n$ let
$k_i\in b\mathcal M(E_i|E_{i-1})$ be a bounded signed kernel. Let $T_{k_i}:\mathcal M(E_{i-1})\to \mathcal M(E_i)$ and
$S_{k_i}:B_b(E_i)\to B_b(E_{i-1})$ be the induced bounded linear operators as in Lemma~\ref{lem:adjoint_T_S}. Then for any $\mu\in\mathcal M(E_0)$ and any $f\in B_b(E_n)$,
\begin{equation}
\label{eq:adjoint_products}
\Big\langle T_{k_n}\cdots T_{k_2}T_{k_1}\mu, f\Big\rangle =
\Big\langle \mu, S_{k_1}S_{k_2}\cdots S_{k_{n-1}}S_{k_n}f\Big\rangle.
\end{equation}
In particular, if $E_0=\cdots=E_n=E$ and $k_i=k$ for all $i$, then for every $n\in\mathbb N$,
\begin{equation}
\label{eq:adjoint_powers}
\langle T_k^{\,n}\mu, f\rangle = \langle \mu, S_k^{\,n}f\rangle.
\end{equation}
\end{lemma}
\begin{proof}
We prove \eqref{eq:adjoint_products} by repeated application of
Lemma~\ref{lem:adjoint_T_S}. Set $\nu_0\coloneqq \mu$ and for $i=1,\dots,n$
define $\nu_i\coloneqq T_{k_i}\nu_{i-1}$, so that $\nu_n=T_{k_n}\cdots T_{k_1}\mu$.
Then by Lemma~\ref{lem:adjoint_T_S},
\[ 
\langle \nu_n,f\rangle = \langle T_{k_n}\nu_{n-1},f\rangle= \langle \nu_{n-1},S_{k_n}f\rangle.
\]
Apply Lemma~\ref{lem:adjoint_T_S} again to $\nu_{n-1}=T_{k_{n-1}}\nu_{n-2}$:
\[
\langle \nu_{n-1},S_{k_n}f\rangle = \langle T_{k_{n-1}}\nu_{n-2},S_{k_n}f\rangle = \langle \nu_{n-2},S_{k_{n-1}}S_{k_n}f\rangle.
\]
Continuing iteratively yields
\[
\langle \nu_n,f\rangle = \langle \nu_0,S_{k_1}\cdots S_{k_n}f\rangle =
\langle \mu,S_{k_1}\cdots S_{k_n}f\rangle,
\]
which is \eqref{eq:adjoint_products}. The power identity \eqref{eq:adjoint_powers} is the special case $k_1=\cdots=k_n=k$ on a single space $E$.
\end{proof}

We are ready to introduce MDPs. Let $S$ and $A$ denote Polish spaces, $P\in \mathcal{P}(S|S\times A)$, $r\in B_b(S\times A)$ and $\gamma\in [0,1)$. The five-tuple $\mathcal{M}=(S,A,P,r,\gamma)$ determines an infinite horizon $\gamma$-discounted Markov decision model. We further assume that $A$ is compact, $P(\cdot|\cdot,a)\in \mathcal{P}(S|S)$ is strongly Feller for all $a\in A$, and that $r(s,\cdot):A\rightarrow \mathbb{R}$ is upper semi-continuous for every $s\in S$ so that \citep[Condition 3.3.3]{hernandez2012discrete} holds, and thus measurable selection condition holds. 

Let $((S\times A)^{\mathbb N},\mathcal{F})$ denote the canonical sample space, where $\mathcal{F}$ is the corresponding $\sigma$-algebra. Elements of $(S\times A)^{\mathbb N}$ are of the form $(s_0,a_0,s_1,a_1,\ldots)$ with $s_n\in S$ and $a_n\in A$ denoting the projections and called the state and action variables, at time $n\in \mathbb{N}_0$, respectively. By \citep[Proposition 7.28]{bertsekas2004stochastic}, for an arbitrarily given initial distribution $\rho\in \mathcal{P}(S)$ and randomized stationary policy $\pi \in \mathcal{P}(A|S)$, there exists a unique product probability measure $\mathbb P^{\pi}_{\rho}$ on the canonical sample space with expectation denoted $\mathbb{E}^{\pi}_{\rho}$ such that for every time $n\in \mathbb{N}_0$,  $\mathbb P^{\pi}_{\rho}(s_0\in \mathcal{S})=\rho(\mathcal{S})$, $\mathbb P^{\pi}(a_n\in \mathcal{A}|(s_0,a_0,\ldots,s_n))=\pi(a_n|s_n)$, and 
\[
\mathbb P^{\pi}_{\rho}(s_{n+1}\in \mathcal{S}|(s_0,a_0,\ldots,s_n,a_n))=\mathcal{P}(\mathcal{S}|s_n,a_n)
\]
for all $\mathcal{S}\in \mathcal{B}(S)$ and $\mathcal{A}\in \mathcal{B}(A)$. In particular, $\{s_n\}_{n\in \mathbb{N}_0}$ is a Markov chain with kernel $P_{\pi}\in \mathcal{P}(S|S)$ given by
\begin{equation}
\label{markov chain kernel}
P_{\pi}(\mathrm{d}s'|s)=\int_{A}P(\mathrm{d}s'|s,a')\pi(\mathrm{d}a'|s).
\end{equation}
For given $s\in S$, we denote $\mathbb{E}^{\pi}_{s}=\mathbb{E}^{\pi}_{\delta_s}$, where $\delta_s\in \mathcal{P}(S)$ denotes the Dirac measure at $s\in S$.
For a given policy $\pi\in\mathcal{P}(A|S)$, we also define the state-action value function $Q_{\pi}\in B_b(S\times A)$ by
\begin{equation}
\label{def:Qpi}
Q_{\pi}(s,a)=r(s,a)+\gamma\int_{S}V_{\pi}(s')P(\mathrm{d}s'|s,a),
\end{equation}
and the advantage function $A_{\pi}\in B_b(S\times A)$ by
\begin{equation*}\label{def:Api}
A_{\pi}(s,a)=Q_{\pi}(s,a)-V_{\pi}(s).
\end{equation*}
The occupancy kernel $d^{\pi}\in \mathcal{P}(S|S)$ is defined by
\begin{equation}
\label{occupancy kernel}
d^{\pi}(\mathrm{d}s'|s)=(1-\gamma)\sum_{n=0}^{\infty}\gamma^n P^n_{\pi}(\mathrm{d}s'|s),
\end{equation}
where $P^0_{\pi}(\mathrm{d}s'|s):=\delta_s(\mathrm{d}s')$, $P^n_{\pi}$ is understood as a product of kernels, and convergence is understood in $b\mathcal{M}(S|S)$. For a given initial distribution $\rho\in\mathcal{P}(S)$,  we define
\begin{equation*}
V_{\pi}(\rho)=\int_{S}V_{\pi}(s)\rho(\mathrm{d}s), \; d^{\pi}_{\rho}(\mathrm{d}s)=\int_{S}d^{\pi}(\mathrm{d}s|s')\rho(\mathrm{d}s').
\end{equation*}
For each $(s,a)\in S\times A$, we define the measurable optimal value and state-action value functions by 
\[
V^{\ast}(s)=\sup_{\pi \in \mathcal{P}(A|S)}V_{\pi}(s) \quad \textnormal{and} \quad  Q^*(s,a)=\sup_{\pi \in \mathcal{P}(A|S)}Q_{\pi}(s,a).
\]
By virtue of \citep[Theorem 4.2.3]{hernandez2012discrete}, we have the following dynamic programming principle.
See, also, \citep[Theorems 1 and 2]{haarnoja2017reinforcement}. In order for~\citep[Theorem 4.2.3]{hernandez2012discrete} to apply we need the following technical assumption.
\begin{assumption}
\label{ass for dpp}
\begin{enumerate}
\item The kernel $P \in \mathcal P(S|S\times A)$ is strongly continuous, that is: for every $v\in B_b(S)$ (bounded and measurable) the function  $w(s,a) = \int_S v(s') P(\mathrm{d}s'|s,a)$ is  bounded and measurable as a function from $S\times A$ to $\mathbb R$.  
\item The reward function $r\in B_b(S\times A)$ is upper semi-continuous and sup-compact on $S\times A$ i.e. for any $s\in S$ and any $l \in \mathbb R$ the set $\{a\in A:r(s,a) \geq l\}$ is compact.  
\end{enumerate}
\end{assumption}

\begin{theorem}[Dynamic programming principle]\label{thm:DPP} 
Let Assumption~\ref{ass for dpp} hold. Then the optimal value function $V^{\ast}\in B_b(S)$ is the unique solution of the Bellman equation given by
\begin{equation*}
\begin{aligned}\label{eq:BellmanDPP}
V^{\ast}(s) &= \max_{a \in A}\left[r(s,a)+\gamma\int_{S}V^{\ast}(s')P(\mathrm{d}s'|s,a)\right].
\end{aligned}
\end{equation*}
Moreover, writing $Q^*(s,a)=r(s,a)+\gamma\int_{S}V^*(s')P(\mathrm{d}s'|s,a)$, there exists a measurable function  $f^{\ast}:S\rightarrow A$ called a selector such that $f^{\ast}(s)\in \operatorname{argmax}_{a\in A}Q^*(s,a)$ and the induced policy $\pi^* \in \mathcal{P}(A|S)$ defined by $\pi^*(\mathrm{d}a|s)=\delta_{f^*(s)}(\mathrm{d}a)$ for all $s\in S$ satisfies $V^*=V_{\pi^*}$.
\end{theorem}
\section{Performance bounds and the lifted FR surrogate}
\label{proofs perf diff and tight bound}
As in~\cite{kerimkulov2024fisherraogradientflowentropyregularised} we first recall that a kernel $k\in b\mathcal{M}(S|S)$ induces a linear operator $L_k\in \mathcal{L}(B_b(S))$ such that for all $h\in B_b(S)$, $L_k h(s)=\int_S h(s')k(\mathrm{d}s'|s)$. Consider the kernel $\gamma P_{\pi} \in b\mathcal{M}(S|S)$  defined by $(\gamma P_{\pi})(B)=\gamma\int_B \int_A P(\mathrm{d}s'|s,a) \pi(\mathrm{d}a|s)$ for all $B\in \mathcal{B}(S)$. Then  as $P_\pi\in \mathcal{P}(S|S)$ and $\|  P_{\pi}\|_{b\mathcal{M}(S|S)}=1$,
\[
\|L_{\gamma P_{\pi}}  \|_{\mathcal{L}(B_b(S))} \le \|\gamma P_{\pi}\|_{b\mathcal{M}(S|S)}=\gamma \| P_{\pi}\|_{b\mathcal{M}(S|S)}=\gamma.
\] 
Let $\operatorname{id} $ be the identity operator on $B_b(S)$. As $\|L_{\gamma P_{\pi}}\|_{\mathcal{L}(B_b(S))}\le \gamma <1$, the operator $\operatorname{id}-L_{\gamma P_{\pi}} \in \mathcal{L}(B_b(S))$ is invertible, and the inverse operator is given by the Neumann series $(\operatorname{id}-L_{\gamma P_{\pi}})^{-1}=\sum_{n=0}^\infty L_{\gamma P_{\pi}}^n$. Observe that $L_{\gamma P_{\pi}}^n= L_{\gamma^n P^n_{\pi}}$ for all $n\in \mathbb N_0$, where $P^n_{\pi}$ is the $n$-times  product of the kernel $P_{\pi}$ with $P^0_{\pi}(\mathrm{d}s'|s)\coloneqq \delta_s(\mathrm{d}s')$. Therefore, 
\begin{equation}
\label{eq-occupation-measures-eq}
(\operatorname{id}-L_{\gamma P_{\pi}})^{-1}= L_{(1-\gamma)^{-1}d^{\pi}}\,.    
\end{equation}
Now let $\pi\in P(A|S)$ and $f,g\in B_b(S)$ be such that for all  $s\in S$, 
\begin{equation}
\label{eq fk assumption}
f(s) - \gamma  \int_S \int_A f(s') P(\mathrm{d}s'|s,a)\pi(\mathrm{d}a|s) = g(s).    
\end{equation}
Then from this and~\eqref{eq-occupation-measures-eq} we have $L_{(1-\gamma)^{-1}d^{\pi}} g = (\operatorname{id}-L_{\gamma P_{\pi}})^{-1}g = (\operatorname{id}-L_{\gamma P_{\pi}})^{-1} (\operatorname{id}-L_{\gamma P_{\pi}})f$. 
In other words
\begin{equation}
\label{eq fk conclusion}
f(s) = \frac{1}{1-\gamma}\int_S  g(s') d^{\pi}(\mathrm{d}s'|s) \,\,\text{for all $s\in S$}.    
\end{equation}
\subsection{Performance difference: Lemma~\ref{lemma perf diff}}
We are now ready to prove the performance difference lemma.
\begin{proof}[Proof of Lemma~\ref{lemma perf diff}]
The on-policy Bellman equation is
\begin{equation}
\label{eq on policy bellman}
V_{\pi}(s) = \int_A \bigg( r(s,a) + \gamma \int_S V_{\pi}(s')P(\mathrm{d}s'|s,a)\bigg) \pi(\mathrm{d}a|s)
\end{equation}
which together with~\eqref{def:Qpi} means that $V_{\pi}(s) = \int_A Q_\pi(s,a) \pi(\mathrm{d}a|s)$.
Hence
\begin{equation*}
\begin{split}
&V_{\pi'}(s) - V_{\pi}(s) = \int_A  Q_{\pi'}(s,a) \pi'(\mathrm{d}a|s) - \int_A Q_\pi(s,a) \pi(\mathrm{d}a|s) \\
& = \int_A  Q_\pi(s,a) (\pi'-\pi)(\mathrm{d}a|s) + \int_A  (Q_{\pi'} - Q_\pi)(s,a) \pi'(\mathrm{d}a|s).
\end{split}
\end{equation*}
Hence, using~\eqref{def:Qpi} again, we have  
\begin{equation*}
V_{\pi'}(s) - V_{\pi}(s) - \gamma \int_A \int_S (V_{\pi'} - V_{\pi})(s')P(\mathrm{d}s'|s,a)\pi'(\mathrm{d}a|s) =  \int_A  Q_\pi(s,a) (\pi'-\pi)(\mathrm{d}a|s).
\end{equation*}
Since $s\mapsto V_{\pi}(s)$ does not depend on $a$ we have 
\begin{equation*}
(V_{\pi'} - V_{\pi})(s) - \gamma \int_A \int_S (V_{\pi'} - V_{\pi})(s')P(\mathrm{d}s'|s,a)\pi'(\mathrm{d}a|s) =  \int_A  (Q_\pi(s,a) - V_{\pi}(s)) (\pi'-\pi)(\mathrm{d}a|s)
\end{equation*}
which is of the form~\eqref{eq fk assumption} and so~\eqref{eq fk conclusion} leads to
\begin{equation*}
V_{\pi'}(s) - V_{\pi}(s) = \frac{1}{1-\gamma}\int_S \int_A  (Q_\pi(s',a) - V_{\pi}(s')) (\pi'-\pi)(\mathrm{d}a|s') d^\pi(\mathrm{d}s'|s)\,.
\end{equation*}
Integrating over $\rho$ and noting that $A_\pi = Q_\pi - V_{\pi}$ concludes the proof.
\end{proof}
\subsection{$\operatorname{TV}^2$ surrogate lower bound: Theorem \ref{thm:tv2_mixed}}
Next we prove one of our main results, which is the $\operatorname{TV}^2$ surrogate lower bound on the performance difference.
\begin{proof}[Proof of Theorem \ref{thm:tv2_mixed}]
For simplicity, for any $\pi,\pi'\in\mathcal P(A|S)$, let us denote
\begin{equation*}
    g(s):= \int_A A_\pi(s,a)(\pi'-\pi)(\mathrm da|s).
\end{equation*}
The performance difference lemma (Lemma \ref{lemma perf diff}) gives
\begin{equation*}
    (V_{\pi'}-V_\pi)(\rho)=\frac{1}{1-\gamma}\int_S g(s)d^{\pi'}_\rho(\mathrm ds).
\end{equation*}
Adding and subtracting $d^\pi_\rho$ and lower bounding by subtracting an absolute value yields
\begin{equation}
\label{eq:decomp_start}
(V_{\pi'}-V_\pi)(\rho)
\ge
\frac{1}{1-\gamma}\int_S g(s)\ d^\pi_\rho(\mathrm{d}s)
-\frac{1}{1-\gamma}\left|\int_S g(s)\ (d^{\pi'}_\rho-d^\pi_\rho)(\mathrm{d}s)\right|.
\end{equation}
The goal is to lower bound the second term in \eqref{eq:decomp_start}.

Integrating \eqref{occupancy kernel} with respect to $\rho \in \mathcal{P}(S)$ gives
\begin{equation*}
    d^\pi_\rho(\mathrm{d}s)=(1-\gamma)\sum_{n\ge 0}\gamma^n\rho P_\pi^n(\mathrm{ds}),
\end{equation*}
where we used the shorthand notation $\rho P_\pi^n$ from \eqref{eq:kernel-T} since $P_\pi^n$ is a product of kernels. Hence, using the duality pairing in Lemma \ref{lem:adjoint_T_S},
\begin{equation*}
\int_S g(s)(d^{\pi'}_\rho-d^\pi_\rho)(\mathrm{d}s)=(1-\gamma)\sum_{n\ge 0}\gamma^n\left[\langle\rho P_{\pi'}^n, g\rangle-\langle\rho P_\pi^n, g\rangle\right].
\end{equation*}
Note that the $n=0$ term cancels. Therefore, by the linearity of the duality pairing,
\begin{equation}
\label{eq:time_expand}
\left|\int_S g(s)\ (d^{\pi'}_\rho-d^\pi_\rho)(\mathrm{d}s)\right| \le (1-\gamma)\sum_{n\ge 1}\gamma^n\left|\langle\rho(P_{\pi'}^n-P_\pi^n), g\rangle\right|.
\end{equation}

For $n\ge 1$, it can be showed by induction that
\begin{equation}
\label{eq:telescoping}
P_{\pi'}^n-P_\pi^n=\sum_{k=0}^{n-1} P_\pi^k(P_{\pi'}-P_\pi)P_{\pi'}^{n-1-k}.
\end{equation}
Plugging \eqref{eq:telescoping} into \eqref{eq:time_expand} and using the triangle inequality gives
\begin{align}
\label{eq:double_sum}
\left|\int_S g(s)\ (d^{\pi'}_\rho-d^\pi_\rho)(\mathrm{d}s)\right|&\le
(1-\gamma)\sum_{n\ge 1}\gamma^n\sum_{k=0}^{n-1}
\left|\langle\rho P_\pi^k (P_{\pi'}-P_\pi) P_{\pi'}^{n-1-k}, g\rangle\right|\\
\nonumber &= (1-\gamma)\sum_{n\ge 1}\gamma^n\sum_{k=0}^{n-1}
\left|\langle\rho P_{\pi'}^k, (P_{\pi'}-P_\pi)P_{\pi'}^{n-1-k} g\rangle\right|,
\end{align}
where the equality follows from Lemma \ref{lem:adjoint_products}. For $0\leq k \leq n-1$, let $\mu_k \coloneqq \rho P_\pi^k\in\mathcal P(S)$ and $m\coloneqq n-1-k\ge 0$. First, since $|\int \varphi\ \mathrm{d}(\pi'-\pi)|\le 2\|\varphi\|_{B_b} \mathrm{TV}(\pi',\pi)$,
\begin{equation}
\label{eq:g_bound}
|g(s)|\le 2\|A_\pi\|_{B_b(S \times A)}\,\mathrm{TV}\big(\pi'(\cdot|s),\pi(\cdot|s)\big).
\end{equation}
For each $s\in S$, let $\xi_s\coloneqq (\pi'-\pi)(\cdot|s) \in \mathcal{M}(A)$ and let $|\xi_s|(A)$ denote its total variation measure. Note that $|\xi_s|(A)=2\mathrm{TV}\big(\pi'(\cdot|s),\pi(\cdot|s)\big)$.
Define the probability kernel $\hat\pi(\cdot|s)\in\mathcal P(A)$ by
\begin{equation*}
\hat\pi(\mathrm da|s):=\begin{cases}
\dfrac{|\xi_s|(\mathrm da)}{|\xi_s|(A)}, &\text{if }|\xi_s|(A)>0,\\
p, &\text{if }|\xi_s|(A)=0,
\end{cases}
\end{equation*}
where $p$ is any probability measure in $\mathcal P(A)$. Define the  kernel $\hat P\in\mathcal P(S|S)$ by
\begin{equation*}
    \hat P(\mathrm ds'|s):= \int_A P(\mathrm ds'|s,a)\,\hat\pi(\mathrm da|s).
\end{equation*}
Then, for any $h\in B_b(S)$, using $|\int q\,\mathrm{d}\xi|\le\int |q|\mathrm{d}|\xi|$ and the definition of $\hat\pi$,
\begin{equation}
\begin{aligned}
\label{eq:D_bound}
|(P_{\pi'}-P_\pi)h(s)| &= \left|\int_A \xi_s(\mathrm da)\int_S h(s')P(\mathrm ds'|s,a)\right| \\
&\le \int_A |\xi_s|(\mathrm da)\int_S |h(s')|P(\mathrm ds'|s,a)\\ 
&= 2\mathrm{TV}\big(\pi'(\cdot|s),\pi(\cdot|s)\big)\int_A \hat\pi(\mathrm da|s)\int_S |h(s')|P(\mathrm ds'|s,a) \\
&= 2\mathrm{TV}\big(\pi'(\cdot|s),\pi(\cdot|s)\big)(\hat P |h|)(s),
\end{aligned}
\end{equation}
where $\hat P |h|$ is the shorthand notation from \eqref{eq:kernel-S}.  By positivity of $P_{\pi'}^{m}$ and \eqref{eq:g_bound},
\begin{equation*}
|(P_{\pi'}^{m}g)(s)|\le (P_{\pi'}^{m}|g|)(s)\le 2\|A_\pi\|_{B_b(S \times A)}\,(P_{\pi'}^{m}\Delta)(s),
\end{equation*}
where, for simplicity, we denoted $\Delta(s) :=\mathrm{TV}\big(\pi'(\cdot|s),\pi(\cdot|s)\big) \in [0,1]$. For $m\ge 0$ define
\begin{equation*}
    K_m^{\pi'} \coloneqq \hat PP_{\pi'}^{m}\in\mathcal P(S|S),
\end{equation*}
since it is a composition of probability kernels. Therefore, applying \eqref{eq:D_bound} with $h=P_{\pi'}^{m}g$ gives
\begin{equation*}
\label{eq:DPg_bound}
|(P_{\pi'}-P_\pi)P_{\pi'}^{m}g(s)|\le 4\|A_\pi\|_{B_b(S \times A)}\,\Delta(s)\,(K_m^{\pi'}\Delta)(s),
\end{equation*}
and hence,
\begin{equation}
\label{eq:integrated_prod}
\left|\langle \mu_k, (P_{\pi'}-P_\pi) P_{\pi'}^{m}g\rangle\right| \le
4\|A_\pi\|_{B_b(S \times A)}\int_S \Delta(s)(K_m^{\pi'}\Delta)(s)\mu_k(\mathrm ds).
\end{equation}
Next, we use the following fact. If $\mu\in\mathcal P(S)$, $K\in\mathcal P(S|S)$ and $u \in B_b(S)$ is non-negative, then
\begin{equation}
\label{eq:amgm_lemma}
\int_S u(s)(Ku)(s)\mu(\mathrm{d}s) \le \frac12\int_S u^2(s)\mu(\mathrm{d}s)+\frac12\int u^2(s) (\mu K)(\mathrm{d}s),
\end{equation}
where $Ku \in B_b(S)$ is non-negative. Indeed, we have the pointwise inequality $u(s)(Ku)(s)\le \frac12(u^2(s)+(Ku)^2(s))$. Jensen's inequality gives $(Ku)^2(s)\le (Ku^2)(s)$ and integrating yields
\begin{equation*}
    \int_S (Ku)^2(s)\mu(\mathrm{d}s)\le \int_S (Ku^2)(s)\mu(\mathrm{d}s)=\int_S u^2(s)(\mu K)(\mathrm{d}s),    
\end{equation*}
where the equality follows from Lemma \ref{lem:adjoint_T_S}. Applying \eqref{eq:amgm_lemma} in \eqref{eq:integrated_prod} with $u=\Delta$, $K = K_m^{\pi'}$, $\mu=\mu_k$ gives
\begin{equation}
\label{eq:mu_k_bound}
\left|\langle \mu_k, (P_{\pi'}-P_\pi) P_{\pi'}^{m}g\rangle\right|
\le 2\|A_\pi\|_{B_b(S \times A)}\left[
\int_S \Delta^2(s) \mu_k(\mathrm ds) + \int_S \Delta^2(s) (\mu_k K_m)(\mathrm ds) \right].
\end{equation}
Define
\begin{equation*}
    \mathcal S := (1-\gamma)\sum_{n\ge 1}\gamma^n\sum_{k=0}^{n-1}\left|\langle\mu_k, (P_{\pi'}-P_\pi) P_{\pi'}^{n-1-k}g\rangle\right|.
\end{equation*}
Plugging \eqref{eq:mu_k_bound} into \eqref{eq:double_sum} gives
\begin{equation*}
    \mathcal S \le 2\|A_\pi\|_{B_b(S \times A)}(\mathcal S_1+\mathcal S_2),
\end{equation*}
where
\begin{equation*}
\mathcal S_1:= (1-\gamma)\sum_{n\ge 1}\gamma^n\sum_{k=0}^{n-1}\int_S \Delta^2(s)\mu_k(\mathrm{d}s),
\
\mathcal S_2:= (1-\gamma)\sum_{n\ge 1}\gamma^n\sum_{k=0}^{n-1}\int \Delta^2(s)(\mu_k K_{n-1-k})(\mathrm{d}s).
\end{equation*}
For $\mathcal{S}_1$, recall $\mu_k=\rho P_\pi^k$ and so
\begin{equation*}
    \mathcal S_1 = (1-\gamma)\sum_{n\ge 1}\gamma^n\sum_{k=0}^{n-1}\int_S \Delta^2(s)\mu_k(\mathrm ds) = (1-\gamma)\sum_{n\ge 1}\gamma^n\sum_{k=0}^{n-1}\int_S \Delta^2(s)(\rho P_\pi^k)(\mathrm ds).
\end{equation*}
We first exchange the order of summation. For a fixed $k\ge 0$, the index $n$ ranges over $n\ge k+1$, hence
\begin{align*}
\mathcal S_1
&=(1-\gamma)\sum_{k=0}^{\infty}\sum_{n=k+1}^{\infty}\gamma^n\int_S \Delta^2(s)(\rho P_\pi^k)(\mathrm ds) =(1-\gamma)\sum_{k=0}^{\infty}\left(\sum_{n=k+1}^{\infty}\gamma^n\right)\int_S \Delta^2(s)(\rho P_\pi^k)(\mathrm ds).
\end{align*}
The inner sum is the geometric series
\begin{equation*}
    \sum_{n=k+1}^{\infty}\gamma^n = \gamma^{n+1}\sum_{j=0}^{\infty}\gamma^j = \frac{\gamma^{k+1}}{1-\gamma}.
\end{equation*}
Substituting this into the preceding equation yields
\begin{align*}
\mathcal S_1 =(1-\gamma)\sum_{k=0}^{\infty}\frac{\gamma^{k+1}}{1-\gamma}\int_S \Delta^2(s)(\rho P_\pi^k)(\mathrm ds) &=\sum_{k=0}^{\infty}\gamma^{k+1}\int_S \Delta^2(s)(\rho P_\pi^k)(\mathrm ds)\\ &=\gamma\sum_{k=0}^{\infty}\gamma^{k}\int_S \Delta^2(s)(\rho P_\pi^k)(\mathrm ds).
\end{align*}
Now using the definition of the occupancy measure, i.e., $d^\pi_\rho = (1-\gamma)\sum_{k=0}^{\infty}\gamma^k\rho P_\pi^k$, we get
\begin{equation}
\label{eq:S1_detailed}
\mathcal S_1 =\gamma\int_S \Delta^2(s)\left(\sum_{k=0}^{\infty}\gamma^k\,\rho P_\pi^k\right)(\mathrm ds)
=\frac{\gamma}{1-\gamma}\int_S \Delta^2(s)d^\pi_\rho(\mathrm ds),
\end{equation}
For $\mathcal{S}_2$, define the mixed occupancy measure $\tilde d_{\rho}^{\pi,\pi'}\in\mathcal P(S)$ by
\begin{equation}
\label{eq:mixocc}
\tilde d_{\rho}^{\pi,\pi'} := (1-\gamma)^2\sum_{k=0}^{\infty}\sum_{m=0}^{\infty}\gamma^{k+m}\;
\rho P_\pi^{k}K_m^{\pi'}.
\end{equation}
Since $(1-\gamma)^2\sum_{k,m\ge 0}\gamma^{k+m}=1$, it follows that $\tilde d_{\rho}^{\pi,\pi'}$ is a probability measure.
Recall
\begin{align*}
    \mathcal S_2 &= (1-\gamma)\sum_{n\ge 1}\gamma^n\sum_{k=0}^{n-1}\int_S \Delta^2(s)(\mu_k K_{n-1-k})(\mathrm ds)\\
&= (1-\gamma)\sum_{n\ge 1}\gamma^n\sum_{k=0}^{n-1}\int_S \Delta^2(s)(\rho P_\pi^k K_{n-1-k})(\mathrm ds).
\end{align*}
With the change of variables $m = n-1-k$, we get $n=k+1+m$. For each fixed $n$, the allowed $k \in \{0,1,...,n-1\}$ uniquely determines $m = n-1-k \geq 0$, and conversely any $k,m \geq 0$ uniquely give $n = m+k+1 \geq 1$. Thus, the map
\begin{equation*}
    (n,k) \mapsto (k,m=n-1-k)
\end{equation*}
is a bijection between $\{(n,k):n\ge 1,\ 0\le k\le n-1\}$ and $\{(k,m):k\ge 0, m\ge 0\}$. Hence,
\begin{align*}
\mathcal S_2 &=(1-\gamma)\sum_{k=0}^{\infty}\sum_{m=0}^{\infty}\gamma^{k+1+m}
\int_S \Delta^2(s)(\rho P_\pi^k K_m^{\pi'})(\mathrm ds)\\ &=\gamma(1-\gamma)\sum_{k=0}^{\infty}\sum_{m=0}^{\infty}\gamma^{k+m}
\int_S \Delta^2(s)(\rho P_\pi^k K_m^{\pi'})(\mathrm ds).
\end{align*}
Substituting the definition of the mixed occupancy measure $\tilde d_{\rho}^{\pi,\pi'}$ in \eqref{eq:mixocc} into the previous equation gives
\begin{equation}
\label{eq:S2_detailed}
\mathcal S_2 =\gamma(1-\gamma)\int_S \Delta^2(s)\left(\sum_{k,m\ge 0}\gamma^{k+m}\rho P_\pi^k K_m^{\pi'}\right)(\mathrm ds) =\frac{\gamma}{1-\gamma}\int_S \Delta^2(s)\tilde d_{\rho}^{\pi,\pi'}(\mathrm ds).
\end{equation}
Combining \eqref{eq:S1_detailed} and \eqref{eq:S2_detailed} gives
\begin{equation}
\label{eq:master_bound}
\mathcal S \le \frac{2\gamma\|A_\pi\|_{B_b(S \times A)}}{1-\gamma}
\left[\int_S \mathrm{TV}^2\big(\pi'(\cdot|s),\pi(\cdot|s)\big)d^\pi_\rho(\mathrm ds)
+ \int_S \mathrm{TV}^2\big(\pi'(\cdot|s),\pi(\cdot|s)\big)\tilde d_{\rho}^{\pi,\pi'}(\mathrm ds)
\right].
\end{equation}
Since $\left|\int_s g(s)(d^{\pi'}_\rho-d^\pi_\rho)(\mathrm{d}s)\right|\le \mathcal S$, plugging \eqref{eq:master_bound} into \eqref{eq:decomp_start} and using $\|A_\pi\|_{B_b(S \times A)}\le \frac{2\|r\|_{B_b(S \times A)}}{1-\gamma}$ implies \eqref{eq:tv2_mixed_reward}.
\end{proof}
\subsection{Lifted $\operatorname{FR}^2$ surrogate lower bound: Corollary \ref{thm lower bound perf diff}}
We finish this section by showing that under concentrability of the auxiliary occupancy measure, we get the surrogate lower bound with respect to the lifted FR$^2$ distance. Then we present standard sufficient conditions under which concentrability holds.
\begin{proof}[Proof of Corollary \ref{thm lower bound perf diff}]
Under the concentrability condition $C_\rho(\pi, \pi')<+\infty$, for any non negative $u \in B_b(S)$,
\begin{equation*}
\int_S u(s)\tilde d_{\rho}^{\pi,\pi'}(\mathrm ds)
\le \left\|\frac{\mathrm d\tilde d_{\rho}^{\pi,\pi'}}{\mathrm d d^\pi_\rho}\right\|_{L^\infty(S,\rho)}
\int_S u(s)d^\pi_\rho(\mathrm ds) \le
C_\rho\int_S u(s)d^\pi_\rho(\mathrm{d}s).
\end{equation*}
Apply this with $u(s)=\mathrm{TV}^2\big(\pi'(\cdot|s),\pi(\cdot|s)\big)$ in Theorem \ref{thm:tv2_mixed} to obtain the conclusion.
\end{proof}

\paragraph{Sufficient conditions for $C_\rho(\pi, \pi')<+\infty$} If $S$ is finite and $\rho_{\min}:=\min_{s\in S}\rho(s)>0$, then, for every policy $\pi$,
\begin{equation*}
    d_\rho^\pi(s)\geq (1-\gamma)\rho(s) \geq(1-\gamma)\rho_{\min}.
\end{equation*}
Since $\widetilde d_{\rho}^{\pi,\pi'}(s)\leq1$, it follows that, uniformly over all policy pairs,
\begin{equation*}
    C_\rho(\pi,\pi') \leq \frac{1}{(1-\gamma)\rho_{\min}}.
\end{equation*}
A continuous state sufficient condition can be stated similarly. If there exist $\nu\in\mathcal P(S)$ and $c,M>0$ such that, as measure inequalities, $\rho \geq c\nu$, $P(\cdot|s,a)\leq M\nu$, for every $s,a$, then
\begin{equation*}
    \widetilde d_{\rho}^{\pi,\pi'}\leq M\nu, \quad d_\rho^\pi\geq(1-\gamma)c\nu,
\end{equation*}
and hence
\begin{equation*}
    C_\rho(\pi,\pi') \leq \frac{M}{(1-\gamma)c}
\end{equation*}
for all $\pi,\pi'$. Consequently, in both cases, the assumption is satisfied in particular for every pair $(\pi^n,\pi^{n+1})$ produced by \eqref{eq:pointwise_min}.

\section{Direct policy updates: well-posedness, improvement, and convergence}
\label{app:C}
\subsection{Well-posedness and the projection representation}
Before proving Theorem \ref{thm policy improvement}, we first show that the update rule in \eqref{eq:pointwise_min} is well-defined.
\begin{lemma}[Well-posedness of the update \eqref{eq:pointwise_min}]
\label{lemma:well-posedness-FR-PPO}
Let $\pi^0 \in \mathcal{P}_\lambda(A| S)$, $\tau > 0$, and $\lambda \in \mathcal{P}(A)$. Then the sequence of iterates $(\pi^n)_{n \in \mathbb{N}}$ generated by the update \eqref{eq:pointwise_min} is well-defined and remains in $\mathcal{P}_\lambda(A|S)$ for all $n$. Moreover
\begin{equation*}
\frac{\mathrm d \pi^{n+1}}{\mathrm d \lambda}(\cdot|s) = \argmin_{\phi \in C}\left\| \phi - \left(\frac{\mathrm d \pi^n }{\mathrm d\lambda}(\cdot|s)  + \frac{\tau}{1-\gamma} A_{\pi^n}(s,\cdot)\right)\right\|_{L^2_\lambda(A)},    
\end{equation*}
where $C=\left\{\phi\in L^2_\lambda(A) \vert \phi\geq 0\ \lambda\text{-a.e. on } A \text{ and } \int_A \phi(a)\lambda(\mathrm{d}a)=1\right\}$.
\end{lemma}
\begin{proof}
The maximum in \eqref{eq:pointwise_min} can be achieved by the following pointwise optimization in $s \in S$:
\begin{equation}
\label{eq:scheme_pointwise}
\pi^{n+1}(\cdot|s)= \argmax_{m \in \mathfrak C_\lambda}\left[\int_A A_{\pi^n}(s,a)(m(\mathrm{d}a)-\pi^n(\mathrm{d}a|s)) - \frac{1-\gamma}{8\tau}\operatorname{FR^2}(m^2,\pi^n(\cdot|s)^2)\right].
\end{equation}
If $\pi^0 \in \mathcal{P}_\lambda(A|S)$, then $\pi^0(\cdot|s) \in \mathfrak C_\lambda$, and it suffices to show that $\pi^1(\cdot|s) \in \mathfrak C_\lambda$. The first-order condition (see e.g., \citep[Section 5.1.1]{Bonnans2000PerturbationAO}) shows that for a.e. $a \in A$,
\begin{equation*} 
\label{eq:first-order_L2}
\left\langle A_{\pi^0}(s,\cdot) - \frac{1-\gamma}{\tau}
\left(\frac{\mathrm d \pi^1}{\mathrm d\lambda }(\cdot|s)  -\frac{\mathrm d \pi^0 }{\mathrm d\lambda}(\cdot|s)  \right),
\phi-  \frac{\mathrm d \pi^1}{\mathrm d\lambda}(\cdot|s) 
\right\rangle_{L^2_\lambda(A)}\le 0,  
\quad \forall \phi\in C, 
\end{equation*}
where $\left\langle \cdot, \cdot \right\rangle_{L^2_\lambda(A)} $ is the inner product on $L^2_\lambda(A)$ and we note that $C$ is nonempty, closed and convex.
Define the projection map $\operatorname{Proj}_{C}: L^2_\lambda(A) \to C$ such that 
$\operatorname{Proj}_{C}(\varphi)=\argmin_{\phi\in C}\|\phi-\varphi\|_{L^2_\lambda(A)}$ for all $\varphi \in L^2_\lambda(A)$, 
which satisfies 
\begin{equation*}
\left\langle \operatorname{Proj}_{C}(\varphi)-\varphi, \phi-  \operatorname{Proj}_{C}(\varphi)\right\rangle_{L^2_\lambda(A)}\ge 0,  \quad \forall \phi\in C. 
\end{equation*}
Then 
\begin{equation*}
\frac{\mathrm d \pi^1}{\mathrm d \lambda}(\cdot|s) = \operatorname{Proj}_{C}\left( \frac{\mathrm d \pi^0 }{\mathrm d\lambda}(\cdot|s)  + \frac{\tau}{1-\gamma} A_{\pi^0}(s,\cdot)  \right).
\end{equation*}
Note $\|\operatorname{Proj}_{C}(\varphi_1)-\operatorname{Proj}_{C}(\varphi_2)\|_{L^2_\lambda(A)}
\le \|\varphi_1-\varphi_2\|_{L^2_\lambda(A)}$ for all $\varphi_1,\varphi_2\in L^2_\lambda(A)$ (see e.g., \citep[Theorem 4.3-1]{ciarlet2013linear}). Moreover, since $\frac{\mathrm d \pi^0}{\mathrm d \lambda}(\cdot|s) = \operatorname{Proj}_{C}\left(\frac{\mathrm d \pi^0}{\mathrm d \lambda}(\cdot|s)\right)$, for a.e. $a \in A$,
\begin{align*}
\left\|\frac{\mathrm d \pi^1}{\mathrm d \lambda}(\cdot|s) \right\|_{L^2_\lambda(A)}
& \le  
\left \| \frac{\mathrm d \pi^0 }{\mathrm d\lambda }(\cdot|s)\right\|_{L^2_\lambda(A)} + \frac{\tau}{1-\gamma}
\left \|A_{\pi^0}(s, \cdot)\right\|_{L^2_\lambda(A)}\\
&\leq \left \| \frac{\mathrm d \pi^0 }{\mathrm d\lambda }(\cdot|s)\right\|_{L^2_\lambda(A)} + \frac{2 \tau \|r\|_{B_b(S\times A)}}{(1-\gamma)^2},
\end{align*}
so $\pi^1(\cdot|s) \in \mathfrak C_\lambda$, and inductively, $\left(\pi^n(\cdot|s)\right)_{n \in \mathbb N} \subset \mathfrak C_\lambda$.
\end{proof}
\begin{remark}[Connection to project policy gradient for tabular MDPs]
\label{connection project policy gradient}
In finite action spaces with uniform reference measure, the direct FR-PPO update in Lemma \ref{lemma:well-posedness-FR-PPO} is equivalent to $L^2$- projection of $\pi^n + \tau A_{\pi^n}$ onto the simplex, statewise.
Indeed, update $\pi^{n+1}$ is obtained by first computing $\pi^n +\tau A_{\pi^n}$ and then performing an $L^2$-projection onto the convex subspace such that $\pi^{n+1}(\cdot|s) \geq 0$ and $\sum_a \pi^{n+1}(a|s)=1$ for all $s$. The projection can be formulated as a quadratic program with convex constraints:
\begin{align*}
    \min_{m \in \mathbb R^{|A|}} \left\|\pi^n(\cdot|s) +\frac{\tau}{1-\gamma} A_{\pi^n}(s,\cdot)-m\right\|_2^2 \
    \text{ s.t. } m\geq 0, \ \sum_{a} m(a) = 1
\end{align*}
There is an exact $O(|A|\log |A|)$ algorithm to compute this. This involves sorting the elements, finding a specific threshold, and then uniformly shifting and clipping any negative values at zero. 
\end{remark}
\subsection{Policy improvement: Theorem \ref{thm policy improvement}}
Now we prove that the policy update rule \eqref{eq:pointwise_min} guarantees improvement of the value function.
\begin{proof}[Proof of Theorem \ref{thm policy improvement}]
Since $\pi^{n+1}$ is given by~\eqref{eq:scheme_pointwise} we have for any $\pi\in \mathcal P_\lambda(A|S)$ and any $s\in S$ that 
\begin{equation*}
\begin{split}
& \int_A A_{\pi^n}(s,a)(\pi^{n+1}-\pi^n)(\mathrm{d}a|s) - \frac{1-\gamma}{8\tau} \operatorname{FR}^2(\pi^{n+1}(\cdot|s)^2,\pi^n(\cdot|s)^2) \\ 
&\geq \int_A A_{\pi^n}(s,a)(\pi-\pi^n)(\mathrm{d}a|s) - \frac{1-\gamma}{8\tau} \operatorname{FR}^2(\pi(\cdot|s)^2,\pi^n(\cdot|s)^2)\,.
\end{split}
\end{equation*}
In particular with $\pi=\pi^n$ we get for all $s\in S$ that
\begin{equation*}
\int_A A_{\pi^n}(s,a)(\pi^{n+1}-\pi^n)(\mathrm{d}a|s) - \frac{1-\gamma}{8\tau} \operatorname{FR}^2(\pi^{n+1}(\cdot|s)^2,\pi^n(\cdot|s)^2) \geq 0\,.
\end{equation*}
From this, the fact that $\operatorname{FR}^2(\pi^{n+1}(\cdot|s)^2,\pi^n(\cdot|s)^2)\geq 0$ and Lemma~\ref{lemma perf diff} we get
\begin{equation*}
\begin{split}
&V_{n+1}(\rho) - V_{n}(\rho)  = \frac{1}{1-\gamma} \int_S \int_A A_{\pi^n}(s,a)(\pi^{n+1}-\pi^n)(\mathrm{d}a|s)d_{\rho}^{\pi^{n+1}}(\mathrm{d}s)\\
& \geq \frac{1}{1-\gamma} \int_S \bigg( \int_A A_{\pi^n}(s,a)(\pi^{n+1}-\pi^n)(\mathrm{d}a|s) - \frac{1-\gamma}{8\tau}\operatorname{FR}^2(\pi^{n+1}(\cdot|s)^2,\pi^n(\cdot|s)^2)\bigg)d_{\rho}^{\pi^{n+1}}(\mathrm{d}s)\\
&\geq 0.
\end{split}
\end{equation*}
Thus $V_{n+1}(\rho) \geq V_{n}(\rho)$.
\end{proof}

\subsection{Pointwise estimate}
Before proving the sublinear and linear convergence of \eqref{eq:pointwise_min}, we need to prove a state-pointwise estimate for the value function between two consecutive iterates. First, we observe from \cite[Proposition 2.20]{kerimkulov2024mirrordescentstochasticcontrol} that
\begin{equation*}
\frac{1}{8}\operatorname{FR}^2(\pi'(\cdot|s)^2, \pi(\cdot|s)^2) = D_h(\pi'(\cdot|s)\|\pi(\cdot|s)),
\end{equation*}
for $h(\pi) = \frac{1}{2}\chi^2(\pi|\lambda)$. Therefore, the lower bound in Corollary \ref{cor lower bound FR^2} becomes
\begin{align*}
(V_{\pi'} - V_{\pi})(\rho) &\geq \frac{1}{1-\gamma} \int_S \int_A \frac{\mathrm{d}\pi'}{\mathrm{d}\pi}(a|s)A_{\pi}(s,a)\pi(\mathrm{d}a|s)d_{\rho}^{\pi}(\mathrm{d}s)\\&- \frac{4\gamma\|r\|_{B_b(S \times A)}}{(1-\gamma)^3}\int_S D_h(\pi'(\cdot|s)\| \pi(\cdot|s))d_{\rho}^{\pi}(\mathrm{d}s),
\end{align*}
which shows that $V_{\pi}$ is smooth relative to the Bregman divergence generated by $\chi^2(\pi|\lambda)$.
\begin{lemma}[Pointwise estimate]
\label{lemma:pointwise_estimate}
Let $V_n := V_{\pi^n}$ for $n\in \mathbb N$, $\tau > 0$ and $\pi^n \in \mathcal P_\lambda(A|S)$ given by~\eqref{eq:pointwise_min}.
Then for any $s\in S$ we have 
\begin{equation*}
(V_{n+1} - V_n)(s) \geq \int_A A_{\pi^n}(s,a)(\pi^{n+1}-\pi^n)(\mathrm{d}a|s) - \frac{1}{\tau} D_h(\pi^{n+1}\|\pi^n)(s).
\end{equation*}	
\end{lemma}
\begin{proof}
Fix $n\in \mathbb N$. Observe that the pointwise optimization in $s \in S$ in update \eqref{eq:pointwise_min} can be equivalently written as
\begin{equation}
\label{eq pointwise}
\pi^{n+1}(\cdot|s) =\argmax_{m \in \mathfrak C_\lambda}\left[\int_A A_{\pi^n}(s,a)(m(\mathrm{d}a)-\pi^n(\mathrm{d}a|s)) - \frac{1}{\tau}D_h(m\|\pi^n(\cdot|s))\right].
\end{equation}
Using Lemma~\ref{lemma perf diff} (performance difference) we get for all $s\in S$ that 
\begin{equation}
\label{eq:sharper_proof_1}
\begin{split}
(V_{n+1} - V_n)(s) & = \frac{1}{1-\gamma} \int_S \int_A A_{\pi^n}(s,a)(\pi^{n+1}-\pi^n)(\mathrm{d}a|s) \, d^{\pi^{n+1}}_s (\mathrm{d}s')\\
& \geq \frac{1}{1-\gamma} \int_S F(s) \, d^{\pi^{n+1}}_s (\mathrm{d}s'),    
\end{split}
\end{equation}
where 
\begin{equation}
\label{F(s) non-negative}
F(s) := \int_A A_{\pi^n}(s,a)(\pi^{n+1}-\pi^n)(\mathrm{d}a|s) - \frac{1}{\tau} D_h(\pi^{n+1}\|\pi^n)(s) \geq 0,
\end{equation}
with the inequality coming from \eqref{eq pointwise}. From~\eqref{occupancy kernel} and~\eqref{eq:sharper_proof_1} we have  for all $s\in S$ that
\begin{equation*}
\!\!(V_{n+1} - V_n)(s) = \int_S F(s') P^0_{\pi^{n+1}}(\mathrm{d}s'|s) + \sum_{k=1}^\infty \int_S \gamma^k F(s') P^k_{\pi^{n+1}}(\mathrm{d}s'|s)	 
\geq \int_S F(s') \delta_s(\mathrm{d}s').
\end{equation*} 
This concludes the proof.
\end{proof}
\subsection{Sublinear and linear convergence: Theorem~\ref{thm:convergence_fr_steps}}
Now we are ready to prove another main result, which is the sublinear convergence with a constant stepsize, and linear convergence with geometrically increasing stepsizes, of the direct policy update \eqref{eq:pointwise_min}.
\begin{proof}[Proof of Theorem \ref{thm:convergence_fr_steps}]
Let $s\in S$, $\pi^n \in \mathcal P_\lambda(A|S)$ and $G:\mathcal P_\lambda(A) \to \mathbb R$ be given by 
\begin{equation*}
G(m) = \tau \int_A A_{\pi^n}(s,a)(m(\mathrm{d}a)-\pi^n(\mathrm{d}a|s)).
\end{equation*}
Note that $G$ is affine, hence both convex and concave, and in particular the hypothesis of \ref{lem three point} is satisfied. By Lemma \ref{lem three point} and update \eqref{eq pointwise} we have, for all $\pi \in \mathcal{P}_\lambda(A|S)$, $s\in S$ and $n\in \mathbb N$ that 
\begin{equation*}
\begin{split}
& \tau\int_A A_{\pi^n}(s,a)(\pi-\pi^n)(\mathrm{d}a|s) -  D_h(\pi\|\pi^n)(s)\\
& \leq \tau\int_A A_{\pi^n}(s,a)(\pi^{n+1}-\pi^n)(\mathrm{d}a|s) - D_h(\pi\|\pi^{n+1})(s) - D_h(\pi^{n+1}\|\pi^n)(s)\,.	
\end{split}
\end{equation*} 
Re-arranging this leads to 
\begin{equation}
\label{eq:proof_of_convergence_after_3_point}
\begin{split}
& D_h(\pi\|\pi^{n+1})(s) - D_h(\pi\|\pi^n)(s)\\
& \leq -\tau \int_A A_{\pi^n}(s,a)(\pi-\pi^n)(\mathrm{d}a|s) + \tau \int_A A_{\pi^n}(s,a)(\pi^{n+1}-\pi^n)(\mathrm{d}a|s)  - D_h(\pi^{n+1}\|\pi^n)(s).
\end{split}
\end{equation} 
From Lemma~\ref{lemma:pointwise_estimate} we thus have, for all $s\in S$, that 
\begin{equation*}
\label{eq:proof_of_convergence_after_3_mdp_step}
\begin{split}
\tau \int_A A_{\pi^n}(s,a)(\pi^{n+1} -\pi^n)(\mathrm{d}a|s) - D_h(\pi^{n+1}\|\pi^n)(s) \leq \tau (V_{n+1} - V_n)(s).
\end{split}
\end{equation*} 
Taking the sum from $n=0$ to $N-1$ we see (spotting the telescoping sums) that for all $s\in S$,
\begin{equation}
\label{eq:convergence_proof_1}
\begin{split}
& D_h(\pi\|\pi^{N})(s) - D_h(\pi\|\pi^0)(s)
\leq -\sum_{n=0}^{N-1}\tau \int_A A_{\pi^n}(s,a)(\pi-\pi^n)(\mathrm{d}a|s) + \tau (V_N - V^0)(s).
\end{split}
\end{equation} 
Notice that $(V_n-V^0)(s) \leq (V^\ast-V^0)(s)$ for all $N\in \mathbb N$. Let 
\begin{equation*}
y^n := \int_S D_h(\pi\|\pi^n)(s) d^{\pi}_\rho(\mathrm{d}s)\,\,\,\text{and}\,\,\, \alpha := \int_S (V^\ast-V^0)(s)	d^{\pi}_\rho(\mathrm{d}s)
\end{equation*}
so that, after integrating~\eqref{eq:convergence_proof_1} over $d^{\pi}_\rho$ we have
\begin{equation}
\label{eq:convergence_proof_2}
y^N - y^0
\leq -\tau\sum_{n=0}^{N-1}\int_S \int_A A_{\pi^n}(s,a)(\pi-\pi^n)(\mathrm{d}a|s)d^{\pi}_\rho(\mathrm{d}s) + \tau\alpha\,.
\end{equation} 
Using the performance difference lemma (Lemma~\ref{lemma perf diff}), we get
\begin{equation*}
\label{eq:convergence_proof_3}
y^N - y^0
\leq -\sum_{n=0}^{N-1}\tau(1-\gamma)(V_{\pi} - V_{\pi^n})(\rho)  + \tau\alpha\,.
\end{equation*} 
Since $V_{\pi^n}(\rho) \leq V_{\pi^N}(\rho)$ for all $n = 0,1,\ldots,N$ we get that  
\begin{equation*}
y^N - y^0
\leq -N\tau(1-\gamma)(V_{\pi} - V_{\pi^N})(\rho) + \tau\alpha\,.	
\end{equation*}
Moreover, since $y^N \geq 0$, we get
\begin{equation*}
N\tau(1-\gamma)(V_{\pi}-V_{\pi^N})(\rho) 
\leq  \tau\alpha + y^0	
\end{equation*}
and so
\begin{equation*}
(V_{\pi} - V_{\pi^N})(\rho) \leq (1-\gamma)^{-1}\left(\alpha + \frac{1}{\tau} y^0\right)N^{-1}\,.
\end{equation*}
This completes the proof of the first estimate. 

For the second estimate, define $\xi := (1-\gamma)^{-1}\big\|\frac{\mathrm d d^{\pi^*}_\rho}{\mathrm d \rho}\big\|_{B_b(S)}$. Note that $\frac{\mathrm d d^{\pi}_\rho}{\mathrm d \rho}(s) \geq 1-\gamma$ for all $\pi$ and $s \in S$. Hence $\frac{\mathrm d d^{\pi^*}_\rho}{\mathrm d \rho}(s) =  \frac{\mathrm d d^{\pi^*}_\rho}{\mathrm d d^{\pi}_\rho}(s) \frac{\mathrm d d^{\pi}_\rho}{\mathrm d \rho}(s) \geq (1-\gamma)\frac{\mathrm d d^{\pi^*}_\rho}{\mathrm d d^{\pi}_\rho}(s)$ for any $\pi$ and $s \in S$. We thus have $\xi_n:=\|\frac{\mathrm d d^{\pi^*}_\rho}{\mathrm d d^{\pi^n}_\rho}\|_{B_b(S)} \leq \xi$ for all $n \in \mathbb N$.

Observe that
\begin{align*}
    &\int_S \left(\int_A A_{\pi^n}(s,a)(\pi^{n+1} -\pi^n)(\mathrm{d}a|s) - \frac{1}{\tau_n} D_h(\pi^{n+1}\|\pi^n)(s)\right) d^{\pi^*}_\rho(\mathrm d s)\\
    &= \int_S \left(\int_A A_{\pi^n}(s,a)(\pi^{n+1} -\pi^n)(\mathrm{d}a|s) - \frac{1}{\tau_n} D_h(\pi^{n+1}\|\pi^n)(s)\right) \frac{\mathrm d d^{\pi^*}_\rho}{\mathrm d d^{\pi^{n+1}}_\rho}(s)d^{\pi^{n+1}}_\rho(\mathrm{d}s)\\
    &\leq \Bigg\|\frac{\mathrm d d^{\pi^*}_\rho}{\mathrm d d^{\pi^{n+1}}_\rho}\Bigg\|_{B_b(S)} \int_S \left(\int_A A_{\pi^n}(s,a)(\pi^{n+1} -\pi^n)(\mathrm{d}a|s) - \frac{1}{\tau_n} D_h(\pi^{n+1}\|\pi^n)(s)\right) d^{\pi^{n+1}}_\rho(\mathrm{d}s)\\
    &= \xi_{n+1} \int_S \left(\int_A A_{\pi^n}(s,a)(\pi^{n+1} -\pi^n)(\mathrm{d}a|s) - \frac{1}{\tau_n} D_h(\pi^{n+1}\|\pi^n)(s)\right) d^{\pi^{n+1}}_\rho(\mathrm{d}s)\\
    &\leq \xi \int_S \int_A A_{\pi^n}(s,a)(\pi^{n+1} -\pi^n)(\mathrm{d}a|s) d^{\pi^{n+1}}_\rho(\mathrm{d}s)\\
    &= \xi (1-\gamma)(V_{n+1}-V_n)(\rho),
\end{align*}
where the first inequality follows because the integrand in $s$ is non-negative cf. \eqref{F(s) non-negative}, the second inequality follows from $\xi_{n+1} \leq \xi$ and $\tau_n^{-1}D_h(\pi^{n+1}\|\pi^n)(s) \geq 0$, and the last equality is due to performance difference (Lemma \ref{lemma perf diff}).

Using this, putting $\pi=\pi^*$ in \eqref{eq:proof_of_convergence_after_3_point} and integrating it with respect to $d^{\pi^*}_\rho$ gives
\begin{equation}
\begin{aligned}
\label{telescope adaptive step}
    &\frac{1}{\tau_n}\int_S D_h(\pi^*\|\pi^{n+1})(s) d^{\pi^*}_\rho(\mathrm{d}s) - \frac{1}{\tau_n}\int_S D_h(\pi^*\|\pi^n)(s) d^{\pi^*}_\rho(\mathrm{d}s)\\
 &\leq -\int_S\int_A A_{\pi^n}(s,a)(\pi^*-\pi^n)(\mathrm{d}a|s)d^{\pi^*}_\rho(\mathrm{d}s)\\ 
 &+ \int_S\left(\int_A A_{\pi^n}(s,a)(\pi^{n+1}-\pi^n)(\mathrm{d}a|s)  - \frac{1}{\tau_n}D_h(\pi^{n+1}\|\pi^n)(s)\right)d^{\pi^*}_\rho(\mathrm{d}s)\\
 &\leq -\int_S\int_A A_{\pi^n}(s,a)(\pi^*-\pi^n)(\mathrm{d}a|s)d^{\pi^*}_\rho(\mathrm{d}s) + \xi (1-\gamma)(V_{n+1}-V_n)(\rho)\\
 &=-(1-\gamma)(V^*-V_n)(\rho) + \xi (1-\gamma)(V_{n+1}-V_n)(\rho),
 \end{aligned}
\end{equation}
where the last equality follows from one more application of  performance difference (Lemma \ref{lemma perf diff}).

With the notation of $y^n$ from earlier and setting $\delta_n:= V^*-V_n \geq 0$ for all $n$, rearranging the estimate above gives
\begin{equation*}
    \delta_n + \xi(\delta_{n+1}-\delta_n) \leq \frac{y^n}{\tau_n(1-\gamma)} - \frac{y^{n+1}}{\tau_n(1-\gamma)},
\end{equation*}
and equivalently
\begin{equation*}
    \delta_{n+1} + \frac{y^{n+1}}{\tau_n(1-\gamma)\xi} \leq \left(1-\frac{1}{\xi}\right)\left(\delta_n + \frac{y^n}{\tau_n(1-\gamma)(\xi-1)}\right).
\end{equation*}
We want $\tau_{n+1}(\xi - 1) \geq \tau_n \xi$. By definition of the occupancy measure, its density satisfies $\left\|
\frac{\mathrm d d_\rho^{\pi^*}}{\mathrm d \rho} \right\|_{B_b(S)} \geq 1$.
Therefore, $\xi = \frac{1}{1-\gamma}\left\|\frac{\mathrm d d_\rho^{\pi^*}}{\mathrm d \rho}\right\|_{B_b(S)}\geq\frac{1}{1-\gamma}$.
Since the map \(x\mapsto \frac{x}{x-1}\) is decreasing on \((1,\infty)\), we have $\frac{\xi}{\xi-1} \leq \frac{\frac{1}{1-\gamma}}{\frac{1}{1-\gamma}-1} =
\frac{1}{\gamma}$. Hence due to the geometric step size relation $\tau_{n+1} \geq \frac{1}{\gamma}\tau_n$ we get $\tau_{n+1}(\xi - 1) \geq \tau_n \xi$ as required.

Therefore, we get
\begin{equation*}
    \delta_{n+1} + \frac{y^{n+1}}{\tau_{n+1}(1-\gamma)(\xi-1)}\leq \delta_{n+1} + \frac{y^{n+1}}{\tau_n(1-\gamma)\xi} \leq \left(1-\frac{1}{\xi}\right)\left(\delta_n + \frac{y^n}{\tau_n(1-\gamma)(\xi-1)}\right).
\end{equation*}
We iterate this recursive inequality and obtain
\begin{equation*}
    \delta_n + \frac{y^n}{\tau_n(1-\gamma)(\xi-1)}\leq \left(1-\frac{1}{\xi}\right)^n\left(\delta_0 + \frac{y^0}{\tau_0(1-\gamma)(\xi-1)}\right).
\end{equation*}
Since $\xi = \frac{1}{1-\gamma}\left\|\frac{\mathrm d d_\rho^{\pi^*}}{\mathrm d \rho}\right\|_{B_b(S)}\geq\frac{1}{1-\gamma}$ implies $(1-\gamma)(\xi-1) \geq \gamma$, we arrive at
\begin{equation*}
    \delta_n \leq \delta_n + \frac{y^n}{\tau_n(1-\gamma)(\xi-1)}\leq \left(1-\frac{1}{\xi}\right)^n\left(\delta_0 + \frac{y^0}{\tau_0(1-\gamma)(\xi-1)}\right)
    \leq \left(1-\frac{1}{\xi}\right)^n\left(\delta_0 + \frac{y^0}{\tau_0\gamma}\right).
\end{equation*}
\end{proof}
\section{Parametrized updates: variational characterizations and convergence}
\label{app:D}
\paragraph{Examples of policies in $\Pi$.} Before giving the proofs for the parametrized FR--PPO, we recall two examples of policies belonging to the class $\Pi$ introduced in Section \ref{sec:npg}.

\begin{example}[Tabular policies]
For finite state and action spaces $S,A$, let $\mathbb H = \mathbb R^{|S|\times|A|}$ and $\lambda$ be the uniform measure over $A$. We consider one-hot feature encodings $\phi_{(s',a')}(s,a) = \mathbbm{1}{\{(s',a') = (s,a)\}}$. In this setting, the policy space $\Pi$ corresponds exactly to the product of probability simplices $\Delta(A)^{|S|}$. A policy $\pi_\theta$ is defined via direct parametrization $\pi_\theta(a|s) \coloneqq \theta_{s,a}$, subject to the constraints $\theta_{s,a} \geq 0$ and $\sum_{a \in A}\theta_{s,a} = 1$, for all $a \in A$ and $s \in S$.
\end{example}

\begin{example}[Gaussian mixture policies]
For continuous state and action spaces $S,A$, let $\mathbb H = \mathbb R^d$ and $\lambda \in \mathcal{P}(A)$. We define the features using normalized Radial Basis Functions (RBFs) (e.g., \citep{lagoudakis}). Each feature $\phi_i$ is a Gaussian kernel centered at fixed points $\{(c_i^s,c_i^a)\}_{i=1}^d$ with variance $\sigma^2$, that is $\phi_i(s,a) \propto \exp{\left(\frac{|s-c_i^s|^2 + |a-c_i^a|^2}{2\sigma^2}\right)}$ normalized such that $\int_A \phi_i(s,a) \lambda(\mathrm{d}a) = 1$. The resulting policy takes the form of a Gaussian Mixture Model (GMM) $\pi_\theta(a|s) \coloneqq \sum_{i=1}^d \theta_i \phi_i(s,a)$ where the parameters satisfy $\theta_i \geq 0$ and $\sum_{i=1}^d \theta_i = 1$.
\end{example}

\subsection{Feature induced geometry and the parameter-to-policy correspondence}
The parametrized updates rely on the following key auxiliary lemma, which shows that the policy class $\Pi$ and the parameter space $\Theta$ are isometric under the policy mapping $\Theta \ni \theta \mapsto \pi_\theta \in \Pi$.
\begin{lemma}[The policy $\pi_\theta \in \Pi$ is an isometry]
\label{lemma:policy-isometry}
The map $\Theta \ni \hat{\theta} \mapsto \pi_{\hat\theta} \in \Pi$ is an isometry, i.e., $\pi_{\hat\theta}$ satisfies
\begin{equation*}
\left\|\frac{\mathrm{d}\pi_{\hat\theta}}{\mathrm{d}\lambda}\right\|_{{L^2_\lambda(A)}\times d_\rho^{\pi_\theta}} = \left\|\hat\theta\right\|_{G(\theta)}.
\end{equation*}
\end{lemma}
\begin{proof}
By the definition of the semi-norm $\|\cdot\|_{G(\theta)}$, we have
\begin{align*}
\left\|\hat\theta\right\|_{G(\theta)}^2 &= \left\langle \hat\theta, G(\theta)\hat\theta\right\rangle_{\mathbb H}= \left\langle \hat\theta, \left(\int_S\int_A \phi(s,a) \otimes \phi^*(s,a)\lambda(\mathrm{d}a)d_\rho^{\pi_\theta}(\mathrm{d}s)\right) \hat\theta\right\rangle_{\mathbb H}\\ 
&= \int_S\int_A \langle \hat\theta, \left(\phi(s,a) \otimes \phi^*(s,a)\right)\hat\theta\rangle_{\mathbb H}\lambda(\mathrm{d}a)d_\rho^{\pi_\theta}(\mathrm{d}s)\\ 
&= \int_S\int_A \langle \hat\theta, \phi(s,a)\langle \hat\theta, \phi(s,a)\rangle_{\mathbb{H}}\rangle_{\mathbb H}\lambda(\mathrm{d}a)d_\rho^{\pi_\theta}(\mathrm{d}s)\\ 
&= \int_S\int_A \left|\langle \hat\theta, \phi(s,a)\rangle_{\mathbb H}\right|^2\lambda(\mathrm{d}a)d_\rho^{\pi_\theta}(\mathrm{d}s)\\ 
&= \int_S\int_A \left|\frac{\mathrm{d}\pi_{\hat\theta}}{\mathrm{d}\lambda}(a|s)\right|^2\lambda(\mathrm{d}a)d_\rho^{\pi_\theta}(\mathrm{d}s) = \left\|\frac{\mathrm{d}\pi_{\hat\theta}}{\mathrm{d}\lambda}\right\|^2_{{L^2_\lambda(A)}\times d_\rho^{\pi_\theta}}.
\end{align*}
\end{proof}

\subsection{Variational characterizations of Proj-NPG and AProj-NPG: Propositions \ref{prop:variational-characterization-Proj--NPG} and \ref{prop:variational-characterization-AProj--NPG}}
Now we prove that the parametrized updates Proj-NPG and AProj-NPG have variational characterizations.
\begin{proof}[Proof of Proposition \ref{prop:variational-characterization-Proj--NPG}]
Consider the quadratic error $L^{\pi_{\theta}}: \mathbb H \rightarrow \mathbb R$ defined by
\begin{equation*}
L^{\pi_\theta}(w) = \frac{1}{2}\int_S \int_A \left|A_{\pi_\theta}(s,a) - \langle w,  \phi(s,a)\rangle_{\mathbb H}\right|^2\lambda(\mathrm{d}a)d_{\rho}^{\pi_{\theta}}(\mathrm{d}s).
\end{equation*}
The first-order optimality condition $\nabla_w L^{\pi_\theta}(w) = 0$ implies
\begin{equation*}
\int_S \int_A A_{\pi_\theta}(s,a) \phi(s,a)\lambda(\mathrm{d}a)d_{\rho}^{\pi_{\theta}}(\mathrm{d}s) = \int_S \int_A \langle w,  \phi(s,a)\rangle_{\mathbb H} \phi(s,a)\lambda(\mathrm{d}a)d_{\rho}^{\pi_{\theta}}(\mathrm{d}s).
\end{equation*}
By the policy gradient theorem (Theorem \ref{thm:policy-grad-theorem}) and the chain rule, we have
\begin{equation*}
\begin{aligned}
\nabla_{\theta} V_{\pi_{\theta}}(\rho) &= \frac{1}{1-\gamma}\int_{S}\int_{A} A_{\pi_{\theta}}(s,a)\nabla_\theta\pi_{\theta}(\mathrm{d}a|s)d_{\rho}^{\pi_{\theta}}(\mathrm{d}s)\\ 
&=\frac{1}{1-\gamma}\int_{S}\int_{A} A_{\pi_{\theta}}(s,a)\phi(s,a)\lambda(\mathrm{d}a)d_{\rho}^{\pi_{\theta}}(\mathrm{d}s).
\end{aligned}
\end{equation*}
Using the definition of the tensor product action, $\left(\phi(s,a) \otimes \phi^*(s,a)\right)w = \phi(s,a)\langle w, \phi(s,a)\rangle_{\mathbb{H}} = \langle w, \phi(s,a)\rangle_{\mathbb{H}}\phi(s,a)$, we conclude that the minimizer $w^*(\theta)$ satisfies
\begin{equation*}
G(\theta)w^*(\theta) = (1-\gamma)\nabla_\theta V_{\pi_{\theta}}(\rho).
\end{equation*}
The Moore--Penrose pseudo-inverse $G^\dagger$ then gives the minimal-norm solution
\begin{equation*}
w^*(\theta) = (1-\gamma)G(\theta)^\dagger\nabla_\theta V_{\pi_{\theta}}(\rho).
\end{equation*}
Setting $\tau = \eta(1-\gamma)^{-1}$ transforms \eqref{eq:proj-NPG-parameters} into 
\begin{equation*}
\theta^{n+1} = \Gamma_{\Theta}^{G(\theta^n)}\left(\theta^n + \tau w^*(\theta^n)\right),\ \theta^0 \in \Theta.
\end{equation*}
Define $\Tilde{\theta}^n = \theta^n + \tau w^*(\theta^n)$. By the definition of the projection map $\Gamma_{\Theta}^{G(\theta)}$,
\begin{equation*}
\|\theta^{n+1}-\Tilde{\theta}^n\|_{G(\theta^n)} = \min_{\bar{\theta} \in \Theta}\|\bar{\theta} - \Tilde{\theta}^n\|_{G(\theta^n)}.
\end{equation*}
For policies, the projection map $\Gamma_\Pi:\mathcal{P}_{\lambda}(A|S) \to \Pi$ satisfies 
\begin{equation*}
\Gamma_\Pi\left(\frac{\mathrm{d}\pi_{\Tilde{\theta}^n}}{\mathrm{d}\lambda}\right) = \argmin_{f \in \Pi}\left\|f-\frac{\mathrm{d}\pi_{\Tilde{\theta}^n}}{\mathrm{d}\lambda}\right\|_{{L_\lambda^2 (A)} \otimes d_\rho^{\pi_{\theta^n}}}.
\end{equation*}
Recalling~\eqref{def:Pi_class}, which is the definition of $\Pi$, there exists $\hat\theta \in \Theta$ such that
\begin{align*}
\min_{f \in \Pi}\left\|f-\frac{\mathrm{d}\pi_{\Tilde{\theta}^n}}{\mathrm{d}\lambda}\right\|_{{L_\lambda^2 (A)} \times d_\rho^{\pi_{\theta^n}}} &= \left\|\frac{\mathrm{d}\pi_{\hat{\theta}}}{\mathrm{d}\lambda}-\frac{\mathrm{d}\pi_{\Tilde{\theta}^n}}{\mathrm{d}\lambda}\right\|_{{L_\lambda^2 (A)} \times d_\rho^{\pi_{\theta^n}}} = \left\|\frac{\mathrm{d}\pi_{\hat{\theta}-\Tilde{\theta}^n}}{\mathrm{d}\lambda}\right\|_{{L_\lambda^2 (A)} \times d_\rho^{\pi_{\theta^n}}}\\
&=\left\|\hat{\theta}-\Tilde{\theta}^n\right\|_{G(\theta^n)} \geq \min_{\theta \in \Theta}\|\theta - \Tilde{\theta}^n\|_{G(\theta^n)}\\&= \|\theta^{n+1}-\Tilde{\theta}^n\|_{G(\theta^n)}= \left\|\frac{\mathrm{d}\pi_{\theta^{n+1}-\Tilde{\theta}^n}}{\mathrm{d}\lambda}\right\|_{{L_\lambda^2 (A)} \times d_\rho^{\pi_{\theta^n}}}\\ &= \left\|\frac{\mathrm{d}\pi_{\theta^{n+1}}}{\mathrm{d}\lambda}-\frac{\mathrm{d}\pi_{\tilde \theta^n}}{\mathrm{d}\lambda}\right\|_{{L_\lambda^2 (A)} \times d_\rho^{\pi_{\theta^n}}} \geq \min_{f \in \Pi}\left\|f-\frac{\mathrm{d}\pi_{\tilde \theta^n}}{\mathrm{d}\lambda}\right\|_{{L_\lambda^2 (A)} \times d_\rho^{\pi_{\theta^n}}}.
\end{align*}
Thus,
\begin{equation*}
\left\|\frac{\mathrm{d}\pi_{\theta^{n+1}}}{\mathrm{d}\lambda}-\frac{\mathrm{d}\pi_{\Tilde{\theta}^n}}{\mathrm{d}\lambda}\right\|_{{L_\lambda^2 (A)} \times d_\rho^{\pi_{\theta^n}}} = \min_{f \in \Pi}\left\|f-\frac{\mathrm{d}\pi_{\Tilde{\theta}^n}}{\mathrm{d}\lambda}\right\|_{{L_\lambda^2 (A)} \times d_\rho^{\pi_{\theta^n}}},
\end{equation*}
and so
\begin{equation*}
\frac{\mathrm{d}\pi_{\theta^{n+1}}}{\mathrm{d}\lambda} = \Gamma_\Pi\left(\frac{\mathrm{d}\pi_{\Tilde{\theta}^n}}{\mathrm{d}\lambda}\right) = \Gamma_{\Pi}\left(\frac{\mathrm{d}\pi_{\theta^n}}{\mathrm{d}\lambda} + \tau  \langle w^*(\theta^n),\phi\rangle_{\mathbb H}\right).
\end{equation*}
If the approximation of $A_{\pi_{\theta^n}}$ by the features $\phi$ is exact (i.e., $L^{\pi_{\theta^n}}(w^*(\theta^n)) = 0$), then $A_{\pi_{\theta^n}}(s,a) = \langle w^*(\theta^n),\phi(s,a)\rangle_{\mathbb H}$, and therefore
\begin{equation*}
\frac{\mathrm{d}\pi_{\theta^{n+1}}}{\mathrm{d}\lambda} =\Gamma_{\Pi}\left(\frac{\mathrm{d}\pi_{\theta^n}}{\mathrm{d}\lambda} + \tau A_{\pi_{\theta^n}}\right).
\end{equation*}
From~\eqref{def:Pi_class} the definition of the policy class we know that it is equivalently
\begin{align*}
\Pi = \Big\{&\pi_\theta \in \mathcal{P}_\lambda(A|S)|\pi_\theta(\mathrm{d}a|s) = \langle\theta, \phi(s,a)\rangle_{\mathbb{H}}\lambda(\mathrm{d}a)\,\,\text{for some}\,\,\theta \in \Theta\\ 
&\text{ such that } \langle \theta,\phi(s,a) \rangle_{\mathbb H} \geq 0 \ \lambda\text{-a.e.},\int_A \langle \theta,\phi(s,a) \rangle_{\mathbb H} \lambda(\mathrm{d}a) = 1, \text{ for all } s\in S \Big\}.
\end{align*}
Finally, by Lemma \ref{lemma:well-posedness-FR-PPO}, 
\begin{equation*}
\pi_{\theta^{n+1}}= \argmax_{\pi \in \Pi}\left[\left\langle \frac{\mathrm{d}\pi}{\mathrm{d}\pi_{\theta^{n}}}A_{\pi_{\theta^n}}, \pi_{\theta^n} \right\rangle_{d^{\pi_{\theta^n}}_{\rho}} - \frac{1}{8\tau}\int_S \operatorname{FR^2}(\pi(\cdot|s)^2,\pi_{\theta^n}(\cdot|s)^2)d_\rho^{\pi_{\theta^n}}(\mathrm{d}s)\right].
\end{equation*}
\end{proof}

\begin{proof}[Proof of Proposition \ref{prop:variational-characterization-AProj--NPG}]
    We observe that Proposition \ref{prop:variational-characterization-AProj--NPG} follows directly from the proof of Proposition \ref{prop:variational-characterization-Proj--NPG} after substituting $A_{\pi_{\theta^n}}$ with $\left\langle \hat{w}(\theta^n), \phi \right\rangle_{\mathbb H}$.
\end{proof}

\subsection{Sublinear and linear convergence of Proj-NPG: Theorem \ref{thm:sublinear_convergence_proj_NPG}}
Here we prove another main result of the paper, which established sublinear convergence with a constant stepsize, and linear convergence with geometrically increasing stepsizes, of the parametrized update Proj-NPG.

\begin{proof}[Proof of Theorem \ref{thm:sublinear_convergence_proj_NPG}]
    The proof of Theorem \ref{thm:sublinear_convergence_proj_NPG} follows directly from the argument used in Theorem \ref{thm:convergence_fr_steps}, applied to the variational characterization established in Proposition \ref{prop:variational-characterization-Proj--NPG}.
\end{proof} 

\subsection{Approximation error terms and convergence of AProj-NPG: Theorem \ref{thm:sublinear_convergence_approx_proj_NPG}}
We finish this section by proving that sublinear convergence with a constant stepsize, and linear convergence with geometrically increasing stepsizes, also holds for AProj-NPG. The proof differs slightly from that of Theorem~\ref{thm:sublinear_convergence_proj_NPG} to account for the approximation error term.

\begin{proof}[Proof of Theorem \ref{thm:sublinear_convergence_approx_proj_NPG}]
Recall that $\pi \in \mathcal P_\lambda(A|S)$ is a fixed policy.
We follow the argument used in Theorem \ref{thm:convergence_fr_steps}. The maximum in \eqref{eq:approx-max_parametrized} can be achieved by the following pointwise optimization in $s \in S$:
\begin{equation}
\label{eq:approx-pointwise_max_parametrized}
\pi_{\theta^{n+1}}(\cdot|s)= \argmax_{m \in \mathfrak C_\lambda}\left[\int_A \left\langle \hat{w}(\theta^n), \phi(s,a) \right\rangle_{\mathbb H}(m(\mathrm{d}a)-\pi_{\theta^n}(\mathrm{d}a|s)) - \frac{1}{8\tau}\operatorname{FR^2}(m^2,\pi_{\theta^n}(\cdot|s)^2)\right].
\end{equation}
Since
\begin{equation*}
\frac{1}{8}\operatorname{FR}^2(\pi'(\cdot|s)^2, \pi(\cdot|s)^2) = D_h(\pi'(\cdot|s)\|\pi(\cdot|s)),
\end{equation*}
for $h = \frac{1}{2}\chi^2(\cdot|\lambda)$, we can apply Lemma \ref{lem three point} to the update in \eqref{eq:approx-pointwise_max_parametrized} with
\begin{equation*}
G(m) = \tau \int_A \left\langle \hat{w}(\theta^n), \phi(s,a) \right\rangle_{\mathbb H}(m(\mathrm{d}a)-\pi_{\theta^n}(\mathrm{d}a|s)).
\end{equation*}
Proceeding as in the proof of Theorem \ref{thm:convergence_fr_steps}, but with $\left\langle \hat{w}(\theta^n), \phi \right\rangle_{\mathbb H}$ in place of $A_{\pi_{\theta^n}}$ in \eqref{eq:proof_of_convergence_after_3_point}, we obtain 
\begin{equation*}
\begin{split}
& D_h(\pi_\theta\|\pi_{\theta^{n+1}})(s) - D_h(\pi_\theta\|\pi_{\theta^n})(s)\\
& \leq -\tau \int_A \left\langle \hat{w}(\theta^n), \phi(s,a) \right\rangle_{\mathbb H}(\pi_\theta-\pi_{\theta^n})(\mathrm{d}a|s)\\ 
&+ \tau \int_A \left\langle \hat{w}(\theta^n), \phi(s,a) \right\rangle_{\mathbb H}(\pi_{\theta^{n+1}}-\pi_{\theta^n})(\mathrm{d}a|s)  - D_h(\pi_{\theta^{n+1}}\|\pi_{\theta^n})(s).
\end{split}
\end{equation*}
After adding and subtracting $\int_A A_{\pi_{\theta^n}}(\pi_\theta-\pi_{\theta^n})(\mathrm{d}a|s)$ and $\int_A A_{\pi_{\theta^n}}(\pi_{\theta^{n+1}}-\pi_{\theta^n})(\mathrm{d}a|s)$ we get
\begin{align*}
    &D_h(\pi_\theta\|\pi_{\theta^{n+1}})(s) - D_h(\pi_\theta\|\pi_{\theta^n})(s)\\
& \leq -\tau \int_A \left(\left\langle \hat{w}(\theta^n), \phi(s,a) \right\rangle_{\mathbb H} - A_{\pi_{\theta^n}}\right)(\pi_\theta-\pi_{\theta^n})(\mathrm{d}a|s)\\ 
&+ \tau \int_A \left(\left\langle \hat{w}(\theta^n), \phi(s,a) \right\rangle_{\mathbb H}-A_{\pi_{\theta^n}}\right)(\pi_{\theta^{n+1}}-\pi_{\theta^n})(\mathrm{d}a|s) - D_h(\pi_{\theta^{n+1}}\|\pi_{\theta^n})(s)\\
&- \tau\int_A A_{\pi_{\theta^n}}(\pi_\theta-\pi_{\theta^n})(\mathrm{d}a|s) + \tau\int_A A_{\pi_{\theta^n}}(\pi_{\theta^{n+1}}-\pi_{\theta^n})(\mathrm{d}a|s)\\
&=\tau \int_A \left(A_{\pi_{\theta^n}}-\left\langle \hat{w}(\theta^n), \phi(s,a) \right\rangle_{\mathbb H}\right)(\pi_\theta-\pi_{\theta^{n+1}})(\mathrm{d}a|s)\\ 
&- \tau\int_A A_{\pi_{\theta^n}}(\pi_\theta-\pi_{\theta^n})(\mathrm{d}a|s) + \tau\int_A A_{\pi_{\theta^n}}(\pi_{\theta^{n+1}}-\pi_{\theta^n})(\mathrm{d}a|s) - D_h(\pi_{\theta^{n+1}}\|\pi_{\theta^n})(s).
\end{align*}
Define
\begin{equation*}
\mathcal{E}_{n+1} \coloneqq \int_S \int_A \left(A_{\pi_{\theta^n}}(s,a)-\left\langle \hat{w}(\theta^n), \phi(s,a) \right\rangle_{\mathbb H}\right)(\pi-\pi_{\theta^{n+1}})(\mathrm{d}a|s)d^{\pi}_\rho(\mathrm{d}s).
\end{equation*}
After integrating over $d^{\pi}_\rho$ we get
\begin{align*}
    &\int_S D_h(\pi\|\pi_{\theta^{n+1}})(s) d^{\pi}_\rho(\mathrm{d}s) - \int_S D_h(\pi\|\pi_{\theta^n})(s) d^{\pi}_\rho(\mathrm{d}s) \leq \tau \mathcal{E}_{n+1}\\ 
    &- \tau\int_S\int_A A_{\pi_{\theta^n}}(\pi_\theta-\pi_{\theta^n})(\mathrm{d}a|s)d^{\pi}_\rho(\mathrm{d}s)\\ 
    &+ \tau\int_S\int_A A_{\pi_{\theta^n}}(\pi_{\theta^{n+1}}-\pi_{\theta^n})(\mathrm{d}a|s)d^{\pi}_\rho(\mathrm{d}s) - \int_S D_h(\pi_{\theta^{n+1}}\|\pi_{\theta^n})(s)d^{\pi}_\rho(\mathrm{d}s)
\end{align*}
The last two terms are treated exactly as in the proof of Theorem \ref{thm:convergence_fr_steps}. We thus have
\begin{align*}
    \tau\int_S\int_A A_{\pi_{\theta^n}}(\pi_{\theta^{n+1}}-\pi_{\theta^n})(\mathrm{d}a|s)d^{\pi}_\rho(\mathrm{d}s) - \int_S D_h(\pi_{\theta^{n+1}}\|\pi_{\theta^n})(s)d^{\pi}_\rho(\mathrm{d}s) \leq \tau (V_{n+1} - V_n)(\rho).
\end{align*}
Taking the sum from $n=0$ to $n=N-1$ gives
\begin{align*}
\int_S D_h(\pi\|\pi_{\theta^N})(s) d^{\pi}_\rho(\mathrm{d}s) &- \int_S D_h(\pi\|\pi_{\theta^0})(s) d^{\pi}_\rho(\mathrm{d}s)\\
&\leq \tau\sum_{n=0}^{N-1}\mathcal{E}_{n+1} -\sum_{n=0}^{N-1}\tau(1-\gamma)(V_{\pi} - V_n)(\rho) + \tau\int_S (V^\ast-V^0)(s)d^{\pi}_\rho(\mathrm{d}s),
\end{align*}
where the last equality uses the performance difference lemma (Lemma \ref{lemma perf diff}). Since $\int_S D_h(\pi\|\pi_{\theta^N})(s) d^{\pi}_\rho(\mathrm{d}s) \geq 0$, it follows that
\begin{align*}
& \min_{n<N} (V_{\pi} - V_n)(\rho)\\ 
&\leq \frac{1}{\tau(1-\gamma)N}\left(\int_S D_h(\pi\|\pi_{\theta^0})(s) d^{\pi}_\rho(\mathrm{d}s)+\tau\int_S (V^\ast-V^0)(s)d^{\pi}_\rho(\mathrm{d}s)\right)\\ 
&+ \frac{1}{(1-\gamma)N}\sum_{n=0}^{N-1}\mathcal{E}_{n+1}\\
&=\frac{1}{8\tau(1-\gamma)N}\left(\int_S \operatorname{FR}^2(\pi(\cdot|s)^2, \pi_{\theta^0}(\cdot|s)^2) d^{\pi}_\rho(\mathrm{d}s)+8\tau\int_S (V^\ast-V^0)(s)d^{\pi}_\rho(\mathrm{d}s)\right)\\ 
&+ \frac{1}{(1-\gamma)N}\sum_{n=0}^{N-1}\mathcal{E}_{n+1}.
\end{align*}
This completes the proof of the first estimate. 

For the second estimate, we follow the proof of Theorem \ref{thm:convergence_fr_steps} identically. Hence we obtain precisely \eqref{telescope adaptive step} to which the approximation error term $\ \mathcal{E}_{n+1}$ is added. We thus have
\begin{equation*}
\begin{aligned}
    &\frac{1}{\tau_n}\int_S D_h(\pi^*\|\pi^{n+1})(s) d^{\pi^*}_\rho(\mathrm{d}s) - \frac{1}{\tau_n}\int_S D_h(\pi^*\|\pi^n)(s) d^{\pi^*}_\rho(\mathrm{d}s)\\
 &\leq -(1-\gamma)(V^*-V_n)(\rho) + \xi (1-\gamma)(V_{n+1}-V_n)(\rho)+ \mathcal{E}_{n+1}.
 \end{aligned}
\end{equation*}
With the same notation of $y^n$ and $\delta_n$ from earlier, rearranging the estimate above gives
\begin{equation*}
    \delta_n + \xi(\delta_{n+1}-\delta_n) \leq \frac{y^n}{\tau_n(1-\gamma)} - \frac{y^{n+1}}{\tau_n(1-\gamma)} + \mathcal{E}_{n+1},
\end{equation*}
and equivalently
\begin{equation*}
    \delta_{n+1} + \frac{y^{n+1}}{\tau_n(1-\gamma)\xi} \leq \left(1-\frac{1}{\xi}\right)\left(\delta_n + \frac{y^n}{\tau_n(1-\gamma)(\xi-1)}\right) + \mathcal{E}_{n+1}.
\end{equation*}
Using the earlier relation $\tau_{n+1}(\xi - 1) \geq \tau_n \xi$ gives
\begin{equation*}
    \delta_{n+1} + \frac{y^{n+1}}{\tau_{n+1}(1-\gamma)(\xi-1)}\leq \left(1-\frac{1}{\xi}\right)\left(\delta_n + \frac{y^n}{\tau_n(1-\gamma)(\xi-1)}\right) + \mathcal{E}_{n+1}.
\end{equation*}
Iterating this recursive inequality gives
\[
\delta_n \leq \left(1-\frac1\xi\right)^n\left(
\delta_0+ \frac{y^0}{\tau_0(1-\gamma)(\xi-1)}\right)+\sum_{k=1}^n\left(1-\frac1\xi\right)^{n-k}\mathcal E_k.
\]
Since $\xi = \frac{1}{1-\gamma}\left\|\frac{\mathrm d d_\rho^{\pi^*}}{\mathrm d \rho}\right\|_{B_b(S)}\geq\frac{1}{1-\gamma}$ implies $(1-\gamma)(\xi-1) \geq \gamma$, we arrive at
\begin{equation*}
    \delta_n \leq \left(1-\frac{1}{\xi}\right)^n\left(\delta_0 + \frac{y^0}{\tau_0\gamma}\right)+\sum_{k=1}^n\left(1-\frac1\xi\right)^{n-k}\mathcal E_k.
\end{equation*}
\end{proof}

\section{Continuous time limit and convergence of the lifted FR flow}
\label{sec:continuous_time_comparison}

We study the continuous time limit and long-time behavior of the limiting flow of the lifted\footnote{Recall that lifted FR$^2$ distance is, up to constant scaling factor, equal to the squared $L^2$ distance between densities.} FR--PPO scheme \eqref{eq:pointwise_min}. We first prove convergence of the discrete updates as $\tau\to0$ and then analyze the limiting flow as $t\to\infty$.

\subsection{Tangent and normal cones: definition and characterization}
Recall from Lemma \ref{lemma:well-posedness-FR-PPO} that 
\begin{equation*}
C=\Big\{v\in L^2_{\lambda}(A):v\geq0\ \lambda\text{-a.e.},\
\int_Av(a)\lambda(\mathrm da)=1\Big\}.
\end{equation*}

For $v\in C$, define the tangent and normal cones
\begin{align}
\label{eq:tgt_cone}
T_C(v)&:=\overline{\{\alpha(w-v):\alpha\geq0,\ w\in C\}}^{L^2_\lambda(A)}, \\
\label{eq:normal_cone}
N_C(v)&:=\{\eta\in L^2_\lambda(A):\langle\eta,w-v\rangle_{L^2_\lambda(A)}\leq0\text{ for every }w\in C\}.
\end{align}
For the tangent cone $T_C(v)$, every vector $\alpha(w-v)$, with $\alpha\geq0$ and $w\in C$, points from $v$ toward another admissible density. Indeed, $v+t\alpha(w-v)=(1-t\alpha)v+t\alpha w\in C$, when $0\leq t\alpha\leq1$. Taking the $L^2$-closure includes limits of these feasible directions. Thus, a tangent direction need not itself give a feasible straight-line step, but it can be approximated by directions that do. The normal cone $N_C(v)$ describes outward pointing directions, which make a nonpositive inner product with every feasible direction.
This is illustrated in Figure~\ref{fig:tangent_normal_cones}.

\begin{figure}
\centering
\begin{tikzpicture}[scale=0.9, >=stealth, every node/.style={font=\small}]

\coordinate (v) at (0,0);
\coordinate (a) at (6.3,1.26);             
\coordinate (b) at (1.4,5.6);              
\coordinate (w) at ($(a)!0.5!(b)$);        
\coordinate (x) at ($(v)!0.333!(a)!0.5!(b)$);
\def\nang{48.47}

\fill[gray!15] (v) -- (a) -- (b) -- cycle;
\draw[thick]   (v) -- (a) -- (b) -- cycle;
\node at (1.3,2.4) {$C$};

\fill[blue!40, opacity=0.35] (v) -- ++(11.31:2.6)
  arc[start angle=11.31, end angle=75.96, radius=2.6] -- cycle;
\draw[blue!70!black, ->] (v) -- ++(11.31:2.6);
\draw[blue!70!black, ->] (v) -- ++(75.96:2.6);
\node[blue!70!black] at ($(v)+(50:2.1)$) {$v+T_C(v)$};

\fill[red!40, opacity=0.35] (v) -- ++(165.96:1.8)
  arc[start angle=165.96, end angle=281.31, radius=1.8] -- cycle;
\draw[red!70!black, ->] (v) -- ++(165.96:1.8);
\draw[red!70!black, ->] (v) -- ++(281.31:1.8);
\node[red!70!black] at ($(v)+(223:2.4)$) {$v+N_C(v)$};

\draw[thin] ($(v)+(11.31:0.25)$) -- ($(v)+(11.31:0.25)+(281.31:0.25)$)
  -- ($(v)+(281.31:0.25)$);
\draw[thin] ($(v)+(75.96:0.25)$) -- ($(v)+(75.96:0.25)+(165.96:0.25)$)
  -- ($(v)+(165.96:0.25)$);

\coordinate (g) at ($(v)+(-30:1.8)$);
\coordinate (p) at ($(v)+(11.31:1.352)$);
\coordinate (q) at ($(v)+(g)-(p)$);

\draw[->, thick] (v) -- (g) node[below right] {$g$};
\draw[->, thick, green!50!black] (v) -- (p);
\draw[->, thick, red!70!black]   (v) -- (q);

\node[font=\scriptsize, green!50!black]
  at ($(p)+(0.0,0.35)$) {$\operatorname{Proj}_{T_C(v)}g$};
\node[font=\scriptsize, red!70!black, anchor=east]
  at ($0.5*(q)+(-0.1,-0.1)$) {$\operatorname{Proj}_{N_C(v)}g$};

\draw[dashed] (p) -- (g);
\draw[dashed] (q) -- (g);

\pgfmathsetmacro{\tA}{\nang+90}
\pgfmathsetmacro{\tB}{\nang-90}
\pgfmathsetmacro{\iN}{\nang+180}

\fill[blue!40, opacity=0.35]
  ($(w)+(\tB:1.8)$) -- ($(w)+(\tA:1.8)$)
  -- ($(w)+(\tA:1.8)+(\iN:1.3)$) -- ($(w)+(\tB:1.8)+(\iN:1.3)$) -- cycle;
\draw[blue!70!black, dashed] ($(w)+(\tB:1.8)$) -- ($(w)+(\tA:1.8)$);
\node[blue!70!black, anchor=west] at ($(w)+(\tB:1.8)$) {$w+T_C(w)$};

\draw[red!70!black, very thick, ->] (w) -- ($(w)+(\nang:1.3)$)
  node[anchor=south west] {$w+N_C(w)$};
\draw[thin] ($(w)+(\nang:0.2)$) -- ($(w)+(\nang:0.2)+(\tA:0.2)$) -- ($(w)+(\tA:0.2)$);

\fill (v) circle (1.5pt) node[below left=1.2pt and 2pt] {$v$};
\fill (w) circle (1.5pt) node[left=2pt] {$w$};
\fill (x) circle (1.5pt) node[right] {$x$};
\end{tikzpicture}
\caption{Tangent cones (blue) and normal cones (red) of a convex set $C\subset\mathbb R^2$, drawn translated to the base point.
At the vertex $v$ (non-smooth boundary point) $T_C(v)$ and $N_C(v)=T_C(v)^\circ$ are ``wedges''.
At the smooth boundary point $w$, $T_C(w)$ is the half-plane bounded by the tangent line and $N_C(w)$ is the outward normal half-line.
At the interior point $x$, $T_C(x)=\mathbb R^2$ and $N_C(x)=\{0\}$.
At $v$, a vector $g$ splits orthogonally as $g=\operatorname{Proj}_{T_C(v)}g+\operatorname{Proj}_{N_C(v)}g$ (Moreau decomposition), which is the mechanism behind the projected flow \eqref{eq:ct_projected_lift_flow}.}
\label{fig:tangent_normal_cones}
\end{figure}

For a cone $K\subseteq L^2_\lambda(A)$, its polar is defined by $K^\circ:=\{\eta\in L^2_\lambda(A):\langle\eta,z\rangle_{L^2_\lambda(A)}\leq0 \text{ for all }z\in K\}$. For the constraint set $C$, the cones have the explicit descriptions given by the following lemma.
\begin{lemma}[Tangent and normal cones to the set of probability densities]
\label{lem:ct_density_cones}
Let $\lambda\in\mathcal P(A)$. For every $v\in C$, the tangent and normal cones satisfy
\begin{align*}
T_C(v)&=\left\{z\in {L^2_\lambda(A)}:\int_Az\mathrm d\lambda=0,\
z\geq0\ \text{on }\{v=0\}\right\},\\
N_C(v)&=\left\{\eta\in L^2_\lambda(A):\exists c\in\mathbb R,\
\eta=c\ \text{on }\{v>0\},\ \eta\leq c\ \text{on }\{v=0\}\right\}.
\end{align*}
All pointwise conditions are understood $\lambda$-a.e. Moreover, $T_C(v)$ and $N_C(v)$ are polar to each other, i.e., $N_C(v)=T_C(v)^\circ$ and $T_C(v)=N_C(v)^\circ$.
\end{lemma}

\begin{proof}
Fix $v\in C$. If $\eta\in N_C(v)$, set $c:=\langle\eta,v\rangle_{L^2_\lambda(A)}$. Testing the defining normal cone inequality with $w=\frac{\mathbbm 1_B}{\lambda(B)}$, for every Borel set $B$ with $\lambda(B)>0$, gives
\begin{equation*}
    \frac{1}{\lambda(B)}\int_B\eta\,\mathrm d\lambda\leq c.
\end{equation*}
Hence $\eta\leq c$ a.e. Since $\int_A(c-\eta)v\mathrm d\lambda=0$ and $(c-\eta)v\geq0$, we also have $\eta=c$ on $\{v>0\}$. Conversely, these conditions imply, for every $w\in C$,
\begin{equation*}
    \langle\eta,w-v\rangle_{L^2_\lambda(A)}=\int_A(\eta-c)w\mathrm d\lambda\leq0,
\end{equation*}
which proves the normal cone $N_C(v)$ characterization.

The definitions give $N_C(v)=T_C(v)^\circ$. Since $T_C(v)$ is a closed convex cone, the bipolar theorem implies $T_C(v)=T_C(v)^{\circ\circ}=N_C(v)^\circ$ \cite[Chapter~6]{BauschkeCombettes2017}. Thus, these cones are polar to each other, i.e., each consists precisely of the vectors having nonpositive inner product with every vector in the other.

The characterization just proved expresses every normal vector as $c\mathbf1-\psi$, where $c\in\mathbb R$ and $\psi\in {L^2_\lambda(A)}$ is nonnegative and vanishes outside $\{v=0\}$. The symbol $\mathbf1$ denotes the constant function equal to $1$. Thus $z\in T_C(v)$ if and only if
\begin{equation*}
    c\int_Az\mathrm d\lambda-\int_{\{v=0\}}\psi z\mathrm d\lambda\leq0 \quad\text{for all such }c,\psi.
\end{equation*}
Varying $c$ forces $\int_Az\mathrm d\lambda=0$. The remaining condition is equivalent to $z\geq0$ on $\{v=0\}$: necessity follows by taking $\psi=\mathbbm1_{\{v=0\}}(-z)_+ = \begin{cases} -z,&v=0\text{ and }z<0,\\
0,&\text{otherwise}.\end{cases}$, and sufficiency is immediate.
\end{proof}

\subsection{Density space} 
We will denote the set of bounded policy densities by
\begin{equation*}
\mathcal U:=\Big\{u\in B_b(S\times A):u\geq 0,\, \int_A u(s,a)\,\lambda(\mathrm da) = 1\ \text{for every }s\in S\Big\}.
\end{equation*}

Define the space of bounded Borel measurable maps from the state space $S$ into $L^2_\lambda(A)$ by
\begin{equation*}
    B_b(S;L^2_\lambda(A)):=\Big\{u:S\to L^2_\lambda(A):u\text{ is Borel measurable},\ \sup_{s\in S}\|u(s)\|_{L^2_\lambda}<\infty\Big\}.
\end{equation*}
Writing $u_s=u(s)$, its norm is $\|u\|_{B_b(S;L^2_\lambda(A))}:=\sup_{s\in S} \|u_s\|_{L^2_\lambda(A)}$.
With this norm, it is a Banach space.
Now introduce its closed convex subset
\begin{equation*}
    \mathcal C:=\{u\in B_b(S;L^2_\lambda(A)):u_s\in C\text{ for every }s\in S\}.
\end{equation*}
For $u\in\mathcal C$, define $\pi_u(B|s):=\langle u_s,\mathbbm 1_B\rangle_{L^2_\lambda(A)} =\int_B u_s\mathrm d\lambda$ for every Borel set $B\subseteq A$. 
Clearly, for each fixed $s$, $\pi_u(\cdot|s)$ is a probability measure.
For each fixed Borel set $B$, the map $v\mapsto\langle v,\mathbbm 1_B\rangle_{L^2_\lambda(A)}$ is continuous on $L^2_\lambda(A)$, since $\lambda$ is finite and hence $\mathbbm 1_B\in L^2_\lambda(A)$. Its composition with the Borel map $s\mapsto u_s$ is therefore Borel measurable, proving that $\pi_u$ is a Borel policy kernel.

Furthermore, every bounded, jointly Borel probability density $u\in\mathcal U$ defines an element of $\mathcal C$ through $s\mapsto u(s,\cdot)$.
Indeed, $\int_Au(s,a)^2\lambda(\mathrm da)\leq\|u\|_{B_b(S\times A)}$, and $s\mapsto u(s,\cdot)$ is Borel measurable into $L^2_\lambda(A)$.
To see the latter, note that $L^2_\lambda(A)$ is separable since $A$ is Polish and $\lambda$ is finite.
For any $h\in L^2_\lambda(A)$, the map $s\mapsto\|u(s,\cdot)-h\|_{L^2_\lambda(A)}$ is Borel by Tonelli's theorem.
Preimages of balls centred at a countable dense set, which generate the Borel $\sigma$-algebra, are therefore Borel.

To summarize: Any $u\in \mathcal C$ defines a valid policy $\pi_u$.
Moreover $\mathcal U \subset \mathcal C \subset B_b(S; L^2_\lambda(A))$.

\subsection{Common estimates}
For $u,v\in\mathcal C$, Cauchy--Schwarz and $\lambda(A)=1$ give
\begin{equation*}
\sup_s\|u_s-v_s\|_{L^1_\lambda(A)}\leq \|u-v\|_{B_b(S;L^2_\lambda(A))}.
\end{equation*} 
Write $r_u(s):=\int_A r(s,a)u_s(a)\lambda(\mathrm da)$ and let $P_u$ be the state transition kernel induced by $\pi_u$. Subtracting the Bellman equations gives
\begin{equation*}
(I-\gamma P_u)(V_{\pi_u}-V_{\pi_v})=r_u-r_v+\gamma(P_u-P_v)V_{\pi_v}.
\end{equation*}
The preceding $L_\lambda^1$ bound, $\|V_{\pi_v}\|_{B_b(S)}\leq (1-\gamma)^{-1}\|r\|_{B_b(S\times A)}$, and $\|(I-\gamma P_u)^{-1}\|_{\mathrm{op}}\leq(1-\gamma)^{-1}$ imply
\begin{equation}
\label{eq:lip-value}
\|V_{\pi_u}-V_{\pi_v}\|_{B_b(S)} \leq\frac{\|r\|_{B_b(S\times A)}}{(1-\gamma)^2}\|u-v\|_{B_b(S;L^2_\lambda(A))}.
\end{equation}
Here $\|\cdot\|_{\mathrm{op}}$ is the operator norm on $B_b(S)$ induced by the supremum norm. Using $A_\pi(s,a)=r(s,a)+\gamma\int_S V_\pi(s')P(\mathrm ds'|s,a)
-V_\pi(s)$, we obtain
\begin{equation}
\label{eq:ct_advantage_lipschitz}
\|A_{\pi_u}\|_{B_b(S\times A)}\leq \frac{2\|r\|_{B_b(S\times A)}}{1-\gamma},\
\|A_{\pi_u}-A_{\pi_v}\|_{B_b(S\times A)}\leq \frac{(1+\gamma)\|r\|_{B_b(S\times A)}}{(1-\gamma)^2}\|u-v\|_{B_b(S;L^2_\lambda(A))}.
\end{equation}
Since $\lambda(A)=1$, we also have
\begin{equation}
\begin{aligned}
\label{eq:ct_global_drift_bounds}
&\|A_{\pi_u}\|_{B_b(S;L^2_\lambda(A))} \leq \|A_{\pi_u}\|_{B_b(S\times A)} \leq \frac{2\|r\|_{B_b(S\times A)}}{1-\gamma},\\
&\|A_{\pi_u}-A_{\pi_v}\|_{B_b(S;L^2_\lambda(A))} \leq\|A_{\pi_u}-A_{\pi_v}\|_{B_b(S\times A)} \leq \frac{(1+\gamma)\|r\|_{B_b(S\times A)}}{(1-\gamma)^2}\|u-v\|_{B_b(S;L^2_\lambda(A))}.
\end{aligned}
\end{equation}

\subsection{Limit behaviour of discrete stepping scheme when $\tau \to 0$ and $t\to\infty$ in the flow}

Throughout, recall that
\begin{equation}
\label{eq:ct_centering}
\int_A A_{\pi_u}(s,a)\pi_u(\mathrm{d}a|s)=\int_Au(s,a)A_{\pi_u}(s,a)\lambda(\mathrm da)=0.
\end{equation}
Moreover, recall that for $v\in C$ and $u \in \mathcal C$ we have
\begin{equation}
\label{eq:ct_lift_identity}
\operatorname{FR}^2(v^2,u_s^2)=4\int_A(v(a)-u_s(a))^2\lambda(\mathrm{d}a).
\end{equation}
The lifted FR update thus is
\begin{align}
\mathcal T_\tau(u)(s,\cdot)&:=\underset{v\in C}{\arg\max}\left\{
\frac1{1-\gamma}\int_A A_{\pi_u}(s,a)v(a)\lambda(\mathrm da)
-\frac1{2\tau}\int_A(v(a)-u_s(a))^2\lambda(\mathrm da)\right\}.
\label{eq:ct_lifted_step}
\end{align}

The proof below applies the usual affine and piecewise constant interpolations and follows methods of e.g.~\cite[Section~2]{Santambrogio2017GradientFlows}.

\begin{proposition}[Convergence of the lifted FR update]
\label{prop:ct_projected_global}
Let $u^0\in\mathcal C$. There is a unique $u\in C([0,\infty);\mathcal C)$ with $u(0)=u^0$ whose statewise paths $t\mapsto u_s(t)$ are locally absolutely continuous in $L^2_\lambda(A)$ and satisfy
\begin{equation}
\label{eq:ct_normal_inclusion}
\partial_tu_s(t)+N_C(u_s(t))\ni \frac{1}{1-\gamma} A_{\pi_{u(t)}}(s,\cdot) \quad\text{for every }s\in S,\ \text{for a.e. }t>0.
\end{equation}
This solution is Lipschitz as an $B_b(S;L^2_\lambda(A))$-valued path, and satisfies the projected equation
\begin{equation}
\label{eq:ct_projected_lift_flow}
\partial_tu_s(t) =\operatorname{Proj}^{L^2_\lambda(A)}_{T_C(u_s(t))}\left(\frac{1}{1-\gamma} A_{\pi_{u(t)}}(s,\cdot)\right)\quad\text{for every }s\in S,\ \text{for a.e. }t>0.
\end{equation}
For every $\tau>0$, the iterates
\begin{equation}
\label{eq:ct_global_projected_step}
u^{n+1}_\tau=\mathcal T_\tau(u^n_\tau), \quad u^0_\tau=u^0,
\end{equation}
are well-defined in $\mathcal C$. Their piecewise affine interpolation $\widehat u_\tau$ satisfies, for every $T>0$,
\begin{equation}
\label{eq:ct_projected_global_error}
\sup_{0\leq t\leq T}\|\widehat u_\tau(t)-u(t)\|_{B_b(S;L^2_\lambda(A))} \leq C(\gamma,\|r\|_{B_b(S\times A)},T)\sqrt{\tau},
\end{equation}
for some constant $C(\gamma,\|r\|_{B_b(S\times A)},T)>0$.
\end{proposition}

\begin{proof}
\emph{Discrete estimates.}
Let $L:=\frac{(1+\gamma)\|r\|_{B_b(S\times A)}}{(1-\gamma)^3}$ and $M:=\frac{2\|r\|_{B_b(S\times A)}}{(1-\gamma)^2}$.
By the proof of Lemma \ref{lemma:well-posedness-FR-PPO}, scheme \eqref{eq:ct_lifted_step} gives
\begin{equation*}u^{n+1}_{\tau,s}=\operatorname{Proj}^{L^2_\lambda(A)}_C\Big(u^{n}_{\tau,s}+\frac{\tau}{1-\gamma} A_{\pi_{u^n_\tau}}(s,\cdot)\Big).
\end{equation*}
The set $C$ is nonempty, closed, and convex in $L^2_\lambda(A)$, so its projection $\operatorname{Proj}^{L^2_\lambda(A)}_C:L^2_\lambda(A)\to C$ exists, is given by $\operatorname{Proj}^{L^2_\lambda(A)}_C(\varphi)=\arg\min_{\phi\in C}\|\phi-\varphi\|_{L^2_\lambda(A)}$ for all $\varphi \in L^2_\lambda(A)$, and satisfies 
\begin{equation*}
    \|\operatorname{Proj}^{L^2_\lambda(A)}_C(\varphi_1)-\operatorname{Proj}^{L^2_\lambda(A)}_C(\varphi_2)\|_{L^2_\lambda(A)}
\le \|\varphi_1-\varphi_2\|_{L^2_\lambda(A)},
\end{equation*}
for all $\varphi_1,\varphi_2\in L^2_\lambda(A)$.
Since $\operatorname{Proj}_C^{L^2_\lambda(A)} u^n_{\tau,s}=u^n_{\tau,s}$, by \eqref{eq:ct_global_drift_bounds},
\begin{equation*}
\|u^{n+1}_\tau-u^{n}_\tau\|_{B_b(S;L^2_\lambda(A))}\leq \frac{\tau}{1-\gamma}\|A_{\pi_{u^n_\tau}}\|_{B_b(S;L^2_\lambda(A))}\leq M\tau.
\end{equation*}
In particular, all iterates belong to $\mathcal C$. 
For $n\tau\leq t<(n+1)\tau$, put $\underline u_\tau(t):=u^n_\tau$ and $\overline u_\tau(t):=u^{n+1}_\tau$. Define the piecewise constant velocity and the interpolation 
\begin{equation*}
    v_\tau(t):=\frac{\overline u_\tau(t)-\underline u_\tau(t)}{\tau}, \quad \widehat u_\tau(t):=\underline u_\tau(t)+(t-n\tau)v_\tau(t), \quad [n\tau,(n+1)\tau).
\end{equation*}
Thus on $[n\tau,(n+1)\tau)$ the velocity satisfies $\|v_\tau(t)\|_{B_b(S;L^2_\lambda(A))}\leq M$.
Since $\widehat u_\tau$ is continuous at the grid points, summing its affine increments gives, for every $0\leq t_1<t_2$,
\begin{equation*}
\widehat u_\tau(t_2)-\widehat u_\tau(t_1)=\int_{t_1}^{t_2}v_\tau(r)\mathrm dr.
\end{equation*}
Hence,
\begin{equation}
\label{eq:unif-lip}
\|\widehat u_\tau(t_2)-\widehat u_\tau(t_1)\|_{B_b(S;L^2_\lambda(A))}
\leq\int_{t_1}^{t_2}\|v_\tau(r)\|_{B_b(S;L^2_\lambda(A))}\mathrm dr
\leq M|t_2-t_1|,
\end{equation}
so $t\mapsto \widehat u_\tau$ is globally $M$-Lipschitz in $B_b(S;L^2_\lambda(A))$, uniformly in $\tau>0$.

Define $\eta_{\tau}(t):=\frac{1}{1-\gamma}A_{\pi_{\underline u_\tau(t)}}-v_{\tau}(t)$.
Then
\begin{equation}
\label{eq:ct_interpolation_bounds}
\begin{split}
\left\|\eta_{\tau}(t)\right\|_{B_b(S;L^2_\lambda(A))} & \leq 2M,\,\\
\|\widehat u_\tau(t)-\underline u_\tau(t)\|_{B_b(S;L^2_\lambda(A))} & \leq M \tau,\,
\|\widehat u_\tau(t)-\overline u_\tau(t)\|_{B_b(S;L^2_\lambda(A))} \leq M \tau.
\end{split}
\end{equation}
\emph{Comparison of two meshes.}
Fix $T>0$ and two stepsizes $\tau,\sigma>0$. By the projection characterization \cite[Theorem~3.16]{BauschkeCombettes2017},
\begin{equation}
\label{eq:ct_discrete_normal}
\eta_{\tau,s}(t)\in N_C(\overline u_{\tau,s}(t)), \quad \eta_{\sigma,s}(t)\in N_C(\overline u_{\sigma,s}(t))\,.
\end{equation}
Then the definition of the normal cone~\eqref{eq:normal_cone} applied to~\eqref{eq:ct_discrete_normal} gives
\begin{equation}
\label{eq:monot-cone}
\langle\overline u_{\tau,s}(t)-\overline u_{\sigma,s}(t),\eta_{\tau,s}(t)-\eta_{\sigma,s}(t)\rangle_{L_\lambda^2(A)}\geq 0\,,
\end{equation}
which we'll use in a moment.
Separately, using \eqref{eq:ct_interpolation_bounds}, it follows that
\begin{align*}
&\|(\widehat u_{\tau,s}-\overline u_{\tau,s})-(\widehat u_{\sigma,s}-\overline u_{\sigma,s})\|_{L_\lambda^2(A)}\leq  M(\tau+\sigma),\quad
\|\eta_{\tau,s}-\eta_{\sigma,s}\|_{L_\lambda^2(A)} \leq 4M.
\end{align*}
Hence, using \eqref{eq:monot-cone} and Cauchy--Schwarz,
\begin{equation*}
\begin{aligned}
&-\langle \widehat u_{\tau,s}-\widehat u_{\sigma,s},\eta_{\tau,s}-\eta_{\sigma,s}\rangle_{L_\lambda^2(A)}\\
&=-\left\langle
\overline u_{\tau,s}-\overline u_{\sigma,s},
\eta_{\tau,s}-\eta_{\sigma,s}
\right\rangle_{L_\lambda^2(A)}
-\langle (\widehat u_{\tau,s}-\overline u_{\tau,s})-(\widehat u_{\sigma,s}-\overline u_{\sigma,s}),\eta_{\tau,s}-\eta_{\sigma,s}\rangle_{L_\lambda^2(A)}\\
&\leq-\langle (\widehat u_{\tau,s}-\overline u_{\tau,s})-(\widehat u_{\sigma,s}-\overline u_{\sigma,s}),\eta_{\tau,s}-\eta_{\sigma,s}\rangle_{L_\lambda^2(A)}\\
&\leq\|(\widehat u_{\tau,s}-\overline u_{\tau,s})-(\widehat u_{\sigma,s}-\overline u_{\sigma,s})\|_{L_\lambda^2(A)}
\|\eta_{\tau,s}-\eta_{\sigma,s}\|_{L_\lambda^2(A)}
\leq 4M^2(\tau+\sigma).\\
\end{aligned}
\end{equation*}
Set $E(t):=\|\widehat u_\tau(t)-\widehat u_\sigma(t)\|_{B_b(S;L^2_\lambda(A))}$. Next, for each fixed $s$, away from the grid times, the definitions give
\begin{equation*}
\partial_t\widehat u_{\tau,s}(t)=v_{\tau,s}(t)=\frac{1}{1-\gamma}A_{\pi_{\underline u_\tau}}(s,\cdot)-\eta_{\tau,s}(t).
\end{equation*}
Therefore, the previous inequality together with \eqref{eq:ct_global_drift_bounds}  and \eqref{eq:ct_interpolation_bounds} now gives for every $s\in S$ and every $t$ which is not a grid point that
\begin{align*}
& \tfrac12\tfrac{\mathrm d}{\mathrm dt}\|\widehat u_{\tau,s}(t)-\widehat u_{\sigma,s}(t)\|_{L^2_\lambda(A)}^2
=\langle \widehat u_{\tau,s}(t)-\widehat u_{\sigma,s}(t),  \partial_t \widehat{u}_{\tau, s}(t)-\partial_t \widehat u_{\sigma,s}(t)\rangle_{L^2_\lambda(A)}\\ 
&=\frac{1}{1-\gamma}\langle \widehat u_{\tau,s}-\widehat u_{\sigma,s},A_{\pi_{\underline u_\tau}}(s,\cdot)-A_{\pi_{\underline u_\sigma}}(s,\cdot)\rangle_{L^2_\lambda(A)}-\langle \widehat u_{\tau,s}-\widehat u_{\sigma,s},\eta_{\tau,s}-\eta_{\sigma,s}\rangle_{L^2_\lambda(A)}\\
& \leq L\,E(t)\big(E(t)+M(\tau+\sigma)\big) + 4M^2(\tau+\sigma)\, \\
& =  L E(t)^2 + L M E(t)(\tau+\sigma) + 4M^2(\tau+\sigma)\,.
\end{align*}

Both interpolations start at $u^0$ and are $M$-Lipschitz, so $E(0) = 0 $ and $E(t)\leq 2 M t$ for $t \geq 0$.
Integrating for each $s\in S$ and then taking the supremum over $s\in S$ gives 
\begin{equation*}
E(t)^2\leq 2 L\int_0^t E(r)^2\mathrm dr + 
 2 LM^2  (\tau+\sigma)t^2 
+ 8 M^2(\tau+\sigma)t\,
\end{equation*}
for all $t\geq 0$.
Hence for $0\leq t\leq T$ we get
\[
E(t)^2\leq 2L\int_0^t E(r)^2\mathrm dr + 
2 LM^2  (\tau+\sigma)T^2 
+ 8 M^2(\tau+\sigma)T\,.
\]
Gr\"onwall's inequality implies
\begin{equation}
\label{eq:ct_two_mesh_error}
\sup_{0\leq t\leq T}
\|\widehat u_\tau(t)-\widehat u_\sigma(t)\|_{B_b(S;L^2_\lambda(A))}^2 = \sup_{0\leq t\leq T} E(t)^2
\leq (2 L M^2 T^2 
+ 8M^2 T) e^{2LT}(\tau+\sigma)\,.
\end{equation}

Since $B_b(S;L^2_\lambda(A))$ is a Banach space, $Y_T:=C([0,T];B_b(S;L^2_\lambda(A)))$ is complete under the norm
$\|f\|_{Y_T}:=\sup_{0\leq t\leq T}\|f(t)\|_{B_b(S;L^2_\lambda(A))}$ \cite[Theorem~43.6]{Munkres2000}. For any sequence $\tau_k\downarrow0$, the two-mesh estimate \eqref{eq:ct_two_mesh_error} shows that $(\widehat u_{\tau_k})_k$ is Cauchy in $Y_T$ and hence converges to some $u_T \in Y_T$.
Letting $\sigma=\tau_k\to0$ in that estimate gives
\begin{equation*}
\|\widehat u_\tau-u_T\|_{Y_T}^2\leq (2 L M^2 T^2 
+ 8M^2 T) e^{2LT}\tau,
\end{equation*}
so the entire family converges as $\tau\downarrow0$. Each affine interpolation takes values in $\mathcal C$ by convexity, and closedness of $\mathcal C$ therefore implies $u_T(t)\in\mathcal C$ for every $t\in[0,T]$. Passing to the limit in the uniform Lipschitz estimate \eqref{eq:unif-lip} gives
\begin{equation*}
    \|u_T(t)-u_T(r)\|_{B_b(S;L^2_\lambda(A))}\leq M|t-r|, \quad r,t\in[0,T].
\end{equation*}
Finally, if $T_1<T_2$, both $u_{T_1}$ and the restriction of $u_{T_2}$ to $[0,T_1]$ are uniform limits of the same interpolations and hence coincide. Taking $T=1,2,\ldots$ defines a single $M$-Lipschitz path $u:[0,\infty)\to\mathcal C$.

\emph{Identification of the limit.}
Fix $s\in S$ and $T>0$.
The path $t\mapsto u_s(t)$ is $M$-Lipschitz in $L^2_\lambda(A)$, hence absolutely continuous.
Given a sequence $\tau_k\downarrow0$, the bound $\|v_{\tau_k,s}\|_{L^2(0,T;L^2_\lambda(A))}\leq M\sqrt T$ and reflexivity of the Hilbert space $L^2(0,T;L^2_\lambda(A))$ yield, along a subsequence not relabelled, $v_{\tau_k,s}\rightharpoonup w_s$ weakly\footnote{We shall write $\rightharpoonup$ to denote weak convergence in Banach spaces.} in $L^2(0,T;L^2_\lambda(A))$ \cite[Theorem~3.18]{Brezis2011}. For every $\varphi\in C_c^\infty(0,T)$ and $h\in L^2_\lambda(A)$, integration by parts gives
\begin{equation*}
\int_0^T\varphi'(t)\langle\widehat u_{\tau_k,s}(t),h\rangle_{L^2_\lambda(A)}\mathrm dt = -\int_0^T\varphi(t)
\langle v_{\tau_k,s}(t),h\rangle_{L^2_\lambda(A)}\mathrm dt.
\end{equation*}
Passing to the limit using uniform strong convergence of the interpolations and weak convergence of the velocities shows that $w_s$ is the weak time derivative of $u_s$. Thus $u_s\in W^{1,2}(0,T;L^2_\lambda(A))$, and the fundamental theorem of calculus gives
\begin{equation*}
    u_s(t)=u_s(0)+\int_0^t w_s(r)\mathrm dr, \quad t\in[0,T],
\end{equation*}
so $w_s=\partial_tu_s$ also in the strong sense almost everywhere \cite[Section~5.9.2, Theorem~2]{Evans2010}. Furthermore, \eqref{eq:ct_interpolation_bounds} implies $\underline u_\tau,\overline u_\tau\to u$ uniformly on $[0,T]$ in $B_b(S;L^2_\lambda(A))$. The drift estimate \eqref{eq:ct_global_drift_bounds} therefore yields $A_{\pi_{\underline u_\tau}}\to A_{\pi_u}$ uniformly on $[0,T]$ in $B_b(S;L^2_\lambda(A))$, and hence strongly in $L^2(0,T;L^2_\lambda(A))$ at the fixed state $s$. Therefore,
\begin{equation*}
\eta_{\tau_k,s}=\frac{1}{1-\gamma}A_{\pi_{\underline u_{\tau_k}}}(s,\cdot)-v_{\tau_k,s}
\rightharpoonup \eta_s:=\frac{1}{1-\gamma}A_{\pi_u}(s,\cdot)-\partial_t u_s \quad\text{in }L^2(0,T;L^2_\lambda(A)).
\end{equation*}
Fix $s\in S$.
For every $w\in C$ and nonnegative $\psi\in L^\infty(0,T)$, \eqref{eq:normal_cone}and~\eqref{eq:ct_discrete_normal} give
\begin{equation*}
    \int_0^T\psi(t)\langle\eta_{\tau_k,s}(t),w-\overline u_{\tau_k,s}(t)\rangle_{L_\lambda^2(A)}\mathrm dt\leq0.
\end{equation*}
Since $\eta_{\tau_k,s}\rightharpoonup\eta_s$ in $L^2(0,T;L_\lambda^2(A))$ and $\psi(w-\overline u_{\tau_k,s})\to\psi(w-u_s)$ strongly in the same space, convergence of the duality pairing \cite[Proposition~3.5 (iv)]{Brezis2011} yields
\begin{equation*}
    \int_0^T\psi(t)\langle\eta_s(t),w-u_s(t)\rangle_{L_\lambda^2(A)}\mathrm dt\leq0.
\end{equation*}
Set $g_w(t):=\langle\eta_s(t),w-u_s(t)\rangle_{L_\lambda^2(A)}\in L^1(0,T)$. Choosing $\psi=\mathbbm1_{\{g_w>0\}}$ gives $\int_0^T(g_w)_+\mathrm dt\leq0$, hence $g_w\leq0$ almost everywhere. Since $L_\lambda^2(A)$ is separable, its subset $C$ admits a countable dense family $(w_j)_{j\geq1}$ \cite[Proposition~3.25]{Brezis2011}. For each $j$, choose an exceptional null set $N_{s,j}$ outside which $g_{w_j}\leq0$, and set $N_s:=\bigcup_{j\geq1}N_{s,j}$. Then $N_s$ is null, and for every $t\notin N_s$, continuity of $w\mapsto\langle\eta_s(t),w-u_s(t)\rangle_{L_\lambda^2(A)}$ extends the inequality from $(w_j)_{j\geq1}$ to all $w\in C$. Therefore $\eta_s(t)\in N_C(u_s(t))$ for almost every $t$. Together with $\eta_s=\frac{1}{1-\gamma}A_{\pi_u}(s,\cdot)-\partial_tu_s$, this establishes \eqref{eq:ct_normal_inclusion} for every state, with an exceptional null set that may depend on $s$.

Fix $s\in S$ and a time $t>0$ at which $u_s$ is differentiable and \eqref{eq:ct_normal_inclusion} holds. For $h>0$, because $u_s(t+h)\in C$, every forward difference quotient has exactly the form used to define the tangent cone. Indeed, recall that
\begin{equation*}
T_C(u_s(t)) =\overline{\bigl\{\alpha(w-u_s(t)):\alpha\geq0,\ w\in C\bigr\}
}^{L_\lambda^2(A)}.
\end{equation*}
For each $h>0$, choose $\alpha=\frac1h\geq0$, and $w=u_s(t+h)\in C$. Then
\begin{equation*}
    \frac{u_s(t+h)-u_s(t)}{h}=\alpha(w-u_s(t))\in T_C(u_s(t)).
\end{equation*}
Passing to the limit $h\to 0$ and using the closedness of the tangent cone gives $\partial_tu_s(t) \in T_C(u_s(t))$. Moreover, since $\eta_s(t)\in N_C(u_s(t))$ and $u_s(t+h)\in C$, we have
\begin{equation*}
    \langle\eta_s(t),u_s(t+h)-u_s(t)\rangle_{L_\lambda^2(A)}\leq0
\end{equation*}
for both signs of sufficiently small $h$. Dividing by $h>0$ and $h<0$ and taking the respective limits yields $\langle\eta_s(t),\partial_tu_s(t)\rangle_{L_\lambda^2(A)}=0$. Hence, the decomposition
\begin{equation*}
    \frac{1}{1-\gamma}A_{\pi_{u(t)}}(s,\cdot)=\partial_tu_s(t)+\eta_s(t)
\end{equation*}
has orthogonal components in $T_C(u_s(t))$ and $N_C(u_s(t))=T_C(u_s(t))^\circ$ by Lemma \ref{lem:ct_density_cones}. Moreau's decomposition theorem for closed convex cones \cite[Theorem~6.30]{BauschkeCombettes2017} therefore gives
\begin{equation*}
\partial_tu_s(t)=\operatorname{Proj}_{T_C(u_s(t))}^{L_\lambda^2(A)}\left(\frac{1}{1-\gamma}A_{\pi_{u(t)}}(s,\cdot)\right),
\end{equation*}
which proves \eqref{eq:ct_projected_lift_flow}.

\emph{Uniqueness.}
Let $u,\widetilde u$ be two solutions in $C([0,\infty);\mathcal{C})$, with initial conditions $u(0), \widetilde u(0)$, respectively, and set
$e_s(t):=u_s(t)-\widetilde u_s(t)$, $E(t):=\|u(t)-\widetilde u(t)\|_{B_b(S;L_\lambda^2(A))}$. For each fixed $s$, the normal cone inclusions give, for almost every $t$,
\begin{equation*}
\partial_tu_s=\frac{1}{1-\gamma}A_{\pi_u}(s,\cdot)-\eta_s, \quad \partial_t\widetilde u_s=\frac{1}{1-\gamma}A_{\pi_{\widetilde u}}(s,\cdot)-\widetilde\eta_s,
\end{equation*}
where $\eta_s\in N_C(u_s)$ and $\widetilde\eta_s\in N_C(\widetilde u_s)$. Since both trajectories take values in $C$, the defining normal cone inequality~\eqref{eq:normal_cone} yields
\begin{equation*}
\langle\eta_s,\widetilde u_s-u_s\rangle_{L_\lambda^2(A)}\leq0, \quad \langle\widetilde\eta_s,u_s-\widetilde u_s\rangle_{L_\lambda^2(A)}\leq0.
\end{equation*}
Adding these inequalities gives the monotonicity relation
\begin{equation*}
    \langle\eta_s-\widetilde\eta_s,e_s\rangle_{L_\lambda^2(A)}\geq0.
\end{equation*}
The chain rule followed by Cauchy--Schwarz and \eqref{eq:ct_global_drift_bounds}, therefore gives
\begin{align*}
&\frac12\frac{\mathrm d}{\mathrm dt}\|e_s(t)\|_{L_\lambda^2(A)}^2\\
&=\frac{1}{1-\gamma}\langle e_s(t), A_{\pi_{u(t)}}(s,\cdot)-A_{\pi_{\widetilde u(t)}}(s,\cdot)\rangle_{L_\lambda^2(A)}
-\langle e_s(t),\eta_s(t)-\widetilde\eta_s(t)\rangle_{L_\lambda^2(A)}\\
&\leq\frac{1}{1-\gamma}\|e_s(t)\|_{L_\lambda^2(A)}\|A_{\pi_{u(t)}}(s,\cdot)-A_{\pi_{\widetilde u(t)}}(s,\cdot)\|_{L_\lambda^2(A)}\\
&\leq L E(t)^2
\end{align*}
for almost every $t$. Integrating this inequality at each fixed state gives, for every $t\geq0$,
\begin{equation*}
\|e_s(t)\|_{L_\lambda^2(A)}^2\leq \|e_s(0)\|_{L_\lambda^2(A)}^2+2L\int_0^t E(r)^2\mathrm dr.
\end{equation*}
The integral on the right is independent of $s$, so taking the supremum over states yields
\begin{equation*}
    E(t)^2\leq E(0)^2+2L\int_0^t E(r)^2\mathrm dr.
\end{equation*}
Since $E$ is continuous, Gronwall's inequality implies $E(t)^2\leq e^{2Lt}E(0)^2$, and hence
\begin{equation*}
\|u(t)-\widetilde u(t)\|_{B_b(S;L_\lambda^2(A))} \leq e^{Lt}\|u(0)-\widetilde u(0)\|_{B_b(S;L_\lambda^2(A))}.
\end{equation*}
In particular, solutions with the same initial condition
coincide. This proves uniqueness of the lifted solution.
\end{proof}

Let $u\in C([0,\infty);\mathcal C)$ be the lifted projected solution of
Proposition~\ref{prop:ct_projected_global}. For a fixed comparison density $v\in\mathcal C$ and initial state distribution $\rho\in\mathcal P(S)$, set
\begin{equation*}
\mathcal E_\rho^v(t)
:=\frac12\int_S\|u_s(t)-v_s\|_{L_\lambda^2(A)}^2d^{\pi_v}_\rho(\mathrm ds)
= \frac18\int_S\operatorname{FR}^2(\pi_{u(t)}(\cdot|s)^2,\pi_v(\cdot|s)^2)\,d^{\pi_v}_\rho(\mathrm ds)\,.
\end{equation*}

\begin{proposition}[Long-time convergence of the lifted FR flow]
\label{prop:ct_long_time}
Let $u\in C([0,\infty);\mathcal C)$ be the flow given by
Proposition~\ref{prop:ct_projected_global}.
For every $s\in S$, the map $t\mapsto V_{\pi_{u(t)}}(s)$ is non-decreasing.
For every $v\in\mathcal C$, $\rho\in\mathcal P(S)$, and $t>0$,
\begin{equation}
\label{eq:ct_comparator_rate}
(V_{\pi_v}-V_{\pi_{u(t)}})(\rho)\leq\frac{\mathcal E_\rho^v(0)}{t}.
\end{equation}
Suppose, in addition, that there is $u^*\in\mathcal C$ such that $V_{\pi_{u^*}}(s)=V^*(s)$ for every $s\in S$. Then
\begin{equation}
\label{eq:ct_optimal_value_rate}
0 \leq (V^*-V_{\pi_{u(t)}})(\rho) \leq\frac{\mathcal E_\rho^{u^*}(0)}{t}, \quad \|V^*-V_{\pi_{u(t)}}\|_\infty\leq\frac{1}{2t}\|u^0-u^*\|_{B_b(S;L_\lambda^2(A))}^2, \quad t>0.
\end{equation}
\end{proposition}

\begin{proof}
We first establish policy improvement without differentiating the
value function along an $B_b(S;L_\lambda^2(A))$-valued path. 

To that end we establish policy improvement for the iterates in~\eqref{eq:ct_global_projected_step}. Multiplying the objective in \eqref{eq:ct_lifted_step} by $1-\gamma$ and using \eqref{eq:ct_lift_identity} and \eqref{eq:ct_centering} shows that, under the identification $m=v\lambda$ between $\mathfrak C_\lambda$ and $C$, it coincides with the statewise problem \eqref{eq:scheme_pointwise}. Hence $\pi_{u^{n+1}_\tau}$ is the statewise FR--PPO update of $\pi_{u^n_\tau}$.
Moreover, $u^n_\tau\in\mathcal C$, since $s\mapsto u^n_{\tau,s}+\frac{\tau}{1-\gamma}A_{\pi_{u^n_\tau}}(s,\cdot)$ is Borel into $L^2_\lambda(A)$ and $\operatorname{Proj}_C$ is continuous, so $\pi_{u^n_\tau}\in\mathcal P_\lambda(A|S)$. Theorem~\ref{thm policy improvement} with $\rho=\delta_s$ therefore gives $V^{n+1}_\tau(s)\geq V^n_\tau(s)$ for every $s\in S$ and $n\in\mathbb N_0$.

Hence, the piecewise constant value approximation $W_\tau(t):=V^{\lfloor t/\tau\rfloor}_\tau$ satisfies $W_\tau(t_1,s)\leq W_\tau(t_2,s)$ whenever
$0\leq t_1\leq t_2$. To pass this inequality to the flow, fix $T>0$. By \eqref{eq:ct_projected_global_error} and \eqref{eq:ct_interpolation_bounds},
\begin{equation*}
    \sup_{0\leq t\leq T}\left\|u^{\lfloor t/\tau\rfloor}_\tau-u(t)\right\|_{B_b(S;L_\lambda^2(A))}\leq M\tau+C_T\sqrt{\tau},
\end{equation*}
where $C_T$ is the constant in \eqref{eq:ct_projected_global_error}. The Lipschitz estimate for the value function \eqref{eq:lip-value} therefore gives
\begin{equation*}
    \sup_{0\leq t\leq T}\|W_\tau(t)-V_t\|_\infty\leq\frac{\|r\|_{B_b(S\times A)}}{(1-\gamma)^2}\left(M\tau+C_T\sqrt{\tau}\right)=:\varepsilon_\tau(T) \to 0 \,\text{as}\,\tau \to 0.
\end{equation*}
For any $0\leq t_1\leq t_2\leq T$ and $s\in S$, it follows that
\begin{equation*}
V_{t_1}(s)\leq W_\tau(t_1,s)+\varepsilon_\tau(T)\leq W_\tau(t_2,s)+\varepsilon_\tau(T)\leq V_{t_2}(s)+2\varepsilon_\tau(T).
\end{equation*}
Letting $\tau\downarrow0$ proves $V_{t_1}(s)\leq V_{t_2}(s)$. Since $T$ is arbitrary, $t\mapsto V_t(s)$ is nondecreasing on $[0,\infty)$ for every state $s$.

Let us now move to proving the rate at which the flow outperforms a comparator~\eqref{eq:ct_comparator_rate}.
Fix $v\in\mathcal C$. For each state $s$, recall that $\eta_s(t)=\frac{1}{1-\gamma}A_{\pi_{u(t)}}(s,\cdot)-\partial_tu_s(t) \in N_C(u_s(t))$ at almost every time. The chain rule in $L_\lambda^2(A)$, the normal cone inequality \eqref{eq:ct_normal_inclusion}, and \eqref{eq:ct_centering} give
\begin{align*}
& \tfrac{\mathrm d}{\mathrm dt}\tfrac12\|u_s(t)-v_s\|_{L_\lambda^2(A)}^2
=\left\langle u_s(t)-v_s,\tfrac{1}{1-\gamma}A_{\pi_{u(t)}}(s,\cdot)-\eta_s(t)\right\rangle_{L_\lambda^2(A)}\\
&\leq\left\langle u_s(t)-v_s,\tfrac{1}{1-\gamma}A_{\pi_{u(t)}}(s,\cdot)\right\rangle_{L_\lambda^2(A)}
=-\frac{1}{1-\gamma}\int_A A_{\pi_{u(t)}}(s,a)v_s(a)\lambda(\mathrm da).
\end{align*}
Let $t_0\geq 0$ and integrate this inequality over $[t_0,t]$ and then against the fixed measure $d^{\pi_v}_\rho$.
The integrands are measurable and bounded on finite time intervals, so Fubini's theorem applies.
The performance difference identity is
\begin{align*}
(V_{\pi_v}-V_r)(\rho)
&=\frac1{1-\gamma}\int_S\langle A_{\pi_r}(s,\cdot),v_s\rangle_{L_\lambda^2(A)} d^{\pi_v}_\rho(\mathrm ds).
\end{align*}
This proves
\begin{equation}
\label{eq:ct_integrated_value_gap}
\mathcal E_\rho^v(t) +\int_{t_0}^t (V_{\pi_v}-V_r)(\rho)\mathrm dr \leq\mathcal E_\rho^v(t_0).
\end{equation}
Since $V_r(\rho)$ is nondecreasing, taking $t_0 =0$, we have
\begin{equation*}
    t(V_{\pi_v}-V_t)(\rho)\leq\int_0^t(V_{\pi_v}-V_r)(\rho)\mathrm dr\leq\mathcal E_\rho^v(0)-\mathcal E_\rho^v(t)\leq\mathcal E_\rho^v(0).
\end{equation*}
This establishes~\eqref{eq:ct_comparator_rate}.
\end{proof}

\begin{remark}[Attainment and convergence of policies]
\label{rem:ct_optimal_attainment}
The optimal density assumption in \eqref{eq:ct_optimal_value_rate} is automatic for finite actions when $\lambda$ has full support, but not for general Polish action spaces. Without it, even convergence to the unrestricted optimal value can fail. For example, take one state, $A=[0,1]$, Lebesgue $\lambda$, $\gamma=0$, and $r(a)=\mathbbm 1_{\{1\}}(a)$. Every density policy has value zero, and its advantage is zero in $L^2_\lambda$, so the projected flow is constant. The unrestricted optimal policy is $\delta_1$, with value one. Even for the continuous reward $r(a)=a$, the only optimal policy is $\delta_1$, which has no $L^2$ density. The value estimate alone therefore does not establish norm convergence of policies. Relative compactness of the trajectory is automatic in the finite tabular setting.
\end{remark}

\begin{remark}[FR, KL and lifted FR expansions]
\label{rem:ct_matching}
Let $u$ be a density bounded away from zero and $v$ a probability density with $\|v-u\|_{B_b(A)}\to0$. Then
\begin{equation*}
\begin{aligned}
\operatorname{FR}^2(v,u) &=\int_A\frac{(v-u)^2}{u}\,\mathrm d\lambda+O\Big(\|v-u\|_{B_b(A)}\int_A(v-u)^2\,\mathrm d\lambda\Big),\\
\operatorname{KL}(v\|u) &=\frac12\int_A\frac{(v-u)^2}{u}\,\mathrm d\lambda+O\Big(\|v-u\|_{B_b(A)}\int_A(v-u)^2\,\mathrm d\lambda\Big),\\
\operatorname{FR}^2(v^2,u^2) &=4\int_A(v-u)^2\,\mathrm d\lambda=\int_A\frac{(v-u)^2}{u}\,4u\,\mathrm d\lambda \,.
\end{aligned}
\end{equation*}
The first two are $f$-divergences, for which $D_f(v\|u)=\frac{f''(1)}{2}\int_A\frac{(v-u)^2}{u}\mathrm d\lambda+\dots$ \cite{schulman2020approximating}.
Since only the quadratic part of the penalty affects the $O(\tau)$ update, replacing the penalty in \eqref{eq:pointwise_min} by $\frac1{8\tau}\operatorname{FR}^2(\pi,\pi^n)$ or $\frac1{4\tau}\operatorname{KL}(\pi\|\pi^n)$ would be expected to lead, though we don't show this here, to the same limit 
\begin{equation*}
\partial_t\pi_t(\mathrm da|s)=\frac4{1-\gamma}A_{\pi_t}(s,a)\pi_t(\mathrm da|s),
\end{equation*}
studied in \cite{kerimkulov2024fisherraogradientflowentropyregularised}.
The lifted $\operatorname{FR}^2$ is not an $f$-divergence but the Bregman divergence generated by $\frac12\chi^2(\cdot|\lambda)$ and we see it has a different second order term.
Consequently it yields the limit \eqref{eq:ct_projected_lift_flow}, which in the interior of $C$ reads $\partial_tu_s=\frac1{1-\gamma}\big(A_{\pi_u}(s,\cdot)-\int_AA_{\pi_u}(s,a)\lambda(\mathrm da)\big)$.
\end{remark}

\section{TV-penalized policy updates: TV-PPO}
\label{app:F}
We recall that the maximum in the TV-PPO scheme \eqref{TV PPO} can be attained by the following pointwise optimization in $s \in S$:
\label{appendix:TV PPO}
\begin{equation*}
\pi^{n+1}(\cdot|s)= \underset{m \in \mathcal P(A)}{\operatorname{argmax}}\left[\int_A A_{\pi^n}(s,a)\frac{\mathrm{d}m}{\mathrm{d}\pi^n}(a|s)\pi^n(\mathrm{d}a|s) - \frac{1}{\tau}\operatorname{TV}(m, \pi^n(\cdot|s))\right],
\end{equation*}
where $\tau > 0$.
\begin{lemma}[Policy improvement for TV--PPO]
Let $V_n := V_{\pi^n}$ for $n\in \mathbb N$ and $(\pi^n)_{n\in \mathbb N_0} \subset \mathcal{P}(A)$ be given by \eqref{TV PPO}. If $\tau > 0$, then for any $\rho \in \mathcal P(S)$ we have 
\begin{equation*}
    V_{n+1}(\rho) \geq  V_n(\rho), \text{ for all } n > 0.
\end{equation*}
\end{lemma}
\begin{proof}
From the update \eqref{TV PPO} we have, for all $\pi(\cdot|s) \in \mathcal{P}(A)$ and $s\in S$ that
\begin{align*}
&\int_A A_{\pi^n}(s,a)\frac{\mathrm{d}\pi^{n+1}}{\mathrm{d}\pi^n}(a|s)\pi^n(\mathrm{d}a|s)- \frac{1}{\tau} \operatorname{TV}(\pi^{n+1}(\cdot|s),\pi^n(\cdot|s))\\ 
&\geq \int_A A_{\pi^n}(s,a)\frac{\mathrm{d}\pi}{\mathrm{d}\pi^n}(a|s)\pi^n(\mathrm{d}a|s)- \frac{1}{\tau} \operatorname{TV}(\pi(\cdot|s),\pi^n(\cdot|s)).
\end{align*} 
This with $\pi = \pi^n$ gives that for all $s\in S$ we have 
\begin{equation}
\label{eq:imp_proof_tv_ppo2}
\int_A A_{\pi^n}(s,a)\frac{\mathrm{d}\pi^{n+1}}{\mathrm{d}\pi^n}(a|s)\pi^n(\mathrm{d}a|s)- \frac{1}{\tau} \operatorname{TV}(\pi^{n+1}(\cdot|s),\pi^n(\cdot|s)) \geq 0.
\end{equation}
From this, the fact that $\operatorname{TV}(\pi^{n+1}(\cdot|s),\pi^n(\cdot|s))\geq 0$ and Lemma~\ref{lemma perf diff} we get
\begin{equation*}
\begin{split}
&V_{n+1}(\rho) - V_{n}(\rho) = \frac{1}{1-\gamma} \int_S \int_A A_{\pi^n}(s,a)(\pi^{n+1}-\pi^n)(\mathrm{d}a|s)d_{\rho}^{\pi^{n+1}}(\mathrm{d}s)\\
& \geq \frac{1}{1-\gamma} \int_S \bigg( \int_A A_{\pi^n}(s,a)(\pi^{n+1}-\pi^n)(\mathrm{d}a|s) - \frac{1}{\tau}\operatorname{TV}(\pi^{n+1}(\cdot|s),\pi^n(\cdot|s))\bigg)d_{\rho}^{\pi^{n+1}}(\mathrm{d}s)\\
&\geq 0.
\end{split}
\end{equation*}
Thus $V_{n+1}(\rho) \geq V_{n}(\rho)$.
\end{proof}
\begin{lemma}[First-order condition for TV--PPO]
\label{first order tv ppo}
For any $s \in S$, the maximizer $\pi^{n+1}(\cdot|s)$ of \eqref{TV PPO} satisfies the first-order condition
\begin{equation*}
A_{\pi^n}(s,a) - \frac{1}{2\tau}\operatorname{sign}\left(\frac{\mathrm{d}\pi^{n+1}}{\mathrm{d}\pi^n}(a|s) - 1\right) = C_n \in \mathbb R, \quad \pi^{n+1}(\cdot|s)\text{-a.e.}
\end{equation*}
where $C_n \coloneqq \int_A\left(A_{\pi^n}(s,a) - \frac{1}{2\tau}\operatorname{sign}\left(\frac{\mathrm{d}\pi^{n+1}}{\mathrm{d}\pi^n}(a|s) - 1\right)\right)\pi^{n+1}(\mathrm{d}a|s)$ is a normalizing constant, for each $n \in \mathbb N$.
\end{lemma}
\begin{proof}
For simplicity, we restrict the optimization in \eqref{TV PPO} to probability measures that are absolutely continuous with respect to Lebesgue measure, denoted by $\mathcal{P}_\text{Leb}(A)$. Let $s \in S$. 
Fix $\pi^{n}(a|s)$. For any $\pi' \in \mathcal{P}_\text{Leb}(A)$, any $\pi \neq \pi^{n}(\cdot|s)$ a.e., and any $\varepsilon \in (0,1)$, \citep[Lemma 2.1]{tsybakov2008introduction} gives
\begin{align*}
&\lim_{\varepsilon \to 0}\frac{1}{\varepsilon}\left(\operatorname{TV}(\pi + \varepsilon(\pi'-\pi), \pi^{n}(\cdot|s)) - \operatorname{TV}(\pi, \pi^{n}(\cdot|s))\right)\\&=\lim_{\varepsilon \to 0}\frac{1}{2\varepsilon}\int_{\mathbb R^d}\left(\left|\pi(a)-\pi^{n}(a|s) + \varepsilon(\pi'(a)-\pi(a))\right| - \left|\pi(a) - \pi^{n}(a|s)\right|\right)\mathrm{d}a.
\end{align*}
Since $|\cdot|$ is differentiable at every $v\neq 0$ with derivative $\operatorname{sign}(v)$, we obtain by dominated convergence 
\begin{align*}
&\lim_{\varepsilon \to 0}\frac{1}{\varepsilon}\left(\operatorname{TV}(\pi + \varepsilon(\pi'-\pi), \pi^{n}(\cdot|s)) - \operatorname{TV}(\pi, \pi^{n}(\cdot|s))\right)\\
&=\frac{1}{2}\int_{\mathbb R^d} \operatorname{sign}\left(\pi(a)-\pi^{n}(\cdot|s)(a)\right)(\pi'-\pi)(\mathrm{d}a).
\end{align*}
To justify dominated convergence, note that for every $a \in A$, the reverse triangle inequality gives
\begin{align*}
\left|\frac{\left|\pi(a)-\pi^{n}(a|s) + \varepsilon(\pi'(a)-\pi(a))\right| - \left|\pi(a) -\pi^{n}(a|s)\right|}{\varepsilon}\right| \leq |\pi'(a)-\pi(a)| \in L^1(\mathbb R^d).
\end{align*}

Observe that for any $m \in \mathcal{P}(A)$,
\begin{equation*}
\int_A A_{\pi^n}(s,a)\frac{\mathrm{d}m}{\mathrm{d}\pi^n}(a|s)\pi^n(\mathrm{d}a|s) = \int_A A_{\pi^n}(s,a)\left(m(\mathrm{d}a)- \pi^n(\mathrm{d}a|s)\right).
\end{equation*}
Define
\begin{equation*}
G(m) \coloneqq \int_A A_{\pi^n}(s,a)\left(m(\mathrm{d}a)- \pi^n(\mathrm{d}a|s)\right) - \frac{1}{\tau}\operatorname{TV}(m, \pi^n(\cdot|s)).
\end{equation*}
Since $\pi^{n+1}(\cdot|s)$ is the unique maximizer of $G$ by \eqref{TV PPO}, it follows that for any $\varepsilon \in (0,1)$,
\begin{align*}
0&\geq \frac{1}{\varepsilon}\left(G(\pi^{n+1}(\cdot|s) +\varepsilon(m-\pi^{n+1}(\cdot|s)))-G(\pi^{n+1}(\cdot|s))\right)\\
&=\int_A A_{\pi^n}(s,a)\left(m(\mathrm{d}a)- \pi^{n+1}(\mathrm{d}a|s)\right)\\ 
&- \frac{1}{\tau \varepsilon}\left(\operatorname{TV}(\pi^{n+1}(\cdot|s) +\varepsilon(m-\pi^{n+1}(\cdot|s)), \pi^{n}(\cdot|s)) - \operatorname{TV}(\pi^{n+1}(\cdot|s), \pi^{n}(\cdot|s))\right).
\end{align*}
Taking the limit $\varepsilon \to 0$ yields
\begin{align*}
0&\geq \lim_{\varepsilon \to 0}\frac{1}{\varepsilon}\left(G(\pi^{n+1}(\cdot|s) +\varepsilon(m-\pi^{n+1}(\cdot|s)))-G(\pi^{n+1}(\cdot|s))\right)\\
&=\int_A \left(A_{\pi^n}(s,a)-\frac{1}{2\tau}\operatorname{sign}\left(\pi^{n+1}(a|s)-\pi^{n}(a|s)\right)\right)\left(m(\mathrm{d}a)- \pi^{n+1}(\mathrm{d}a|s)\right),
\end{align*}
for all $m \in \mathcal{P}(A)$.
The result then follows from \citep[Lemma 33]{jabir2021meanfieldneuralodesrelaxed}.
\end{proof}  
From Lemma~\ref{first order tv ppo}, we observe that a fundamental limitation of scheme \eqref{TV PPO} is the need to know in advance whether the ratio $\frac{\mathrm d\pi^{n+1}}{\mathrm d\pi^n}$ is above or below $1$ in order to compute $\pi^{n+1}$. Since this information is not available a priori, the scheme is implicit and thus impractical.

\section{The FR-PPO algorithm: implementation and experiments}
\label{sec:frpoo_algorithm}
\subsection{FR-PPO as a drop-in replacement for PPO}
\label{sec frppo practical objective}
To obtain a practical implementation of FR--PPO we note that we can re-use any implementation of PPO changing only the policy loss.
To fix ideas say we have an arbitrary parametrization of the policy $\pi_{\theta_\text{old}}$. 
Recall that $f_\varepsilon$ is defined in Section~\ref{sec:ppo}.
We replace the PPO loss 
\begin{equation}
\label{eq:PPO_objective}
\begin{split}
L_{\text{PPO}}(\theta)
& := \mathbb E^{s\sim d_\rho^{\pi_{\theta_\text{old}}}}_{a\sim \pi_{\theta_\text{old}}(\cdot|s)} \Big( f_{\varepsilon}\left(\tfrac{\mathrm d \pi_{\theta}}{\mathrm d\pi_{\theta_{\text{old}}}}, A_{\pi_{\theta_{\text{old}}}}\right)(s,a)\Big)\,
\end{split}
\end{equation}
with the FR--PPO loss given by
\begin{equation}
\label{eq:FRPPO_objective}
\begin{split}
L_{\text{FR-PPO}}(\theta) & = \mathbb E^{s\sim d_\rho^{\pi_{\theta_\text{old}}}}_{a\sim \pi_{\theta_\text{old}}(\cdot|s)}\Big(A_{\pi_{\theta_{\text{old}}}} \tfrac{\mathrm d \pi_{\theta}}{\mathrm d\pi_{\theta_{\text{old}}}}(a|s)- \tfrac{1}{2\tau}\Big|\tfrac{\mathrm d \pi_{\theta}}{\mathrm d\pi_{\theta_{\text{old}}}}(a|s)-1\Big|^2 \tfrac{\mathrm d\pi_{\theta_{\text{old}}}}{\mathrm d \lambda}(a|s) \Big)\,.
\end{split}
\end{equation}
Note that $L_{\text{FR-PPO}}(\theta)$ implements the minimization objective~\eqref{eq:pointwise_min} for policies parametrized by $\theta$ in the sense that  
\[
\begin{split}
\big\langle \tfrac{\mathrm{d}\pi_\theta}{\mathrm{d}\pi_{\theta_\text{old}}}A_{\pi_{\theta_\text{old}}}, \pi_{\theta_\text{old}} \big\rangle_{d^{\pi_{\theta_\text{old}}}_{\rho}} - \frac{1}{8\tau}\int_S \operatorname{FR^2}(\pi_\theta(\cdot|s)^2,\pi_{\theta_\text{old}}(\cdot|s)^2)d_\rho^{\pi_{\theta_\text{old}}}(\mathrm{d}s) = L_{\text{FR-PPO}}(\theta)\,.
\end{split}
\]
We also note that the FR--PPO objective~\eqref{eq:FRPPO_objective} can be estimated from samples collected during the rollout of policy $\pi_{\theta_{\text{old}}}$ and that it is no more or less costly computationally than~\eqref{eq:PPO_objective}. For completeness we present the entire FR--PPO algorithm with a critic in the form a value function parametrized by $\varphi$ as Algorithm~\ref{algo:frppo}.

The only change needed to implement FR--PPO versus a standard PPO algorithm is highlighted in {\color{blue}blue} in Algorithm~\ref{algo:frppo}.
The pseudocode does not include a number of implementation details present in~\cite{schulman2017} that are known to be important~\cite{huanga} for achieving good empirical performance. We chose not to include a discussion of all those details as they will just add clutter here while adding little value since to implement FR--PPO one can just change the loss and keep everything the same as any reference PPO implementation of their choice. We remind the reader that we modified the SB3 implementation of PPO to obtain our FR--PPO code and the code has been made available.

\begin{algorithm}
\caption{FR-PPO}
\label{algo:frppo}
\begin{algorithmic}[1]

\STATE Set discount factor $\gamma$, GAE parameter $\lambda_\text{GAE}$, rollout length $N$, number of opt. epochs $K$, minibatch size $M$
\STATE Initialize actor network $\pi_{\theta}$ and critic network $V_{\varphi}$ 
\STATE Initialize optimizer (e.g., Adam)

\FOR{iteration $= 1, 2, \dots$}
\STATE Rollout (Algorithm~\ref{algo:rollout}) with policy  $\pi_\theta$ to get $\mathcal{D}$
\STATE GAE($\gamma, \lambda_\text{GAE}$) (Algorithm~\ref{algo:gae}) to get $((\hat A_t)_t, (R_t)_t)$

\FOR{epoch $k = 1$ to $K$}
\STATE Shuffle $\mathcal{D}$ and partition into minibatches size $M$
\FOR{each minibatch $B \subset \mathcal{D}$}
\STATE Compute $r_i(\theta) \leftarrow \frac{\pi_\theta}{\pi_{\theta_\text{old}}}(a_i|s_i)$ for $(s_i,a_i) \in B$
{\color{blue}\STATE Compute $L^{\text{pol}}(\theta) = \tfrac{1}{|B|}\sum_{(s_i,a_i) \in B} \big[ r_i(\theta)\hat{A}_i - (2\tau)^{-1} |r_i(\theta) - 1|^2\pi_{\theta_\text{old}}(a_i|s_i)\big]$}
\STATE Compute $L^{\text{val}}(\phi) = \frac{1}{|B|}\sum_{(s_i,R_i) \in B}(V_\varphi(s_i) - R_i)^2$            
\STATE Set $L(\theta,\varphi) = -L^{\text{pol}}(\theta) + L^{\text{val}}(\varphi)$  and perform optimizer step to update $(\theta,\varphi)$ 
\ENDFOR
\ENDFOR
\ENDFOR
\end{algorithmic}
\end{algorithm}

\begin{algorithm}
\caption{Rollout}
\label{algo:rollout}
\begin{algorithmic}[1]

\INPUT Policy $\pi_\theta$, critic $V_\varphi$, environment, number of rollout steps $N$
\STATE Initialize buffer $\mathcal{D}$
\FOR{step $t = 1$ to $N$}
\STATE Sample $a_t \sim \pi_{\theta}(\cdot|s_t)$.
\STATE Use $a_t$ in the environment, get $(r_t, s_{t+1})$, append $(s_t, a_t, r_t, s_{t+1}, \ln \pi_{\theta_\text{old}}(a_t|s_t), V_{\varphi}(s_t))$ to $\mathcal{D}$
\ENDFOR
\OUTPUT $\mathcal D$
\end{algorithmic}
\end{algorithm}

\begin{algorithm}
\caption{GAE~\cite{schulman2015high}}
\label{algo:gae}
\begin{algorithmic}[1]

\INPUT Discount factor $\gamma$, GAE parameter $\lambda_\text{GAE}$, rollout buffer $\mathcal D$ of length $N$    
\STATE Compute advantages and returns using GAE($\gamma, \lambda$) for all $t \in \{1, \ldots, N\}$:
\STATE \quad $\delta_t = r_t + \gamma V_{\varphi}(s_{t+1}) - V_{\varphi}(s_t)$
\STATE \quad $\hat{A}_t = \sum_{l=0}^{N-t-1} (\gamma \lambda_\text{GAE})^l \delta_{t+l}$, \quad $R_t = \hat{A}_t + V_{\varphi}(s_t)$
\OUTPUT $((\hat A_t)_t, (R_t)_t)$
\end{algorithmic}
\end{algorithm}

Providing a method with stronger theoretical guarantees which performs comparably to PPO is the ``raison d'\^etre'' of this paper.
Thus we only compare the performance of FR--PPO with PPO.
We believe that this is justified as~\cite{schulman2017} establishes that PPO compares favourably to many state-of-the art RL algorithms.

\subsection{Numerical implementation}
We have taken the PPO implementation in Stable Baselines\footnote{\href{https://github.com/DLR-RM/stable-baselines3}{\texttt{github.com/DLR-RM/stable-baselines3}}} and implemented\footnote{\href{https://github.com/deterministicdavid/stable-baselines3-contrib-with-frppo}{\texttt{github.com/deterministicdavid/stable-baselines3-contrib-with-frppo}}} FR--PPO by creating a class which duplicates the PPO implementation except that the policy loss is now an empirical estimate over each minibatch of~\eqref{eq:FRPPO_objective}, just as in Algorithm~\ref{algo:frppo}.

We now choose a number of test environments from Atari and Mujoco. We choose all meta-parameters and network architectures following~\cite{huanga} allowing us to match state-of-the-art results with PPO.

\begin{table}
\centering
\begin{tabular}{lcccc}
\hline
\textbf{Domain} & \multicolumn{2}{c}{\textbf{Norm. Reward}} & \multicolumn{2}{c}{\textbf{Envs. Won (max, sum)}} \\
& \textbf{PPO} & \textbf{FR--PPO} & \textbf{PPO} & \textbf{FR--PPO} \\
\hline
Mujoco & 5.4 & 4.4 & (3,2) & (2,3) \\
Atari  & 5.8 & 5.6 & (3,3) & (3,3) \\
\hline
\end{tabular}
\caption{Comparison of PPO vs. FR--PPO on Mujoco and Atari.}
\label{tab:ppo_vs_frppo_summary}
\end{table}

In Table~\ref{tab:ppo_vs_frppo_summary} we provide aggregate normalized results on 5 Mujoco environments (Hopper, HalfCheetah, Reacher, Swimmer, Walker) and 6 Atari environments (BeamRider, Breakout, KungFuMaster, MsPacman, Pong, SpaceInvaders).

\begin{figure}
\begin{center}
\includegraphics[width=0.7\columnwidth]{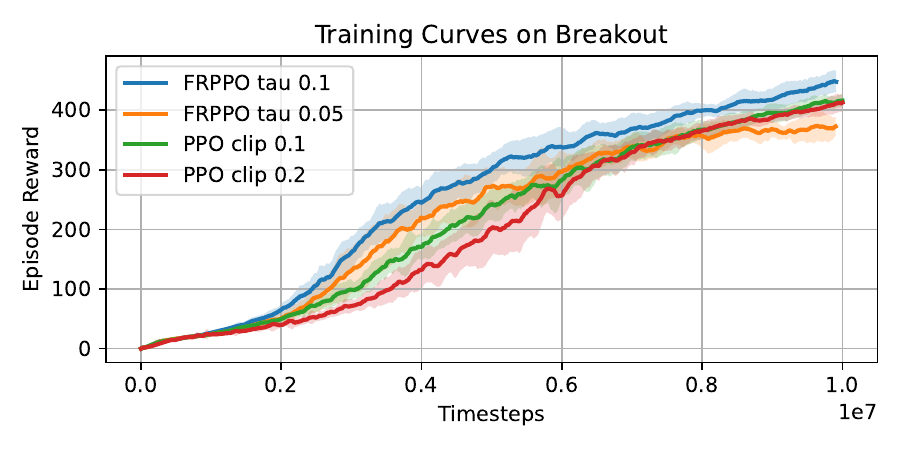}
\end{center}
\caption{Training curve for Atari Breakout using PPO and FR-PPO with various clipping / penalty parameters.}
\label{fig:breakout_ppo_frppo}
\end{figure}

In Figure~\ref{fig:breakout_ppo_frppo} we plot training reward curves for the Breakout environment to show that we reach the same state-of-the-art scores as in~\cite{huanga}.
Further details, results and discussion of numerical experiments are in Appendix~\ref{sec:frpoo_algorithm}.
The picture that emerges is that FR--PPO can perform comparably to PPO but the penalty parameter $\tau>0$ in FR--PPO seems to play a stronger role than the clipping parameter $\varepsilon\in(0,1)$ in PPO. The code for reproducing the presented results is available.\footnote{\href{https://github.com/deterministicdavid/frppo-validation/}{\texttt{github.com/deterministicdavid/frppo-validation/}}}

\subsection{Comparison between PPO and FR--PPO on Atari discrete action environments}

The results are provided in Figure~\ref{fig:atari_ppo_frppo}. The policy and value network is the default provided by SB3 with policy and value networks sharing convolutional neural network (CNN) layers as feature extractors. Let us remind the reader that this and other details follow~\cite{huanga}. Note in particular that the recommended value for PPO clipping parameter $\varepsilon$ is $0.1$. The results using the alternative value of $0.2$ are plotted to compare the sensitivity of the outcome to this metaparameter and also to contrast it to the sensitivity of FR--PPO to the penalty parameter $\tau$. 

\begin{figure}
\begin{center}
\includegraphics[width=0.95\textwidth]{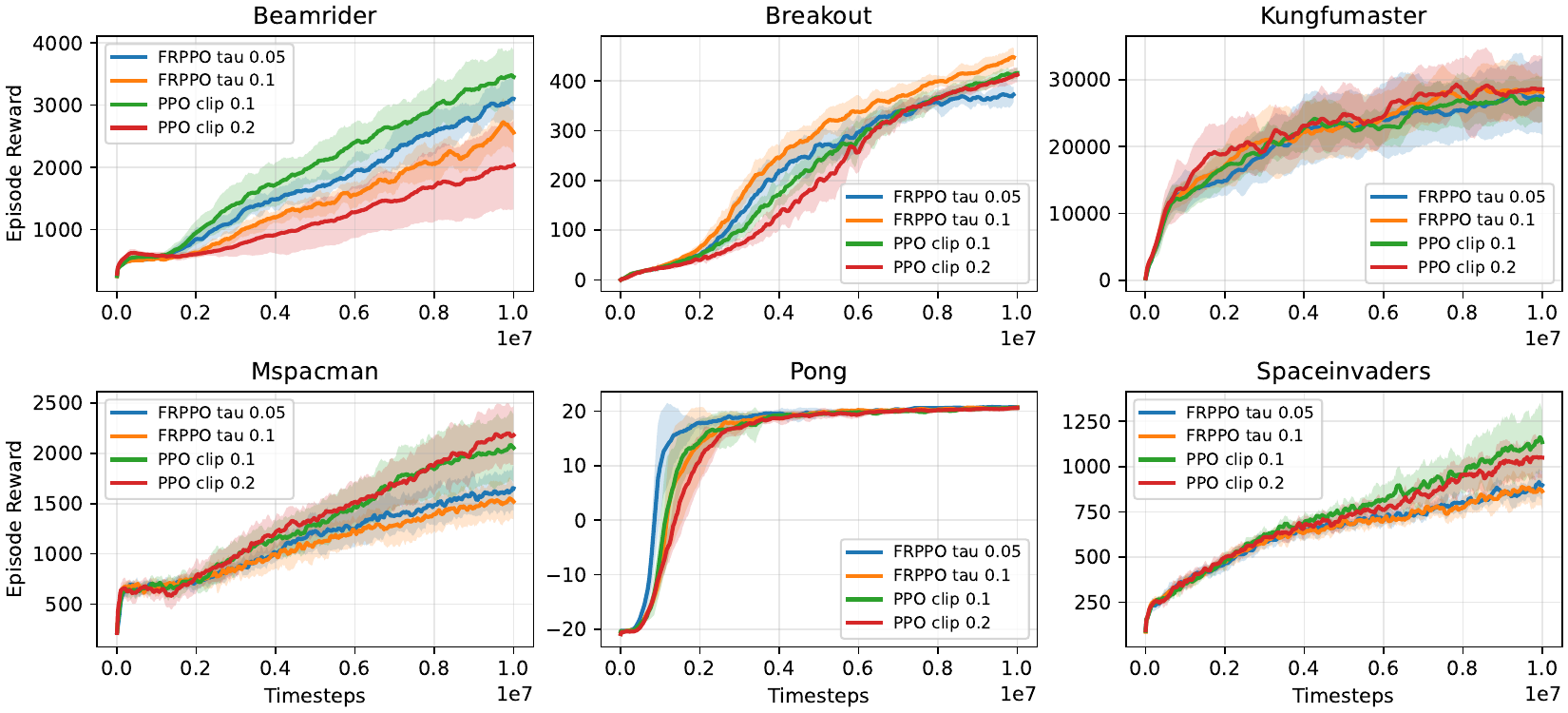}
\end{center}
\caption{Training curve for Atari environments with various clipping / penalty parameters.}
\label{fig:atari_ppo_frppo}
\end{figure}

Following training, we have run evaluation runs as follows.
On each environment 12 training runs with different random seeds.
For each trained policy we then run the environment with a random initialization 40 times. We thus have an average reward for each of the 12 trained policies. The minimum and maximum average reward across all 12 policies on a given environment was used to produce the normalized score and those recorded in Table~\ref{tab:atari_evaluation_results}.

\begin{table}
\centering
\begin{tabular}{lcccccc}
\hline
\textbf{Environment} & \multicolumn{3}{c}{\textbf{PPO}} & \multicolumn{3}{c}{\textbf{FR--PPO}} \\
\cmidrule(lr){2-4} \cmidrule(lr){5-7}
& $\varepsilon=0.1$ & $\varepsilon=0.2$ & Sum & $\tau=0.1$ & $\tau=0.05$ & Sum \\
\hline
BeamRiderNoFrameskip-v4     & {\em 0.65} & 0.24 & 0.89 & 0.44 & 0.52 & \textbf{0.96} \\
BreakoutNoFrameskip-v4      & 0.35 & 0.28 & \textbf{0.63} & {\em 0.39} & 0.21 & 0.60 \\
KungFuMasterNoFrameskip-v4  & 0.47 & {\em 0.48} & \textbf{0.95} & 0.46 & 0.31 & 0.77 \\
MsPacmanNoFrameskip-v4      & 0.44 & {\em 0.64} & \textbf{1.08} & 0.35 & 0.35 & 0.70 \\
PongNoFrameskip-v4          & 0.89 & 0.75 & 1.64 & 0.85 & {\em 0.90} & \textbf{1.75} \\
SpaceInvadersNoFrameskip-v4 & 0.28 & 0.35 & 0.63 & 0.28 & {\em 0.54} & \textbf{0.82} \\
\hline
\textbf{Total wins} & \multicolumn{2}{c}{Max wins {\em 3}} & Sum wins \textbf{3} & \multicolumn{2}{c}{Max wins {\em 3}} & Sum wins \textbf{3}\\
\hline
\end{tabular}
\caption{Average normalized rewards on Atari environments with summed scores.}
\label{tab:atari_evaluation_results}
\end{table}

\subsection{Comparison between PPO and FR--PPO on Mujoco continuous action environments}

\begin{table}
\centering
\begin{tabular}{lcccccc}
\hline
\textbf{Environment} & \multicolumn{3}{c}{\textbf{PPO}} & \multicolumn{3}{c}{\textbf{FR--PPO}}  \\
\cmidrule(lr){2-4} \cmidrule(lr){5-7}
& $\varepsilon=0.1$ & $\varepsilon=0.2$ & Sum & $\tau=0.1$ & $\tau=0.05$ & Sum \\
\hline
HalfCheetah-v5      & 0.48 & 0.49 & 0.97 & 0.63 & {\em 0.67} & \textbf{1.30} \\
Hopper-v5           & {\em 0.73} & 0.62 & \textbf{1.35} & 0.43 & 0.40 & 0.83 \\
Reacher-v5          & 0.54 & {\em 0.72} & 1.26 & 0.68 & 0.70 & \textbf{1.38} \\
Swimmer-v5          & 0.30 & 0.27 & 0.57 & 0.15 & {\em 0.51} & \textbf{0.66} \\
Walker2d-v5         & {\em 0.71} & 0.55 & \textbf{1.26} & 0.10 & 0.11 & 0.21 \\
\hline
\textbf{Total wins} & \multicolumn{2}{c}{Max wins {\em 3}} & Sum wins \textbf{2} & \multicolumn{2}{c}{Max wins {\em 2}} & Sum wins \textbf{3}\\
\hline
\end{tabular}
\caption{Average normalized rewards on Mujoco environments with summed scores.}
\label{tab:mujoco_evaluation_results}
\end{table}

The results are provided in Figure~\ref{fig:mujoco_ppo_frppo}. Here the recommended value for PPO clipping parameter $\varepsilon$ is $0.2$, see~\cite{huanga}. In this case the alternative value is $0.1$. These are plotted to compare the sensitivity of the outcome to this metaparameter and also to contrast it to the sensitivity of FR--PPO to the penalty parameter $\tau$. 

\begin{figure}
\begin{center}
\includegraphics[width=0.95\textwidth]{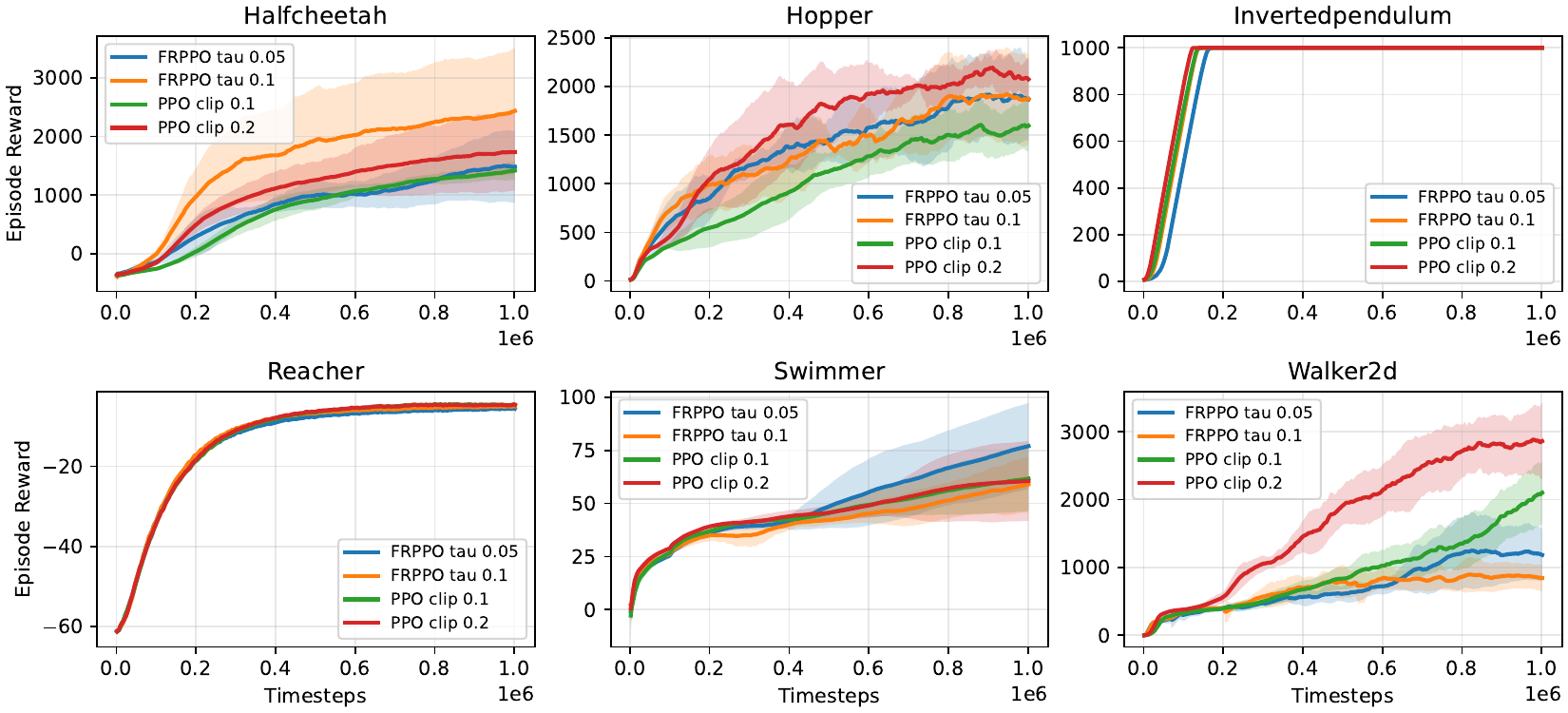}
\end{center}
\caption{Training curve for Mujoco environments with various clipping / penalty parameters.}
\label{fig:mujoco_ppo_frppo}
\end{figure}
Following training, we have run evaluation producing normalized scores in the same manner as for the Atari environments. 
These are recorded in Table~\ref{tab:mujoco_evaluation_results}.
Note that the summary Table~\ref{tab:ppo_vs_frppo_summary} and Table~\ref{tab:mujoco_evaluation_results} does not include results from the inverted pendulum environment because all trained policies were getting full score.

\begin{figure}
\begin{center}
\includegraphics[width=0.4\textwidth]{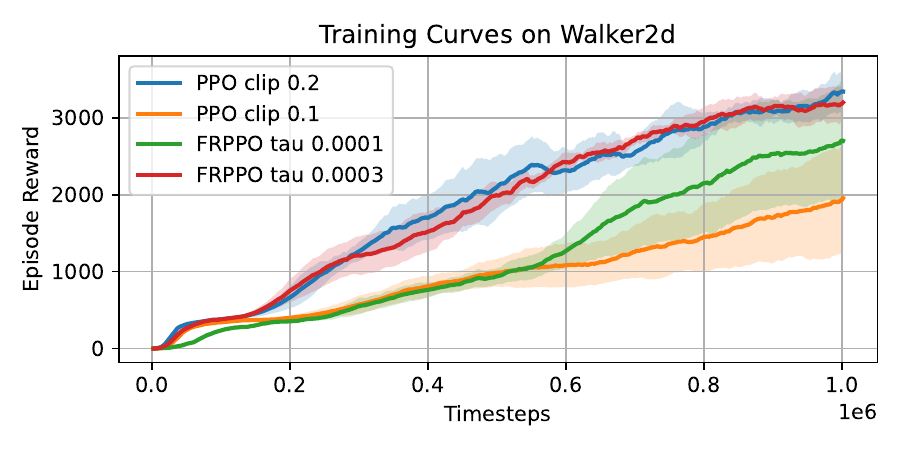}
\caption{Training curves for Walker 2d with alternative $\tau$ values.}
\label{fig:walker2d_alt}    
\end{center}
\end{figure}

\begin{table}
\centering
\begin{tabular}{lcccc}
\hline
\textbf{Environment} & \multicolumn{2}{c}{\textbf{PPO}} & \multicolumn{2}{c}{\textbf{FR--PPO}} \\
\cmidrule(lr){2-3} \cmidrule(lr){4-5}
& $\varepsilon=0.1$ & $\varepsilon=0.2$ & $\tau=1\times 10^{-4}$ & $\tau=3\times 10^{-4}$ \\
\hline
Walker2d-v5    & 0.86 & 0.29 & 0.44 & 0.84 \\
\hline
\end{tabular}
\caption{Average normalized rewards on Walker2d with alternative $\tau$ values.}
\label{tab:walker2d_alt_evaluation_results}
\end{table}

For fairness of comparison we kept the PPO $\varepsilon$ and FR--PPO $\tau$ fixed. 
We see that for example in the Walker2d environment FR-PPO with both $\tau=0.1$ and $\tau=0.05$ performs noticeably worse than PPO.
But when took an order of smaller $\tau$ (larger penalty) in the Walker2d environment we obtained better results for FR--PPO (matching PPO), see Figure~\ref{fig:walker2d_alt} and Table~\ref{tab:walker2d_alt_evaluation_results}.
This is why we believe that FR-PPO has higher sensitivity to the choice of $\tau$ than PPO's sensitivity to choice of $\varepsilon$.
As the main focus of this paper is theoretical we leave further numerical experiments for future work and for those researchers with better access to compute resources.

We have only used 6 Atari and 6 Mujoco environments due to relatively limited access to compute.
The results presented took on the order of days to produce on a desktop-level GPU.

\subsection{Gridworld FR--PPO}
\label{sec:gridworld_frppo}

We implement\footnote{\href{https://github.com/deterministicdavid/frppo-gridworld}{\texttt{github.com/deterministicdavid/frppo-gridworld}}}
FR--PPO for a tabular gridworld environment. Here, we can use exact advantage, which is obtained by iterating the contraction mapping for the on-policy, linear, Bellman equation~\eqref{eq on policy bellman}. The tabular policy then gets updated by first shifting the old policy along the advantage direction (unconstrained) and then doing the projection onto the probability simplex as outlined in Remark~\ref{connection project policy gradient}, which we recall below.

The update $\pi^{n+1}$ can be formulated as a quadratic program with convex constraints performed state-wise:
\begin{equation}
\label{eq:l2_proj_onto_prob_simplex}
\pi^{n+1}(\cdot|s) = \argmin_{m \in \mathcal P(A)} \big\|\pi^n(\cdot|s) +\tfrac{\tau}{1-\gamma} A_{\pi^n}(s,\cdot)-m\big\|_2^2\,.
\end{equation}
Algorithm~\ref{alg:l2_project_onto_simplex} solves the minimization~\eqref{eq:l2_proj_onto_prob_simplex} in $O(|A|\log|A|)$ operations, see~\cite{duchi2008efficient}. The complexity is immediately obvious: all operations are $O(|A|)$ except the initial sorting step.

\begin{algorithm}
\caption{Project Onto Probability Simplex}
\label{alg:l2_project_onto_simplex}
\begin{algorithmic}[1]
\REQUIRE $v \in \mathbb{R}^{|A|}$
\STATE $u \gets \text{sort}(v)$ s.t. $u_1 \geq \cdots \geq u_{|A|}$.
\STATE $c_j \gets \sum_{i=1}^j u_i$, $j=1,\ldots, |A|$.
\STATE Find largest $k\in\{1,\ldots, |A|\}$ s.t. $ u_{k}  > \tfrac{1}{k}(c_{k} - 1)$
\STATE $\theta \gets \frac{c_{k} - 1}{k}$
\STATE $m \gets \max(v - \theta, 0)$
\STATE Return $m$
\end{algorithmic}
\end{algorithm}

\begin{algorithm}
\caption{Tabular exact FR-PPO}
\label{alg:tabular_exact_frppo}
\begin{algorithmic}[1]
\REQUIRE $P$ transition probability matrix, $r$ rewards, $\gamma\in (0,1)$,  $\rho$ state distribution, $\pi^{\text{old}}$ initial policy, $\hat \delta > 0$ stopping threshold.
\STATE $V^\text{old} \gets $ solution of on-policy Bellman~\eqref{eq on policy bellman} with $P$, $r$, $\gamma$ and $\pi^{\text{old}}$.  
\REPEAT
\STATE $Q^\text{old}(s,a) \gets r(s,a) + \gamma \sum_{s'\in S} V^\text{old}(s')P(s'|s,a)$, $s\in S$, $a\in A$.
\STATE $A^\text{old} \gets Q^\text{old} - V^\text{old}$
\STATE $\pi^{\text{new}} \gets \argmin_{m \in \mathcal P(A)} \big\|\pi^{\text{old}}(\cdot|s) +\tfrac{\tau}{1-\gamma} A^{\text{old}}(s,\cdot)-m\big\|_2^2$ solved using Algorithm~\ref{alg:l2_project_onto_simplex}.
\STATE $V^\text{new} \gets$ solution of on-policy Bellman~\eqref{eq on policy bellman} with $P$, $r$, $\gamma$ and $\pi^{\text{new}}$.  
\STATE $\delta \gets V^{\text{new}}(\rho) - V^{\text{old}}(\rho)$
\STATE $\pi^{\text{old}} \gets \pi^{\text{new}}$, $V^{\text{old}} \gets V^{\text{new}}$
\UNTIL $\delta < \hat \delta$
\STATE Return $\pi^{\text{old}}$, $V^{\text{old}}$.

\end{algorithmic}
\end{algorithm}

Full tabular FR-PPO is summarized in Algorithm~\ref{alg:tabular_exact_frppo}. We observe improvement, see Table~\ref{table:fr2_stepping}, as predicted by Theorem~\ref{thm policy improvement}. We also observe convergence, see Figure~\ref{fig:tabular_frppo_conv} consistent with Theorem~\ref{thm:convergence_fr_steps}.

\begin{table}[h]
    \centering
    \begin{minipage}{0.48\textwidth}
        \centering
        \begin{tabular}{ccc}
        \toprule
        \textbf{Step} & \textbf{Mean Value} & \textbf{Improvement} \\
        \midrule
        0 & $-0.78$ & $     -0$ \\
        1 & $-0.66$ & $   0.12$ \\
        2 & $-0.54$ & $   0.12$ \\
        3 & $-0.43$ & $   0.12$ \\
        4 & $-0.33$ & $  0.092$ \\
        5 & $-0.21$ & $   0.12$ \\
        6 & $-0.07$ & $   0.14$ \\
        7 & $0.08$ & $   0.16$ \\
        8 & $ 0.20$ & $   0.11$ \\
        9 & $ 0.27$ & $  0.071$ \\
        \bottomrule
        \end{tabular}
        
    \end{minipage}\hfill
    \begin{minipage}{0.48\textwidth}
        \centering
        \begin{tabular}{ccc}
        \toprule
        \textbf{Step} & \textbf{Mean Value} & \textbf{Improvement}  \\
        \midrule
        10 & $ 0.33$ & $  0.063$ \\
        11 & $ 0.39$ & $  0.056$ \\
        12 & $ 0.44$ & $  0.047$ \\
        13 & $ 0.47$ & $  0.034$ \\
        14 & $ 0.50$ & $   0.03$ \\
        15 & $ 0.53$ & $  0.028$ \\
        16 & $ 0.55$ & $  0.024$ \\
        17 & $ 0.57$ & $  0.023$ \\
        18 & $ 0.59$ & $  0.019$ \\
        19 & $ 0.61$ & $  0.016$ \\
        \bottomrule
        \end{tabular}

    \end{minipage}
    \caption{Tabular FR2 Improvement in value, consistent with Theorem~\ref{thm policy improvement}.}
    \label{table:fr2_stepping}
\end{table}

\begin{figure}[h]
\begin{center}
\includegraphics[width=1.0\textwidth]{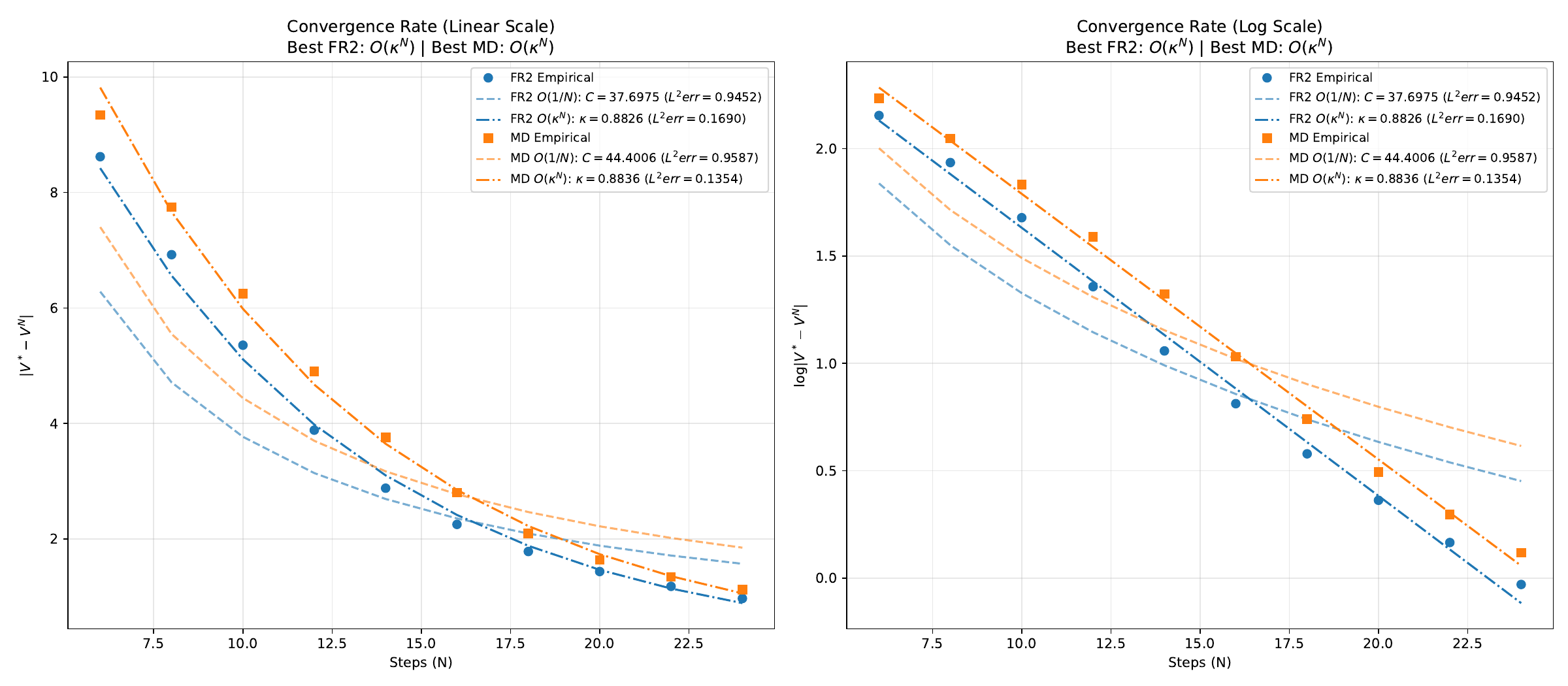}
\caption{Convergence plot for tabular FR--PPO comparing with Mirror Descent (MD). This is consistent with Theorem~\ref{thm:convergence_fr_steps}.}
\label{fig:tabular_frppo_conv}    
\end{center}
\end{figure}

\clearpage

\section*{NeurIPS Paper Checklist}
\begin{enumerate}

\item {\bf Claims}
    \item[] Question: Do the main claims made in the abstract and introduction accurately reflect the paper's contributions and scope?
    \item[] Answer: \answerYes{} 
    \item[] Justification: See the main results in Section \ref{section:main-results}.

\item {\bf Limitations}
    \item[] Question: Does the paper discuss the limitations of the work performed by the authors?
    \item[] Answer: \answerYes{} 
    \item[] Justification: See the discussion in Section \ref{sec:frpoo_algorithm}.

\item {\bf Theory assumptions and proofs}
    \item[] Question: For each theoretical result, does the paper provide the full set of assumptions and a complete (and correct) proof?
    \item[] Answer: \answerYes{} 
    \item[] Justification: See the main results in Section \ref{section:main-results} and the appendices for proofs.

    \item {\bf Experimental result reproducibility}
    \item[] Question: Does the paper fully disclose all the information needed to reproduce the main experimental results of the paper to the extent that it affects the main claims and/or conclusions of the paper (regardless of whether the code and data are provided or not)?
    \item[] Answer: \answerYes{} 
    \item[] Justification: We provide access to an anonymous repository.

\item {\bf Open access to data and code}
    \item[] Question: Does the paper provide open access to the data and code, with sufficient instructions to faithfully reproduce the main experimental results, as described in supplemental material?
    \item[] Answer: \answerYes{} 
    \item[] Justification: We provide access to an anonymous repository that allows to reproduce all experimental results.

\item {\bf Experimental setting/details}
    \item[] Question: Does the paper specify all the training and test details (e.g., data splits, hyperparameters, how they were chosen, type of optimizer) necessary to understand the results?
    \item[] Answer: \answerYes{} 
    \item[] Justification: The experimental setting is presented in Appendix \ref{sec:frpoo_algorithm} and details are in the repositories.

\item {\bf Experiment statistical significance}
    \item[] Question: Does the paper report error bars suitably and correctly defined or other appropriate information about the statistical significance of the experiments?
    \item[] Answer: \answerYes{} 
    \item[] Justification: On each environment 12 training runs with different random seeds. For each trained policy we then run the environment with a random initialization 40 times.

\item {\bf Experiments compute resources}
    \item[] Question: For each experiment, does the paper provide sufficient information on the computer resources (type of compute workers, memory, time of execution) needed to reproduce the experiments?
    \item[] Answer: \answerYes{} 
    \item[] Justification: The experiments were run on GPU clusters and took on the order of tens of hours.
    
\item {\bf Code of ethics}
    \item[] Question: Does the research conducted in the paper conform, in every respect, with the NeurIPS Code of Ethics \url{https://neurips.cc/public/EthicsGuidelines}?
    \item[] Answer: \answerYes{} 
    \item[] Justification: The authors have verified that the research complies with the NeurIPS Code of Ethics.

\item {\bf Broader impacts}
    \item[] Question: Does the paper discuss both potential positive societal impacts and negative societal impacts of the work performed?
    \item[] Answer: \answerNA{} 
    \item[] Justification: The paper is primarily theoretical, and no immediate societal impact is anticipated.
    
\item {\bf Safeguards}
    \item[] Question: Does the paper describe safeguards that have been put in place for responsible release of data or models that have a high risk for misuse (e.g., pre-trained language models, image generators, or scraped datasets)?
    \item[] Answer: \answerNA{} 
    \item[] Justification: The paper is primarily theoretical and the experiments are run on toy models.

\item {\bf Licenses for existing assets}
    \item[] Question: Are the creators or original owners of assets (e.g., code, data, models), used in the paper, properly credited and are the license and terms of use explicitly mentioned and properly respected?
    \item[] Answer: \answerYes{} 
    \item[] Justification: See Section \ref{sec:frpoo_algorithm}.

\item {\bf New assets}
    \item[] Question: Are new assets introduced in the paper well documented and is the documentation provided alongside the assets?
    \item[] Answer: \answerNA{} 
    \item[] Justification: The paper does not release new assets.

\item {\bf Crowdsourcing and research with human subjects}
    \item[] Question: For crowdsourcing experiments and research with human subjects, does the paper include the full text of instructions given to participants and screenshots, if applicable, as well as details about compensation (if any)? 
    \item[] Answer: \answerNA{} 
    \item[] Justification: The paper does not involve crowdsourcing nor research with human subjects.

\item {\bf Institutional review board (IRB) approvals or equivalent for research with human subjects}
    \item[] Question: Does the paper describe potential risks incurred by study participants, whether such risks were disclosed to the subjects, and whether Institutional Review Board (IRB) approvals (or an equivalent approval/review based on the requirements of your country or institution) were obtained?
    \item[] Answer: \answerNA{} 
    \item[] Justification: The paper does not involve crowdsourcing nor research with human subjects.

\item {\bf Declaration of LLM usage}
    \item[] Question: Does the paper describe the usage of LLMs if it is an important, original, or non-standard component of the core methods in this research? Note that if the LLM is used only for writing, editing, or formatting purposes and does \emph{not} impact the core methodology, scientific rigor, or originality of the research, declaration is not required.
    \item[] Answer: \answerNA{} 
    \item[] Justification: The core method development in this research does not involve LLMs as any important, original, or non-standard components.

\end{enumerate}

\end{document}